\documentclass{article}
\usepackage{arxiv}

\usepackage[normalem]{ulem} 

\usepackage[utf8]{inputenc} 
\usepackage[T1]{fontenc}    
\usepackage{hyperref}       
\usepackage{url}            
\usepackage{booktabs}       
\usepackage{amsfonts}       
\usepackage{microtype}      
\usepackage{color}
\usepackage{lipsum}
\usepackage{amsthm}
\usepackage{indentfirst}
\usepackage{graphicx}
\usepackage{epstopdf}
\usepackage{fourier}
\usepackage{bm}
\usepackage{bbm} 
\usepackage{mathtools}
\usepackage{latexsym,enumerate}
\usepackage{ulem}

\def\bm{\boldsymbol}

\newcommand{\comment}[1]{}

\newcommand{\BEA}{\begin{eqnarray}}
\newcommand{\EEA}{\end{eqnarray}}

\newcommand{\td}{\text{d}}
\newcommand{\OL}{\mathcal{L}}

\newcommand{\rank}{\operatorname{rank}}

\DeclarePairedDelimiter{\ceil}{\lceil}{\rceil}
\newtheorem{lem}{Lemma}[section]
\newtheorem{theo}{Theorem}[section]
\newtheorem{prop}{Proposition}[section]
\newtheorem{defi}{Definition}[section]
\newtheorem{coro}{Corollary}[section]
\newtheorem{assu}{Assumption}[section]
\newtheorem{remark}{Remark}
\newtheorem{example}{Example}[section]

\title{Error Bounds of the Invariant Statistics in Machine Learning of Ergodic It\^o Diffusions}

\author{
  He Zhang \\
  Department of Mathematics \\
  The Pennsylvania State University, University Park, PA 16802, USA\\
  \texttt{hqz5159@psu.edu} \\
  \And
  John Harlim \\
  Department of Mathematics, Department of Meteorology and Atmospheric Science, \\ Institute for Computational and Data Sciences \\
  The Pennsylvania State University, University Park, PA 16802, USA\\
  \texttt{jharlim@psu.edu} \\
  \And
  Xiantao Li \\
  Department of Mathematics\\
  The Pennsylvania State University, University Park, PA 16802, USA\\
  \texttt{xxl12@psu.edu} \\
}

\begin{document}

\maketitle

\begin{abstract}
This paper studies the theoretical underpinnings of machine learning of ergodic It\^o diffusions. The objective is to understand the convergence properties of the invariant statistics when the underlying system of stochastic differential equations (SDEs) is empirically estimated with a supervised regression framework. Using the perturbation theory of ergodic Markov chains and the linear response theory, we deduce a linear dependence of the errors of one-point and two-point invariant statistics on the error in the learning of the drift and diffusion coefficients. More importantly, our study shows that the usual $L^2$-norm characterization of the learning generalization error is insufficient for achieving this linear dependence result. We find that sufficient conditions for such a linear dependence result are through learning algorithms that produce a uniformly Lipschitz and consistent estimator in the hypothesis space that retains certain characteristics of the drift coefficients, such as the usual linear growth condition that guarantees the existence of solutions of the underlying SDEs. We examine these conditions on two well-understood learning algorithms: the kernel-based spectral regression method and the shallow random neural networks with the ReLU activation function.
\end{abstract}


\section{Introduction}

Model error is inevitable, whether the model is formulated from direct empirical observations or deduced from fundamental physical principles, e.g., conservation laws. In this paper, we study modeling error arising from learning dynamical systems that obey a system of stochastic differential equations (SDEs) driven by Brownian noise \cite{kloeden2013numerical, oksendal2013stochastic,Pav_book:14}, which are used in many scientific disciplines. In this context, the task is to identify the drift and diffusion coefficients from a time series of the SDEs. This inverse problem has been a central topic of interest for a long time and posted under various names, from parameter estimation, data-driven modeling, closure modeling, and lately, as a supervised learning task as machine learning becomes popular.

When the function forms are presumed, many classical methods, e.g., moment methods, maximum likelihood, and filtering, can be used \cite{nielsen2000parameter}. Along this line, the MCMC-based Bayesian inference \cite{peavoy2015systematic,mbalawata2013parameter} is an important direction that allows for the estimation of the distribution of the parameters instead of point estimation in the traditional approaches. Since the same problem can be posed as a supervised learning task, a lot of recent interest has been shifting to machine learning approaches. Among the linear estimators, a popular approach is the kernel-based method  \cite{chen2020non,nickl2020nonparametric,garcia2017nonparametric,rajabzadeh2016robust,lamouroux2009kernel,chmiela2018towards,chmiela2020accurate}, whose connection to the parametric modeling paradigm has been studied in \cite{jh:20}. In this direction, many nonparametric models have been proposed, including the orthogonal polynomials \cite{rajabzadeh2016robust}, wavelets \cite{nickl2020nonparametric}, Gaussian processes \cite{garcia2017nonparametric}, radial kernels \cite{garcia2017nonparametric}, diffusion maps based models \cite{berry2015nonparametric,berry2020bridging,gilani2021kernel}, just to name a few. Beyond the kernel approaches, the neural-network approach has been applied to estimate the drift coefficient \cite{koner2020permutationally} with application in biomolecular modeling, and the missing component in the drift term \cite{HJLY:19} with application to modeling atmospheric flow over topography. 

Recurrent neural networks were shown to produce state-of-art numerical performances in learning high-dimensional nonlinear dynamical systems,  even beyond SDEs \cite{vlachas2018data,ma2018model,pan2018data,HJLY:19}.
These empirical successes, however, are not completely understood. Specifically, while the approximation theory of recurrent neural networks has been studied (see e.g., \cite{hammer2000,schafer2006recurrent}), it remains unclear whether the neural network model, obtained from a training procedure that involves solving a nonlinear, highly non-convex, optimization problem, can provide a convergence guarantee.  Reservoir Computing (RC) \cite{jaeger2001echo} was introduced as an alternative to the tedious training procedure in the neural network model. This method, which is effectively a random neural network \cite{rahimi2008uniform} in the context of recurrent neural network architecture, is based on the premise that fitting randomized function can be as effective yet computationally cheaper than solving the corresponding nonlinear optimization problem. This class of approach is effectively a conditionally linear estimator since it specifies the parameters in the activation function by randomly generated weights and trains the outer weights using the linear (ridge-) regression method. This surprisingly simple training procedure was shown to be effective in learning attractors of chaotic dynamical systems \cite{jaeger2004harnessing,pathak2018model,pathak2017using}. Recent theoretical results also shed some lights on its approximation and estimation properties \cite{gonon2020approximation,gonon2020risk}, and universality in learning stochastic processes \cite{grigoryeva2018universal}.

Building on the above independently reported positive successes, our goal is to understand under which conditions the underlying stochastic processes, driven by unknown SDEs, can be accurately emulated by a supervised learning procedure. Various metrics can be used to quantify the consistency of the estimated dynamics. For example, the (strong) pathwise error convergence that is classically used to characterize the numerical discretization error  \cite{kloeden2013numerical} has been used for quantifying the accuracy in learning partially known dynamics \cite{HJLY:19}. In the SDE application, they deduced under mild conditions that one can achieve accurate pathwise predictions up to a finite time with an error bound that is polynomial as a function of the learning error rate. In this paper, we will quantify the error in the estimation of one-point invariant statistics and two-point correlation statistics. Although these two metrics are commonly used to empirically assess the performance of the estimated dynamics through various algorithms \cite{cev:08,mh:13,cl:15,hl:15,peavoy2015systematic,chmiela2018towards,jh:20,HJLY:19},
as a means to validate the consistency of the estimated dynamics, our emphasis is placed on the theoretical analysis. In particular, we will show that the errors in these statistics will depend linearly on a parameter $\epsilon$ that reflects the error in the estimation of drift and diffusion coefficients. This result not only guarantees the convergence of the invariant statistical estimation as $\epsilon\to 0$, but also provides a means for designing efficient learning algorithms when the parameter $\epsilon$ is specified as a function of the size of training data and other parameters that characterize the ``size'' of the hypothesis space, the strength of the noise, and the step size of the discrete-time series.     

Our study will be based on the perturbation theory of ergodic Markov chains \cite{rudolf2018perturbation,shardlow2000perturbation,mattingly2002ergodicity} and the linear response theory \cite{hairer2010simple}, which will be reviewed in Section~\ref{sec:review}. In the context of learning, we will specify the perturbation as the error induced by a regression learning framework in estimating the drift and diffusion coefficients of SDEs. Our main contribution, which will be discussed in Section~\ref{sec:per_theory}, is to deduce error bounds of the estimation of one-point and two-point invariant statistics in terms of the error of the learning framework. Beyond these error bounds, the more important aspect of this study is to specify mathematical conditions that allow for the error bounds to be valid. Practically, these conditions allow one to pre-determine whether the proposed learning method is adequate or whether it can be adjusted to guarantee a convergent estimation. We will examine the validity of these conditions on two machine learning methods. In Section~\ref{sec:application}, we will discuss a kernel-based spectral regression method. We consider an RKHS induced by the orthonormal set of eigenfunctions of an integral operator defined over the invariant distribution of the data, which can be empirically estimated from the discrete samples. Subsequently, in section~\ref{sec:application_ML}, we will discuss a random neural network model of a simple single hidden-layer feed-forward neural network with ReLU activation function, which is a simple randomized function approximation relative to the reservoir computing. In these two sections, we will provide an overview of the generalization errors of these methods adopted in our application. This discussion is mainly based on the results in \cite{rosasco2010learning,gonon2020approximation,wang2011optimal, wang2012erm, cucker2002mathematical}. For these two machine learning algorithms, we will also examine the validity of the Assumption~\ref{assu:coe_per} that underpins the perturbation theory of Markov chains.  In Section~\ref{sec:summary}, we close the paper with a summary and some discussions on open issues.

\section{Existing theory on statistics of perturbed Markov chains} \label{sec:review}

In this section we will review the essential concepts and results in the perturbation theory of Markov chains \cite{rudolf2018perturbation, shardlow2000perturbation} (Section~\ref{sec:per_MC}), ergodic theory of SDEs \cite{mattingly2002ergodicity} (Section~\ref{sec:Ito} and  Section~\ref{sec:erg_Ito}), and the long-time linear response theory \cite{hairer2010simple} (Section~\ref{sec:lin_resp}). The theory involves both continuous Markov processes, e.g., It\^o diffusions, and discrete Markov chains, e.g., the Euler-Maruyama approximation of the It\^o diffusions. We will use the notation $\{\cdot(t)\}$, e.g., $\{X(t)\}$, and $\{\cdot_{n}\}$, e.g., $\{X_{n}\}$, to denote Markov processes and Markov chains, respectively. Throughout the paper, $\|\cdot\|$ always denotes the standard Euclidean norm in $\mathbb{R}^{d}$.

\subsection{A perturbation theory for ergodic Markov chains}\label{sec:per_MC}

In this section, we review the perturbation theory for ergodic Markov chains in \cite{rudolf2018perturbation,shardlow2000perturbation}. Let $\mathcal{B}(\mathbb{R}^{d})$ denote the Borel $\sigma$-algebra on $\mathbb{R}^{d}$, and $\{u_{n}\}_{n=0}^{\infty}$ always denotes a Markov chain from a probability space $(\Omega, \mathcal{F}, \mathbb{P})$ to $(\mathbb{R}^{d}, \mathcal{B}(\mathbb{R}^d))$. In what follows, we will use the shorthand notation $|f|\leq V$ to mean $|f(x)|\leq V(x)$ for all $x\in \mathbb{R}^{d}$. Our first definition, following  \cite{shardlow2000perturbation}, focuses on the concept of geometrically ergodic Markov chains.
\begin{defi} \label{def:geo_erg}
A Markov chain $\{u_{n}\}_{n=0}^{\infty}$ is \emph{geometrically ergodic}, if:
\begin{enumerate}[i.]
    \item There exists a unique invariant measure, $\pi$, on $(\mathbb{R}^{d}, \mathcal{B}(\mathbb{R}^d))$.
    \item There exists a measurable function $V: \mathbb{R}^{d}\rightarrow [1,+\infty)$ such that
    \begin{equation*}
\mathbb{E}^{x}[V(u_{n})] < \infty, \quad \forall n \geq 0,
    \end{equation*}
    where $\mathbb{E}^{x}[\cdot]$ denotes the expectation under the initial condition $u_{0}=x$.
    \item Let $\mathcal{G}$ be the set of all measurable functions $f$ with $|f|\leq V$. There exists a set $\mathcal{G}_{0}\subset \mathcal{G}$ containing $V$ such that
    \begin{equation}\label{eq:geo_erg}
    \sup_{f\in \mathcal{G}_{0}} \left| \mathbb{E}^{x}[f(u_{n})]  - \pi(f) \right| \leq R \rho^{n} V(x), \quad \forall n \geq 0,
\end{equation}
for some constants $R\in(0,+\infty)$ and $\rho\in(0,1)$. Here, $\pi(f):= \int f \pi(\td x)$. 
\end{enumerate}
\end{defi}
The inequality in \eqref{eq:geo_erg}, as the key component of the geometrical ergodicity, describes the decay rate of the $V$-norm distance \cite{rudolf2018perturbation} between the distribution of $u_{n}$ and the invariant measure $\pi$ as $n\rightarrow +\infty$. In general, we may consider other probability distances to characterize the convergence in \eqref{eq:geo_erg}, e.g., the Wasserstein distance \cite{rudolf2018perturbation}. Here, the set $\mathcal{G}_{0}$ is introduced to rule out certain ``ill-behaved'' observables for simplicity. In our later discussions, $\mathcal{G}_{0}$ is either $\mathcal{G}$ itself, e.g., in Theorem~\ref{thm:ergodic}, or a set of locally Lipschitz functions, e.g., in Proposition~\ref{prop:erg_euler}. 

We approximate the geometrically ergodic Markov chain $\{u_{n}\}$ by another perturbed Markov chain $\{u_{n}^{\epsilon}\}$ (which may not be geometrically ergodic), where $\epsilon$ represents the ``scale'' of the perturbation. To specify how the approximation error is carried over to the error of the resulting statistics in the long run, we will state the following perturbation bound, which is a result of the Corollary 3.4 in \cite{rudolf2018perturbation}.

\begin{prop}\label{prop:per_bound}
Let $\{u_{n}\}$ be a geometrically ergodic Markov chain satisfying the condition in Definition~\ref{def:geo_erg}. We further assume that the Lyapunov function $V$ satisfies
\begin{equation}\label{eq:Exp_Lyp}
    \mathbb{E}^{x}[V(u_{1})] \leq \alpha V(x) + \beta, \quad \forall x\in \mathbb{R}^{d},
\end{equation}
with constants $\alpha\in(0,1)$ and $\beta\in (0,+\infty)$. Let $\{u_{n}^{\epsilon}\}$ be a perturbed Markov chain with respect to $\{u_{n}\}$. We define
\begin{equation}\label{eq:V_norm}
    \gamma : = \sup_{x\in \mathbb{R}^{d}} \sup_{f\in \mathcal{G}_{0}}  \frac{\left| \mathbb{E}^{x}[f(u_{1}^{\epsilon})] - \mathbb{E}^{x}[f(u_1)]\right|}{V(x)}. 
\end{equation}
If $\gamma\in (0, 1-\alpha)$, then, for any fixed initial condition $u_{0}=u^{\epsilon}_0 = x \in \mathbb{R}^{d}$, we have,  
\begin{equation}\label{eq:prop_21}
    \sup_{f\in \mathcal{G}_{0}} \left| \mathbb{E}^{x}[f(u_{n}^{\epsilon})] - \mathbb{E}^{x}[f(u_n)]\right| \leq R(1-\rho^{n}) \frac{\gamma \kappa}{1-\rho}, \quad \kappa: = \max\left\{V(x), \; \frac{\beta}{1-\gamma-\alpha}\right\}, \quad \forall n \geq 0,
\end{equation}
for some constant $R\in (0,+\infty)$, where $\rho\in(0,1)$ is defined by \eqref{eq:geo_erg}.
\end{prop}
We would like to point out that the original error bound presented by Corollary 3.4 in \cite{rudolf2018perturbation} allows for $u_{n}$ and $u_{n}^{\epsilon}$ to have different initial conditions. 

As a direct consequence of Proposition~\ref{prop:per_bound} and Eq.~\eqref{eq:geo_erg}, we have
\begin{equation}\label{eq:long_time}
    \sup_{f\in \mathcal{G}_0} \left| \mathbb{E}^{x}[f(u_{n}^{\epsilon})] - \pi(f)\right| \leq R \left[ (1-\rho^{n}) \frac{\gamma\kappa}{1-\rho} + \rho^{n}V(x) \right], \quad \forall n\geq 0,
\end{equation}
for some constant $R\in (0,+\infty)$.
The fact that the error bound in \eqref{eq:prop_21} depends on $\gamma$, defined through the V-norm in \eqref{eq:V_norm}, provides a convenient way for characterizing the errors of the estimated transition kernel. Namely, we only need to study the ``one-step'' error between the statistics of $\{u_1\}$ and $\{u_1^\epsilon\}$. When $\mathcal{G}_{0} = \mathcal{G}$, the constant $\gamma$ in \eqref{eq:V_norm} is also used as the upper bound of the V-norm difference between the transition kernels of $\{u_{n}\}$ and $\{u_{n}^{\epsilon}\}$ over $x\in \mathbb{R}^{d}$ \cite{rudolf2018perturbation}.

\subsection{The It\^o diffusion and its approximations}\label{sec:Ito}

The key results developed in our paper focus on the case where the underlying dynamics is an It\^o diffusion. In this section, we review some basic concepts and results related to the It\^o diffusion and its approximations. For classical theory of It\^o diffusions, readers may refer to \cite{kloeden2013numerical, oksendal2013stochastic,Pav_book:14}.

A $d$-dimensional It\^o diffusion is a SDE of the form
\begin{equation}\label{eq:Ito_diff}
\dot{X}(t) = b(X(t)) + \sigma(X(t)) \dot{W}, \quad X(0) = x, \quad t\geq 0,
\end{equation}
where $x\in \mathbb{R}^{d}$, $b: \mathbb{R}^{d}\rightarrow \mathbb{R}^{d}$ and $\sigma: \mathbb{R}^{d}\rightarrow \mathbb{R}^{d\times m}$ ($m\leq d$) are the initial condition, drift and diffusion coefficients, respectively. The process $W$ in \eqref{eq:Ito_diff} denotes a standard $m$-dimensional Brownian motion. When $m<d$, the noise in \eqref{eq:Ito_diff} is degenerate. It\^o diffusions with degenerate noise are common in applications, e.g., Langevin dynamics \cite{Pav_book:14}. Non-degeneracy often makes it convenient to prove certain properties of the corresponding It\^o diffusions, e.g., the ergodicity \cite{mattingly2002ergodicity} and the regularity of the invariant measure \cite{bogachev2015fokker}. But for the sake of generality, in our paper, we will consider It\^o diffusions with possible degenerate noise.
We propose the following assumption on the coefficients.
\begin{assu}\label{assu:Ito_coef}
The coefficients $b$ and $\sigma$ in \eqref{eq:Ito_diff} are Borel measurable and satisfy the following conditions:
\begin{enumerate}[i.]
    \item \textbf{Globally Lipschitz condition:} There exists a constant $K_1\in(0,+\infty)$ such that \begin{equation*}
        \|b(x)-b(y)\| \leq K_{1}\|x-y\|, \quad \|\sigma(x) - \sigma(y)\|_{F} \leq K_{1}\|x-y\|,  \quad \forall x,y\in \mathbb{R}^{d}.
    \end{equation*}
    \item \textbf{Linear growth bound:} There exists a constant $K_{2}\in(0,+\infty)$ such that
    \begin{equation*}
    \|b(x)\|^2\leq K_{2}^{2}(1 + \|x\|^2), \quad \|\sigma(x)\|^2_{F} \leq K^{2}_{2}(1+\|x\|^2),  \quad \forall x,y\in \mathbb{R}^{d}.
    \end{equation*}
\end{enumerate}
Here, $\|\cdot\|_{F}$ denotes the Frobenius norm, that is, $\|\sigma\|_{F} = \left(\sum\limits_{i,j} \sigma^2_{ij}\right)^{\frac{1}{2}}$.
\end{assu}
We shall henceforth hold fixed a Brownian motion $W$, and the associated family of $\sigma$-algebra $\{\mathcal{A}_{t}, t\geq 0\}$. Assumption~\ref{assu:Ito_coef}(i) ensures the existence and uniqueness of the $\{\mathcal{A}_{t}\}$-adapted strong solution of \eqref{eq:Ito_diff} \cite{kloeden2013numerical,oksendal2013stochastic}. In particular, the following lemma provides useful bounds on the even order moments of the solution to \eqref{eq:Ito_diff}.

\begin{lem}\label{lem:mom_Ito}
Suppose Assumption~\ref{assu:Ito_coef} (ii). Then, for any integer $p\geq 1$, the solution $X(t)$ of \eqref{eq:Ito_diff} satisfies
\begin{equation*}
    \mathbb{E}^{x}\left[\|X(t)\|^{2p}\right] \leq e^{2K_{2}(4p^2+2p) t}(1+\|x\|^{2p}), \quad \forall t\geq 0,
\end{equation*}
where the constant $K_2$ is the same as in Assumption~\ref{assu:Ito_coef}.
\end{lem}

\begin{proof}
The proof for $d=1$ can be found in \cite{kloeden2013numerical} (Theorem 4.5.4), which can be directly generalized to higher dimensional cases.
\end{proof}

The It\^o diffusion \eqref{eq:Ito_diff} defines a time-continuous Markov process. In practice, we may introduce the corresponding Markov chains either by sampling or numerical discretizations. By sampling, we mean the resulting Markov chain $\{X_{n}\}$ satisfies $X_{n} = X(n\delta)$, for some $\delta>0$. For numerical discretizations, a classical example is the Euler-Maruyama (EM) scheme \cite{kloeden2013numerical}. 

For a fixed step size $\delta>0$, the EM scheme generates a Markov chain $\{X_{n}^{\delta}\}$ via,
\begin{equation}\label{eq:Euler_approx}
    X^{\delta}_{0} = x, \quad  X_{n+1}^{\delta} = X_{n}^{\delta} + \delta b(X_{n}^{\delta}) + \sqrt{\delta}\sigma(X_{n}^{\delta})\xi_{n}, \quad  n =0 ,1,\dots,
\end{equation}
where $\{\xi_{n}\}$ denotes a sequence of independent, identically distributed, $m$-dimensional standard Gaussian random variables. The EM scheme attains the $1/2$-order of strong convergence \cite{kloeden2013numerical}, that is, under Assumption~\ref{assu:Ito_coef}, there exists constants $R,D\in(0,+\infty)$ such that
\begin{equation}\label{eq:strong_euler}
    \mathbb{E}^{x}\left[\|X_{n} - X_{n}^{\delta}\|^{2}\right] \leq Re^{Dn \delta}(1+\|x\|^{2}) \delta, \quad n=0,1,\dots,
\end{equation}
where the constants $R,D \in (0,+\infty)$ are independent of $\delta$. Moreover, $X_{n}^{\delta}$ yields similar moment bounds as in Lemma~\ref{lem:mom_Ito} (see Theorem 10.2.2 in \cite{kloeden2013numerical} for details). It is worthwhile to mention that the bounds discussed in Lemma~\ref{lem:mom_Ito} and Eq.~\eqref{eq:strong_euler} can be improved under extra assumptions, e.g.,  drift coefficients $b$ satisfying dissipative conditions \cite{lamba2007adaptive}. 

\subsection{The ergodic It\^o diffusions}\label{sec:erg_Ito}

In Section~\ref{sec:per_MC}, we have discussed a perturbation result for ergodic Markov chains. In this paper, the Markov chains are generated from ergodic It\^o diffusions by sampling or numerical discretization. Thus, we need to first inspect the concept of ergodic It\^o diffusions and how they are connected to ergodic Markov chains. The unlisted proofs in this section can be found in \cite{mattingly2002ergodicity}. Following \cite{mattingly2002ergodicity}, we consider It\^o diffusions \eqref{eq:Ito_diff} with additive noise,
\begin{equation}\label{eq:unper}
\dot{X}(t) = b(X(t)) + \sigma \dot{W}, \quad X(0) = x,
\end{equation}
where $\sigma\in \mathbb{R}^{d\times m}$ ($m\leq d$). The fixed constant matrix $\sigma$ is assumed to have linearly independent column vectors, that is, $\rank(\sigma) = m$. Here, $\{X(t)\}$ in \eqref{eq:unper} forms a Markov process on the state space $(\mathbb{R}^{d}, \mathcal{B}(\mathbb{R}^d))$. We denote the transition kernel of the Markov process $\{X(t)\}$ by
\begin{equation}\label{eq:tran_kernel}
P_{t}(x, A):=\mathbb{P}(X(t)\in A \;|\; X(0)=x), \quad t\geq 0, \; x\in \mathbb{R}^{d}, \; A\in \mathcal{B}(\mathbb{R}^d).   
\end{equation}
To establish the geometric ergodicity for the system \eqref{eq:unper}, we introduce the following series of assumptions \cite{mattingly2002ergodicity}.
\begin{assu}\label{assu:0}
The transition kernel $P_{t}$ in \eqref{eq:tran_kernel} satisfies, for some fixed compact set $S\in \mathcal{B}(\mathbb{R}^{d})$, the following:

\begin{enumerate}[i.]
\item For some $y^{*}\in \operatorname{int}(S)$ (the interior of $S$) and any $r>0$, there is a $t_1= t_{1}(r)>0$ such that
    \begin{equation*}
        P_{t_1}(x, \mathcal{B}_{r}(y^{*})) >0, \quad \forall x\in S.
    \end{equation*}
    \item For any $t>0$ the transition kernel yields a density $p_{t}(x,y)$, that is,
    \begin{equation*}
        P_{t}(x,A) = \int_{A} p_{t}(x,y) \td y, \quad \forall x\in S, \; A\in \mathcal{B}(\mathbb{R}^{d})\cap \mathcal{B}(S),
    \end{equation*}
    and $p_{t}(x,y)$ is jointly continuous in $(x,y)\in S \times S$.
\end{enumerate}
Here, $\mathcal{B}_{r}(y^{*})$ denotes the open ball of radius $r$ centered at $y^{*}$, and $\mathcal{B}(S)$ denotes the sub-$\sigma$-algebra on $S$ with respect to $\mathcal{B}(\mathbb{R}^{d})$.
\end{assu}

\begin{assu}\label{assu:1}
There is a function $V:\mathbb{R}^{d}\rightarrow [1,+\infty)$, with $\lim\limits_{x\rightarrow \infty} V(x)= +\infty$, and $a_{1},d_{1}\in(0,+\infty)$ such that
\begin{equation*}
    \OL V(x) \leq -a_{1} V(x) + d_{1}, \quad \forall x\in \mathbb{R}^{d}.
\end{equation*}
Here $\OL$ is the generator for \eqref{eq:unper} given by
\begin{equation}\label{eq:gen}
    \OL f = \sum_{i=1}^{d} b_{i} \frac{\partial f}{\partial x_i}+ \frac{1}{2}\sum_{i,j=1}^{d}(\sigma\sigma^{\top})_{ij} \frac{\partial^2 f}{\partial x_i \partial x_j},
\end{equation}
where $(\sigma\sigma^{\top})_{ij}$ denotes the $ij$-component of the matrix $\sigma\sigma^{\top}$.
\end{assu}
As a direct consequence of Assumption~\ref{assu:1} and the Dynkin's formula \cite{oksendal2013stochastic,mattingly2002ergodicity}, we have
\begin{equation}\label{eq:Lyap_unper_1}
    \mathbb{E}^{x}[V(X(t))]  \leq e^{-a_{1}t}V(x) + \frac{d_{1}}{a_{1}}(1-e^{-a_{1}t}), \quad \forall x\in \mathbb{R}^{d}, \; \forall t\geq 0,
\end{equation}
where $\mathbb{E}^{x}[\cdot]$ denotes the expectation under \eqref{eq:unper}, with respect to the initial condition $X_{0}=x$. When $t = \delta$, we have
\begin{equation*}
    \mathbb{E}^{x}[V(X_1)] =\mathbb{E}^{x}[V(X(\delta))] \leq e^{-a_{1}\delta}V(x) + \frac{d_{1}}{a_{1}}(1-e^{-a_{1}\delta}), \quad \forall x\in \mathbb{R}^{d},
\end{equation*}
that is, the Markov chain $\{X_{n} = X(n\delta)\}$ satisfies the condition \eqref{eq:Exp_Lyp} in Proposition~\ref{prop:per_bound}. The following theorem (Theorem 2.5 in \cite{mattingly2002ergodicity}), guaranteeing the geometric ergodicity of the Markov chain $\{X_n\}$, is the foundation of our study of the perturbation theory.
\begin{theo}\label{thm:ergodic}
Let $\{X(t)\}$ be the Markov process defined in \eqref{eq:unper} that satisfies Assumptions~\ref{assu:0} and \ref{assu:1} with the compact set $S$ given by
\begin{equation*}
    S = \left\{x \in \mathbb{R}^{d}\; |\; V(x) \leq \frac{2d_{1}}{a_{1}(\zeta - e^{-a_{1}\delta})}\right\},
\end{equation*}
for some $\zeta \in (e^{-a_{1}\delta/2},1)$ and $\delta>0$. Then there exists a unique invariant measure $\pi$. Furthermore, there exist $\rho = \rho(\zeta)\in(0,1)$ and $R = R(\zeta)\in(0,+\infty)$ such that
\begin{equation*}
    \sup_{f\in\mathcal{G}}\left| \mathbb{E}^{x}[f(X_{n})] - \pi(f)\right| \leq R\rho^{n}V(x), \quad \forall x\in\mathbb{R}^{d}, \; \forall n\geq 0,
\end{equation*}
where $\mathcal{G}$ denotes the set of all measurable functions  with $|f|\leq V$, that is, $\{X_{n}\}$ is a geometrically ergodic Markov chain defined in Definition~\ref{def:geo_erg} with $\mathcal{G}_{0} = \mathcal{G}$.
\end{theo}
In \cite{mattingly2002ergodicity}, Theorem~\ref{thm:ergodic} has been applied to a variety of SDEs, including the Langevin dynamics, monotone and dissipative systems, and stochastic gradient systems. The function $V$ in Assumption~\ref{assu:1} is called the \emph{Lyapunov function} of the dynamical system \eqref{eq:unper}. In particular, we further assume that $V$ is of a polynomial growth rate. 
\begin{assu}\label{assu:2}
The Lyapunov function $V$ in Assumption~\ref{assu:1} is of the form $V = W^{\ell}$ for some $\ell\geq 1$, where $W$ is essentially quadratic, i.e., there exist constants $C_{i}\in(0,+\infty)$, $i=1,2,3$, such that
\begin{equation}\label{eq:esse_quad}
    C_{1}\left(1+\|x\|^2 \right) \leq W(x) \leq C_{2}\left(1 + \|x\|^{2} \right), \quad \|\nabla W(x)\|\leq C_{3}\left( 1+ \|x\| \right), \quad \forall x\in \mathbb{R}^{d}.
\end{equation}
\end{assu}

Assumption~\ref{assu:2} is not only useful in deriving perturbation bounds in Section~\ref{sec:per_theory} (see Lemma~\ref{lem:well_pose} for the details), but also ensures that any ``reasonable'' numerical discretization scheme will inherit the ergodicity of \eqref{eq:unper} \cite{mattingly2002ergodicity}. In particular, we have the following proposition.
\begin{prop} \label{prop:erg_euler}
Let Assumptions~\ref{assu:Ito_coef}- \ref{assu:2} hold. Then, there exists $\delta_0>0$, such that $\forall \delta \in (0, \delta_0)$ the Markov chain generated by the EM scheme with step size $\delta$, $\{X_{n}^{\delta}\}$ in \eqref{eq:Euler_approx}, is geometrically ergodic with invariant measure $\tilde{\pi}^{\delta}$ and with same Lyapunov function $V(x)$ as in Assumption~\ref{assu:1}. In particular, we define
\begin{equation*}
\mathcal{G}_{\ell}: = \left\{f\in \mathcal{G}\; \big| \;|f(x)-f(y)| \leq C_{\ell}\left(1+ \|x\|^{2\ell-1} + \|y\|^{2\ell-1}\right)\|x-y\|, \quad \forall x,y\in \mathbb{R}^{d} \right\},
\end{equation*}
as the set ``$\mathcal{G}_{0}$'' (in Definition~\ref{def:geo_erg}) for $\{X_{n}^{\delta}\}$, where $C_{\ell}>0$ is a fixed constant such that $\mathcal{G}_{\ell}$ contains $V(x)$. We have the following results:
\begin{enumerate}[i.]
    \item There exists $a_{2}= a_{2}(\delta)\in (0,a_{1})$ ($a_{1}$ and $d_{1}$ are defined in Assumption~\ref{assu:1}) such that
    \begin{equation*}
    \mathbb{E}^{x}[V(X_{1}^{\delta})] \leq e^{-a_{2} \delta}V(x) + \frac{d_{1}}{a_{2}}, \quad \forall x\in \mathbb{R}^{d}.
    \end{equation*}
    \item There exist $R = R(\ell,\delta)\in(0,+\infty)$ and $D = D(\ell, \delta)\in(0,+\infty)$ such that,
\begin{equation}\label{eq:euler_1}
     \sup_{f\in \mathcal{G}_{\ell}} \left| \mathbb{E}^{x}[f(X_{n}^{\delta})] - \tilde{\pi}^{\delta}(f)\right| \leq Re^{-D n\delta}V(x), \quad \forall x\in\mathbb{R}^{d}, \; \forall n \geq 0.
\end{equation}
\item There exist $K = K(\ell)$ and $\nu \in (0,1/2)$ independent of $\ell$, such that
\begin{equation}\label{eq:euler_2}
   \sup_{f\in \mathcal{G}_{\ell}}\left|\pi(f)- \tilde{\pi}^{\delta}(f)\right| \leq K \delta^{\nu} \pi(V).
\end{equation}
\end{enumerate}
\end{prop}
Here, the set $\mathcal{G}_{\ell}$ is well-defined since by Assumption~\ref{assu:2} and Eq.~\eqref{eq:esse_quad}, one has,
\begin{equation*}
    \|\nabla V(x)\| = \ell W^{\ell-1} \|\nabla W(x)\| \leq \ell C_{2}^{\ell-1}C_{3}(1+\|x\|^{2})^{\ell-1}(1+\|x\|),
\end{equation*}
which leads to
\begin{equation*}
\begin{split}
    |V(x)-V(y)| & \leq \int_{0}^{1} \|\nabla V(sx+ (1-s)y)\| \|x-y\|\td s  \leq \ell C_{2}^{\ell-1}C_{3}(1+\|x\|^2+\|y\|^2)^{\ell-1}(1+\|x\|+\|y\|)\|x-y\| \\
    & \leq  C_{\ell}(1+\|x\|^{2\ell-1} + \|y\|^{2\ell-1})\|x-y\|,
\end{split}
\end{equation*}
for some constant $C_{\ell} = O(2^{2\ell})$. Proposition~\ref{prop:erg_euler} is a direct corollary of Theorem 7.3 in \cite{mattingly2002ergodicity}, whose proof is closely related to the result in \cite{shardlow2000perturbation}.

\comment{\color{blue} (If we restrict ourselves to dissipative systems only, Sections 2.2- 2.4 can be shortened significantly. I will update after our discussion.)
\begin{example}
A typical class of It\^o diffusions that satisfy Assumption~\ref{assu:2} is the dissipative system, that is, there exist constants $a,d\in (0, +\infty)$ such that 
\begin{equation}\label{eq:dissi}
    \langle b(x), x\rangle \leq a - d\|x\|^2,
\end{equation}
where $\langle \cdot, \cdot, \rangle$ always denotes the inner product in Euclidean space. For dissipative system, we can take $V(x) = 1 + \|x\|^{2\ell}$ for every $\ell \geq 1$ as the Lyapunov function of the SDE \eqref{eq:unper}, e.g., Lemma 4.2 in \cite{mattingly2002ergodicity}.
\end{example}

}

\subsection{The long-time linear response theory}\label{sec:lin_resp}

So far, all the perturbation bounds introduced only focus on the one-point statistics given an observable $f\in \mathcal{G}_{0}$ satisfying $|f|\leq V$. However, the set of admissible observables $\mathcal{G}_{0}$ is not general enough for our implementation.  As a remedy, we will review the long-time linear response theory, which is justified in \cite{hairer2010simple} in an abstract setting. This will help us capture the leading order term of the error for more general observables.

We consider a family of Markov evolution operators $\{\mathcal{P}_{t}^{\epsilon}\;|\; t\geq 0,\; \epsilon \in (-\epsilon_0, \epsilon_0)\}$ on $\mathbb{R}^{d}$ that characterize the unperturbed dynamics in \eqref{eq:unper} and its perturbations. We will specify such $\mathcal{P}_{t}^{\epsilon}$ in Section~\ref{sec:per_theory}. Here, to help readers understand the notations, one can interpret the parameter $\epsilon$ as the strength of the perturbation. In other words, when $\epsilon = 0$, $\mathcal{P}_{t}^{0}$ reduces to the evolution operator of the unperturbed dynamics, e.g., the It\^o diffusion \eqref{eq:unper}. Namely,
\begin{equation*}
    \left(\mathcal{P}_{t}^{0}f\right)(x) = \int f(y)P_{t}(x, \td y) = \mathbb{E}^{x}[f(X(t))],
\end{equation*}
where the transition kernel $P_{t}$ is defined in \eqref{eq:tran_kernel}. We are interested in the long-time behavior of the perturbed system described by $\mathcal{P}_{t}^{\epsilon}$ for $\epsilon$ close to 0, which requires the following assumption \cite{hairer2010simple}.

\begin{assu}\label{assu:3}
There exists an $\epsilon_{0}>0$ such that for all $\epsilon\in (-\epsilon_0, \epsilon_0)$, $\mathcal{P}_{t}^{\epsilon}$ yields an invariant probability measure $\pi^{\epsilon}$ on $\mathbb{R}^{d}$.
\end{assu}
When the unperturbed dynamics corresponds to an ergodic It\^o diffusion \eqref{eq:unper}, by Theorem~\ref{thm:ergodic}, we know there exists a unique invariant measure $\pi$ for $\mathcal{P}^{0}_{t}$, that is, $\pi^{0}=\pi$. But for $\epsilon\not = 0$, Assumption~\ref{assu:3} only ensures the existence of the invariant measure. The aim of the long-time linear response theory is to show that the map
\begin{equation*}
    \epsilon\, \mapsto \, \pi^{\epsilon}(f)
\end{equation*}
is differentiable at $\epsilon=0$ for every sufficiently regular observable $f:\mathbb{R}^{d}\rightarrow \mathbb{R}$. We will briefly review the assumptions in \cite{hairer2010simple} that lead to the desirable result.

Let $C_{c}^{\infty}(\mathbb{R}^{d})$ be the set of all smooth functions $f:\mathbb{R}^{d}\rightarrow \mathbb{R}$ that are compactly supported. Given continuous functions $G,H,U: \mathbb{R}^{d} \rightarrow [1, +\infty)$, we set $C^{1}_{G,H}$ to be the closure of $C_{c}^{\infty}(\mathbb{R}^{d})$ under the norm
\begin{equation}\label{eq:GH_norm}
    \|f\|_{1;G,H}: = \sup_{x\in \mathbb{R}^{d}} \left( \frac{|f(x)|}{G(x)} + \frac{\|\nabla f(x)\|}{H(x)}  \right),
\end{equation}
and $C_{U}$ to be the weighted space of continuous functions obtained by completing $C_{c}^{\infty}(\mathbb{R}^{d})$ under the norm
\begin{equation*}
    \|f\|_{U} = \sup_{x\in \mathbb{R}^{d}} \frac{|f(x)|}{U(x)}.
\end{equation*}
The following assumption targets the spectral gap of $\mathcal{P}_{t}^{0}$ as an operator on $C^{1}_{G,H}$.
\begin{assu}\label{assu:4}
There exists a time $t > 0$ and a constant $\lambda \in (0,1)$ such that
\begin{equation*}
    \|\mathcal{P}_{t}^{0}f - \pi(f) \|_{1;G,H} \leq \lambda \| f - \pi(f) \|_{1;G,H}, \quad \forall f\in C_{G,H}^{1}.
\end{equation*}
\end{assu}
{
Assumption~\ref{assu:4} also implies that the invariant measure, $\pi$, of the unperturbed dynamics \eqref{eq:unper} is unique \cite{hairer2010simple}, that is, $\operatorname{Null}(I - (\mathcal{P}_{t}^{0})^{*}) = \operatorname{span}\{\pi\}$, where $(\mathcal{P}_{t}^{0})^{*}$ denotes the adjoint operator of $\mathcal{P}_{t}^{0}$. Therefore, by Fredholm theorem, we have $\operatorname{Range}(I - \mathcal{P}_{t}^{0}) = \operatorname{Null}(I - (\mathcal{P}_{t}^{0})^{*})^{\perp}$, which means for every function $\varphi \in C^{1}_{G,H}$ centered with respect to $\pi$, there exists a unique function $\psi\in C^{1}_{G,H}$ such that
\begin{equation*}
    \psi - \mathcal{P}^{0}_{t} \psi = \varphi,
\end{equation*}
and $\psi$ is also centered with respect to $\pi$. We will henceforth use the notation $\psi = (I - \mathcal{P}^{0}_{t})^{-1}\varphi$. Our next assumption concerns the Fr\'echet derivative of $\mathcal{P}^{\epsilon}_{t}$ with respect to $\epsilon$.}

\begin{assu}\label{assu:5}
Let $C^{1}_{G,H}$ be the same as in Assumption~\ref{assu:4}. There exists a continuous function $U\geq G$ such that, for some fixed $t>0$ and every $f\in C^{1}_{G,H}$, the map $\epsilon \mapsto \mathcal{P}^{\epsilon}_{t}f$, viewed as a map from $(-\epsilon_0, \epsilon_0)$ to $C_{U}$, is differentiable on $(-\epsilon_0, \epsilon_0)$. Denoting this Fr\'echet derivative by $\partial \mathcal{P}^{\epsilon}_{t}$, we furthermore assume that,
\begin{equation*}
    \left\|\partial \mathcal{P}^{0}_{t} f\right\|_{U} \leq C \|f\|_{1;G,H}, \quad \forall f \in C_{G,H}^{1},
\end{equation*}
for some constant $C\in(0,+\infty)$ independent of $f$.
\end{assu}
Finally, we assume that we have an a priori bound on the integrability of the invariant measures. 
\begin{assu}\label{assu:6}
For $U$ in Assumption~\ref{assu:5} and $\pi^{\epsilon}$ in Assumption~\ref{assu:3}, we have
\begin{equation*}
\sup_{\epsilon\in (-\epsilon_0, \epsilon_0)} \pi^{\epsilon}(U) = \sup_{\epsilon\in (-\epsilon_0, \epsilon_0)} \int U(x)\pi^{\epsilon}(\td x) <\infty.
\end{equation*}
\end{assu}
Assumption~\ref{assu:6} ensures that observables in $C_{U}$ yield finite first moments with respect to the invariant measures $\pi^{\epsilon}$ for all $\epsilon \in (-\epsilon_{0}, \epsilon_{0})$. We state the following theorem in \cite{hairer2010simple}.
\begin{theo}\label{theo:lin_resp}
Let $\{\mathcal{P}^{\epsilon}_{t}\;|\; \epsilon\in (-\epsilon_0, \epsilon_0)\}$ be a family of Markov evolution operators over $\mathbb{R}^{d}$ such that there exist $C^{1}$ functions $U,G,H: \mathbb{R}^{d}\rightarrow [1,+\infty)$ such that Assumptions~\ref{assu:3}-\ref{assu:6} hold for some fixed $t>0$. Then, the map $\epsilon\rightarrow \pi^{\epsilon}(f)$ is differentiable at $\epsilon=0$ for all $f\in C^{1}_{G,H}$. In particular, we have
\begin{equation}\label{eq:lin_resp1}
    \frac{\td }{\td\; \epsilon} \pi^{\epsilon}(f) \Big|_{\epsilon =0} = \mathbb{E}_{\pi}\left[ \partial\mathcal{P}^{0}_{t}  (I - \mathcal{P}^{0}_{t})^{-1} \left(f - \pi(f) \right)\right],
\end{equation}
where the right-hand side, as an expectation with respect to the invariant measure $\pi = \pi^{0}$, is well-defined.
\end{theo}
Using Theorem~\ref{theo:lin_resp}, we can capture the leading order term of the error $|\pi^{\epsilon}(f) - \pi(f)|$ for $f\in C_{G;H}^{1}$. Although the result in Theorem~\ref{theo:lin_resp} is observable-dependent, unlike the error bounds reviewed in Section~\ref{sec:per_MC}, where the inequalities act as uniform bounds for a class of observables, Eq.~\eqref{eq:lin_resp1} can be applied to more general observables, e.g., observables not controlled by the Lyapunov function $V$. To some extent, we trade the uniformity for generality. 

Before we close the section, we would like to point out that the linear response theory can also be applied to the short-time response of the dynamics subject to perturbations \cite{Pav_book:14}. We list Theorem~\ref{theo:lin_resp} in this  paper to prove Proposition~\ref{prop:two_point} in Section~\ref{sec:two_point}.

\section{Error bounds of the invariant statistics in learning ergodic It\^o diffusions}\label{sec:per_theory}

In Section~\ref{sec:review}, we have reviewed the concept of geometrically ergodic Markov chain and discussed a series of results and bounds for ergodic It\^o diffusions in a relatively abstract setting. In this section, we reformulate the perturbation theory as a problem in the context of learning the dynamical system \eqref{eq:unper}. In particular, we will specify the perturbation as the error induced by the statistical learning (Section~\ref{sec:learning_SDE}) and develop error bounds for both one-point statistics (Section~\ref{sec:one_point}) and two-point statistics (Section~\ref{sec:two_point}) based on the results in Section~\ref{sec:review}.

\subsection{Learning ergodic It\^o diffusions} \label{sec:learning_SDE}

Interpreting the system \eqref{eq:unper} as the unperturbed dynamics, under the same initial condition, we introduce a family of perturbed dynamics of the form,
\begin{equation}\label{eq:per}
\dot{X}^{\epsilon}(t) =  b_{\epsilon}(X^{\epsilon}(t)) + \sigma_{\epsilon} \dot{W}, \quad X_{0}^{\epsilon} = x, \quad 0<\epsilon \ll 1,
\end{equation}
where $b_{\epsilon}: \mathbb{R}^{d}\rightarrow \mathbb{R}^{d}$ and $\sigma_{\epsilon}\in \mathbb{R}^{d\times m}$, the perturbed drift and diffusion coefficients, respectively, are parameterized by a parameter $\epsilon>0$ corresponding to the ``scale'' of the perturbation. For simplicity, we assume $W$ in \eqref{eq:per} to be the same standard $m$-dimensional Brownian motion as in the unperturbed dynamics \eqref{eq:unper}. We should point out that since we are interested in the error bound of the invariant statistics (rather than the pathwise error between $X(t)$ and $X^{\epsilon}(t)$), the perturbed diffusion coefficients $\sigma_{\epsilon}$ is defined so that $\sigma_{\epsilon}\sigma_{\epsilon}^{\top}$ is an estimate of $\sigma\sigma^{\top}$. To gain an intuition of the perturbed dynamics in \eqref{eq:per} and develop a proper interpretation of the parameter $\epsilon$, we introduce the following regression problem in estimating the drift coefficient $b$ of the unperturbed dynamics \eqref{eq:unper}.

Consider the Markov chain, $\{X_{n}^{\delta}\}$, generated by EM discretization of the unperturbed dynamics \eqref{eq:unper}. Based on the numerical scheme in \eqref{eq:Euler_approx}, we define the finite difference process,
\begin{equation}
    Y_{n}^{\delta}: = \frac{1}{\delta}\left(X_{n+1}^{\delta} - X_{n}^{\delta}\right) = b(X_{n}^{\delta}) + \delta^{-\frac{1}{2}} \sigma \xi_{n}.\label{discretesupervisedmodel}
\end{equation}
Since $\xi_{n}$ is independent of $X_{n}^{\delta}$, we may express the drift coefficient $b$ as the following conditional expectation,
\begin{equation}\label{eq:b=EX}
    b(x) = \mathbb{E}\left[Y_{n}^{\delta} \;\big|\; X_{n}^{\delta} = x  \right], \quad \forall n \geq 0.
\end{equation}
Let $\mu^\delta$ denote the joint stationary distribution of the random variable $(X,Y):=(X_{n}^{\delta}, Y_{n}^{\delta})$. Eq.~\eqref{eq:b=EX} suggests that the solution to the following regression problem
\begin{equation}
  \min_{h = (h_{1},\dots, h_{d})^{\top}} \mathcal{E}[h],\quad \mathcal{E}[h]:= \mathbb{E}_{\mu^\delta}\Big[ \left\|h(X) - Y \right\|^{2}\Big], \quad h_{i}\in L^2(\mathbb{R}^{d},\tilde{\pi}^{\delta}), \quad i = 1,2\dots, d,
  \label{general_regression}
\end{equation}
is an unbiased estimator of $b$ \cite{cucker2002mathematical}. Furthermore, the covariance matrix of the residual error satisfies
\begin{equation*}
   \mathbb{E}_{\mu^\delta} \Big[(b(X ) - Y)(b(X) - Y)^{\top}\Big] = \delta^{-1} \sigma\sigma^{\top}.
\end{equation*}
Since the noise has independent components,  the residual error  is given by $\mathcal{E}[b]=\delta^{-1}\operatorname{Tr}[\sigma\sigma^\top]$, where $\operatorname{Tr}[\cdot]$ denotes the standard matrix trace operation. While this estimator is unbiased, the bias (also known as the approximation error \cite{cucker2002mathematical}) may appear depending on the choice of the hypothesis space $H$, which will be clarified in Sections~\ref{sec:application}- \ref{sec:application_ML}.

In practice, the cost function $\mathcal{E}[h]$ in \eqref{general_regression} is approximated by an empirical cost function, 
\BEA\mathcal{E}_{N}[h]:= \frac{1}{N} \sum_{i=1}^N \left\| h(x_i) -y_i \right\|^2,\label{empiricalcost}\EEA 
from  i.i.d. samples $\left\{x_i,y_i\right\}_{i=1}^N$ of $(X,Y)$ with stationary distribution $\mu^\delta$. Here, we consider i.i.d. samples only for the convenience of the theoretical analysis in Sections~\ref{sec:application} and \ref{sec:application_ML}.
Practically, the samples can be obtained by subsampling from the labelled time series $\{x_n^\delta,y_n^\delta\}_{n\geq 0}$ to reduce the temporal correlation, and thus the sampling error.

We define $b_{\epsilon}$ and $\sigma_{\epsilon}\sigma_{\epsilon}^{\top}$ as follows,
\begin{equation}\label{eq:sig_e_1}
    b_{\epsilon}: = \arg\min_{h\in H} \mathcal{E}_{N}[h], \quad \sigma_{\epsilon}\sigma_{\epsilon}^{\top}: = \frac{\delta}{N} \sum_{i=1}^N\left[ \left(y_i - {b}_{\epsilon}(x_{i})\right) \left(y_{i} - {b}_{\epsilon}(x_{i})\right)^{\top}  \right].
\end{equation}
One can see that the sample covariance is a biased estimator of $\sigma \sigma^{\top}$, that is,
\BEA
\mathbb{E}_{\mu^\delta} [\sigma_{\epsilon}\sigma_{\epsilon}^{\top}] - \sigma\sigma^\top =\delta \mathbb{E}_{\tilde{\pi}^{\delta}}[(b(X) - b_\epsilon(X))(b(X) - b_\epsilon(X))^\top], \nonumber
\EEA
where, by Jensen's inequality, the bias satisfies
\begin{equation}\label{eq:sig_e_2}
    \left\| \mathbb{E}_{\tilde{\pi}^{\delta}}\left[(b(X) - b_\epsilon(X))(b(X) - b_\epsilon(X))^\top\right]\right\|_{F} \leq  \mathbb{E}_{\tilde{\pi}^{\delta}}\left[\left\| (b(X) - b_\epsilon(X))(b(X) - b_\epsilon(X))^\top\right\|_{F} \right] =  \mathbb{E}_{\tilde{\pi}^{\delta}}\left[\left\|(b(X) - b_\epsilon(X))\right\|^2\right].
\end{equation}
To have a better understanding of the error between $\sigma_{\epsilon}\sigma_{\epsilon}^{\top}$ and $\sigma \sigma^{\top}$, we introduce
\begin{equation}
D_i := \delta(y_i - {b}_{\epsilon}(x_i))(y_{i} - {b}_{\epsilon}(x_{i}))^\top- \sigma\sigma^\top -\delta \mathbb{E}_{\tilde{\pi}^{\delta}}[(b(X) - b_\epsilon(X))(b(X) - b_\epsilon(X))^\top], \quad i = 1,2,\dots, N,\label{rvD}
\end{equation}
which defines a finite sequence of independent, random, symmetric matrices of mean $0$. In particular, subtracting $\sigma_{\epsilon}\sigma_{\epsilon}^{\top}$ \eqref{eq:sig_e_1} from $\sigma\sigma^\top$ and using the definition in~\eqref{rvD}, we deduce that,
\begin{eqnarray}\label{eq:sig_e_3}
    \left\|\sigma\sigma^{\top} - \sigma_{\epsilon}\sigma_{\epsilon}^{\top}  \right\|_{2} &=&  \left\|\frac{1}{N}\sum_{i=1}^{N} D_i  + 
    \delta \mathbb{E}_{\tilde{\pi}^{\delta}}\left[(b(X) - b_\epsilon(X))(b(X) - b_\epsilon(X))^\top\right] \right\|_{2}  \notag\\
    & \leq & \left\|\frac{1}{N}\sum_{i=1}^{N} D_i   \right\|_2 + \delta \left\|\mathbb{E}_{\tilde{\pi}^{\delta}}\left[(b(X) - b_\epsilon(X))(b(X) - b_\epsilon(X))^\top\right] \right\|_{2} \notag\\ 
    &\leq & \left\|\frac{1}{N}\sum_{i=1}^{N} D_i   \right\|_2 + \delta\mathbb{E}_{\tilde{\pi}^{\delta}}\left[\left\|(b(X) - b_\epsilon(X))\right\|^2\right],
\end{eqnarray}
where $\|\cdot\|_2$ denotes the matrix $2$-norm. Here, we have used the relation in \eqref{eq:sig_e_2} and the fact that $\|A\|_2 \leq \|A\|_{F}$ for any matrix $A$. The first term on the right-hand side of Eq.~\eqref{eq:sig_e_3} is the error induced by the empirical estimation, can be bounded by the matrix Bernstein inequality, e.g., Theorem 6.2 in \cite{tropp2012user}, assuming that $\tilde{\pi}^{\delta}$ belongs to the sub-exponential class. Ignoring the parameter $\delta$, we will refer the second term as the generalization error of the learning algorithm of $b$. We will provide detailed discussions of this error term in Sections~\ref{sec:application}-\ref{sec:application_ML} for specific learning methods. Based on these observations, we define the parameter $\epsilon$, which corresponds to the ``scale'' of the perturbation, as the spectral error of the diffusion matrix estimator, $\sigma_\epsilon\sigma_\epsilon^\top$,
\begin{equation}\label{eq:epsilon_def}
    \epsilon: = \|\sigma\sigma^{\top} - \sigma_{\epsilon}\sigma_{\epsilon}^{\top} \|_{2},
\end{equation}
which is well-defined and is small in high probability for large enough $N$.


Our goal is to analyze the error bounds of the invariant statistics when the underlying ergodic It\^o diffusion \eqref{eq:unper} is approximated by the perturbed dynamics in \eqref{eq:per}. In particular, we would like to understand how the error in the invariant statistics depends on $\epsilon$ in \eqref{eq:epsilon_def} and the generalization error, $\mathbb{E}_{\tilde{\pi}^{\delta}}\left[\left\|(b(X) - b_\epsilon(X))\right\|^2\right]$. To develop results toward this direction, we need the following critical assumptions on the family of coefficients $\{b_{\epsilon}\}$ in \eqref{eq:per} and their errors $\{b_{\epsilon}-b\}$. These assumptions not only elucidate the dependence on $\epsilon$ but also  are also conditions that we need to implement the results reviewed in Section~\ref{sec:review}.

\begin{assu}\label{assu:coe_per}
Consider the unperturbed and the family of perturbed It\^o diffusions in \eqref{eq:unper} and \eqref{eq:per}, respectively. For any $0<\epsilon \ll 1$ ($\epsilon$ defined in \eqref{eq:epsilon_def}), we assume
\begin{enumerate}[i.]
    \item The coefficient $b_{\epsilon}$ is Borel measurable and satisfies the globally Lipschitz condition as in Assumption~\ref{assu:Ito_coef} with uniform Lipschitz constant with respect to $\epsilon$.
    \item The family of coefficients $\{b_{\epsilon}\}$ is a sequence of consistent estimators of $b$ in the space of continuous (vector-valued) functions of linear growth. That is,
    \begin{equation}\label{eq:error_linear}
        \|b_{\epsilon}(x)-b(x)\|^2 \leq  K^2_{3}(1+\|x\|^2)\epsilon^2, \quad \forall x\in \mathbb{R}^{d},
    \end{equation}
  for some constant  $K_{3} \in (0, +\infty)$ independent of $\epsilon$.
\end{enumerate}
\end{assu}

Assumption~\ref{assu:coe_per}(i) ensures that the coefficients $b_{\epsilon}$ in the perturbed dynamics \eqref{eq:per} satisfy Assumption~\ref{assu:Ito_coef} with related constants, including the Lipschitz constants and those constants in the linear growth bound, independent of $\epsilon$. The condition~\eqref{eq:error_linear} suggests that not only $b_\epsilon \to b$ as $\epsilon\to 0$ on a weighted continuous function space, but it also suggests the following scaling on the generalization error,
\begin{equation}
\mathbb{E}_{\tilde{\pi}^{\delta}}\left[\left\|(b(X) - b_\epsilon(X))\right\|^2\right] \leq K_{3}^{2} \mathbb{E}_{\tilde{\pi}^{\delta}}\left[1 + \left\| X\right\|^2\right] \epsilon^2 = O(\epsilon^2).\label{generalizationerror}
\end{equation}
This implies that whenever the generalization error is scaled as in \eqref{generalizationerror}, which will always be satisfied when the condition \eqref{eq:error_linear} is valid, the spectral error of the diffusion matrix estimator (with error bound in \eqref{eq:sig_e_3}) is dominated by the Monte-Carlo error that can be controlled by concentration inequalities, instead of the generalization error in learning $b$.

To deduce the result below, we adopt the notation in Section~\ref{sec:review}. For $\delta>0$, let $X_{n}^{\epsilon}$ and $X_{n}^{\epsilon,\delta}$ denote the Markov chain sampled from the perturbed dynamics \eqref{eq:per} with $X_{n}^{\epsilon} = X^{\epsilon}(n\delta)$ and the discretized Markov chain of \eqref{eq:per} generated by EM scheme \eqref{eq:Euler_approx} with step size $\delta$, respectively. Here, $\{X_{n}^{\epsilon}\}$ and $\{X_{n}^{\epsilon,\delta}\}$ can be interpreted as the perturbed Markov chains of $X_{n}$ and $X_{n}^{\delta}$, respectively.  To employ the existing theory reviewed in Section~\ref{sec:review}, recall that the Proposition~\ref{prop:per_bound} in Section~\ref{sec:per_MC} involves an unperturbed Markov chain $\{u_{n}\}$ and the corresponding perturbed Markov chain $\{u_{n}^{\epsilon}\}$, where $\{u_{n}\}$ is assumed to be geometrically ergodic. While it is obvious that we are fundamentally interested in the case where $(u_{n}, u_{n}^{\epsilon}) = (X_{n}, X_{n}^{\epsilon})$, in practice, we are rarely given a realization of $\{X_{n}\}$. What is usually available is the time series of  $\{X_{n}^{\delta}\}$, obtained e.g, via EM integrator. Given such constraints, we consider also the case where
$(u_{n}, u_{n}^{\epsilon}) = (X_{n}^{\delta}, X_{n}^{\epsilon,\delta})$. By Assumption~\ref{assu:1} and Proposition~\ref{prop:erg_euler}, we know that the Markov chain $\{X_{n}^{\delta}\}$ is geometrically ergodic and the condition related to the Lyapunov function \eqref{eq:Exp_Lyp} is valid for both two cases. Thus, it is enough to derive the relation between the $\gamma$ in \eqref{eq:V_norm} and the parameter $\epsilon$, which is given by the following lemma.

\begin{lem}\label{lem:well_pose}
Consider the unperturbed ergodic It\^o diffusion \eqref{eq:unper} and the corresponding perturbed dynamics \eqref{eq:per} satisfying the Assumptions~\ref{assu:Ito_coef}-\ref{assu:1} and Assumption~\ref{assu:coe_per}. For the fixed step size $\delta$ specified in Theorem~\ref{thm:ergodic} and all $0<\epsilon\ll 1$, we have
\begin{eqnarray*}
\gamma_{X} &:=& \sup_{x\in \mathbb{R}^{d}} \sup_{f\in \mathcal{G}_{\ell}}  \frac{\left| \mathbb{E}^{x}[f(X_{1})] - \mathbb{E}^{x}[f(X^{\epsilon}_1)]\right|}{V(x)}\leq K\epsilon ,\nonumber \\
\gamma_{X^{\delta}} &:=& \sup_{x\in \mathbb{R}^{d}} \sup_{f\in \mathcal{G}_{\ell}}  \frac{\left| \mathbb{E}^{x}[f(X_{1}^{\delta})] - \mathbb{E}^{x}[f(X^{\epsilon,\delta}_1)]\right|}{V(x)}\leq K^{\delta}\epsilon,
\end{eqnarray*}
for some constant $K, K^{\delta}\in (0, +\infty)$ that are independent of $\epsilon$. Here, $\mathcal{G}_{\ell}$ is defined as in Proposition~\ref{prop:erg_euler}.
\end{lem}

See Appendix~\ref{app:gron} for the proof. We should point out that in the course of the proof, the linear scaling in \eqref{eq:error_linear} is important for balancing the $\epsilon$ scaling that is defined as the spectral error of the estimator $\sigma_\epsilon\sigma_\epsilon^\top$ in \eqref{eq:epsilon_def}. The proof also suggests that while the scaling of generalization error in \eqref{generalizationerror} is necessary for small spectral error bound in \eqref{eq:epsilon_def} as we pointed out right after \eqref{generalizationerror}, it is not a sufficient condition to achieve the bounds in Lemma~\ref{lem:well_pose} which serves as the backbone for the main results in the following two subsections.

\subsection{One-point statistics}\label{sec:one_point}

For one-point statistics, we consider an observable $f:\mathbb{R}^{d} \rightarrow \mathbb{R}$ with finite first moment, $\pi(|f|) < \infty$, and try to derive error bounds for $  \big| \mathbb{E}^{x}[f(X^{\epsilon}_{n})] - \pi(f) \big|$ and $\big| \mathbb{E}^{x}[f(X^{\epsilon,\delta}_{n})] - \pi(f) \big|$, where $\pi$ denotes the invariant measure of the unperturbed dynamics \eqref{eq:unper}. By Lemma~\ref{lem:well_pose}, if the observable $f\in \mathcal{G}_{\ell}$ (defined in Proposition~\ref{prop:erg_euler}), such bounds can be derived immediately from Proposition~\ref{prop:per_bound} and it is uniform for all $f\in \mathcal{G}_{\ell}$. In particular, the following proposition summarizes the corresponding results.

\begin{prop}\label{prop:euler_inv_bound}
Under the same circumstances as in Lemma~\ref{lem:well_pose}, consider the invariant measure of the unperturbed dynamics \eqref{eq:unper}, $\pi$, and the discretized Markov chain generated by the EM scheme \eqref{eq:Euler_approx}, $\{X_{n}^{\epsilon,\delta}\}$. Then, we have
\begin{equation*}
    \sup_{f\in \mathcal{G}_{\ell}} \left| \mathbb{E}^{x}[f(X^{\epsilon}_{n})] - \pi(f) \right| \leq R_1 \left[\left(\rho_{1}^{n} + \frac{1-\rho_{1}^{n}}{1-\rho_{1}}\epsilon\right)V(x) \right], \quad \forall n\geq 0,
\end{equation*}
and
\begin{equation*}
    \sup_{f\in \mathcal{G}_{\ell}} \left| \mathbb{E}^{x}[f(X^{\epsilon,\delta}_{n})] - \pi(f) \right| \leq R_2 \left[\left(\rho_2^{n} + \frac{1-\rho_{2}^{n}}{1-\rho_{2}}\epsilon\right)V(x) + \delta^{\nu}\pi(V)\right],  \quad \forall n\geq 0,
\end{equation*}
for some constants $R_{1}, R_{2}\in (0,+\infty)$, $\rho_{1}, \rho_{2}\in(0,1)$, and $\nu \in (0, \frac{1}{2})$ (same as in Eq.~\eqref{eq:euler_2}).
\end{prop}

\begin{proof}
Both inequalities are results of Eq.~\eqref{eq:long_time}, which is a corollary of Proposition~\ref{prop:per_bound}.  By Lemma~\ref{lem:well_pose}, we know Proposition~\ref{prop:per_bound} holds for both $(u_{n}, u_{n}^{\epsilon}) = (X_{n}, X_{n}^{\epsilon})$ and $(u_{n}, u_{n}^{\epsilon}) = (X_{n}^{\delta}, X_{n}^{\epsilon,\delta})$.

To obtain the first inequality, we apply Eq.~\eqref{eq:long_time} to $X_{n}^{\epsilon}$ (with $\pi$ being the invariant measure of $X_{n}$). As for the second inequality, we apply Eq.~\eqref{eq:long_time} to $X_{n}^{\epsilon,\delta}$ (with $\pi = \tilde{\pi}^{\delta}$ being the invariant measure of $X_{n}^{\delta}$ defined as in Proposition~\ref{prop:erg_euler}). The desired bounds can be obtained using  the triangular inequality together with Eq.~\eqref{eq:euler_2} in Proposition~\ref{prop:erg_euler}. Here, we have used that fact that $V(x)\in [1,+\infty)$ to replace the term $\kappa$ in Eq.~\eqref{eq:long_time} by $V(x)$.
\end{proof}

Essentially, this Proposition suggests that the error in one-point statistics depends linearly on the perturbation size, $\epsilon$, as $n\to\infty$. 
In practice, since the available data for training is subjected to numerical discretization error, the error rate is of order $O(\epsilon, \delta^\nu),$ from some $\nu\in(0,1/2)$ specified in Proposition~\ref{prop:erg_euler}.

\subsection{Two-point statistics}\label{sec:two_point}
For two-point statistics, we consider observables $A, B: \mathbb{R}^{d} \rightarrow \mathbb{R}$ with finite second moments, that is, $\pi(A^{2}), \pi(B^{2}) < \infty$.
The two-point statistics of $A$ and $B$ for the unperturbed dynamics \eqref{eq:unper} is defined as
\begin{equation}\label{eq:two_point}
\begin{split}
k_{A,B}(t)&:= \mathbb{E}_{\pi}\left[A(X(t))B(X(0))\right] = \int\int A(x)B(x_0)P_{t}(x_0, \td x)\pi(\td x_{0}).
\end{split}
\end{equation}
In applications, the two-point statistics can arise from Fluctuation-Dissipation Theory and is a route to approximate the statistics of a system driven out of equilibrium, e.g., \cite{harlim2017parameter, leith1975climate,majda2005information}. In \cite{zhang2019linear}, we have shown that the two-point statistics in \eqref{eq:two_point} are well-defined for all $t\geq 0$. Formally, under Assumption~\ref{assu:3}, we can define the corresponding two-point statistics for the perturbed dynamics \eqref{eq:per}, as an approximation of \eqref{eq:two_point}. It is given by
\begin{equation*}
k^{\epsilon}_{A,B}(t):= \mathbb{E}_{\pi^{\epsilon}}\left[A(X^{\epsilon}(t))B(X^{\epsilon}(0))\right]= \int\int A(x)B(x_0)P^{\epsilon}_{t}(x_0, \td x)\pi^{\epsilon}(\td x_{0}).
\end{equation*}
Here, $P_{t}^{\epsilon}$ denotes the transition kernel \eqref{eq:tran_kernel} of the perturbed dynamics. In terms of the Markov chains $X_{n}$ and $X_{n}^{\epsilon}$, the two-point statistics reduce to 
\begin{equation}\label{eq:two_point_em}
(k_{A,B})_{n}:= k_{A,B}(n\delta) = \mathbb{E}_{\pi}[A(X_{n})B(X_{0})], \quad (k^{\epsilon}_{A,B})_{n}:= k^{\epsilon}_{A,B}(n\delta) = \mathbb{E}_{\pi^{\epsilon}}[A(X^{\epsilon}_{n})B(X^{\epsilon}_{0}))], \quad \forall n\geq 0.
\end{equation}
The following proposition, as the main result of this section, provides an error bound for the two-point statistics.
\begin{prop}\label{prop:two_point}
Under the same circumstances as in Lemma~\ref{lem:well_pose}, for a fixed step size $\delta>0$, let Assumptions~\ref{assu:3}-\ref{assu:6} hold for the family of Markov operators,
\begin{equation*}
    \left(\mathcal{P}_{\delta}^{\epsilon}f\right)(x) :=\int f(y)P^{\epsilon}_{\delta}(x, \td y) = \mathbb{E}^{x}[f(X^{\epsilon}(\delta))], \quad  0 <\epsilon \ll 1,
\end{equation*}
induced by the perturbed dynamics \eqref{eq:per} with respect to the functions $G,H,U$ (see Theorem~\ref{theo:lin_resp} for the details) satisfying $G\geq V$. We further assume
\begin{equation*}
    \pi^{\epsilon}(V^2) <\infty,
\end{equation*}
where $V$ is the Lyapunov function of the unperturbed dynamics. Then, for any  observables $A\in \mathcal{G}_{\ell}\cap C^{1}_{V,VH/G}$ and $B\in  C^{1}_{G/V,H/V}$ satisfying
\begin{equation*}
  \pi(A^2),\;\pi(B^2),\;     \pi^{\epsilon}(A^2), \; \pi^{\epsilon}(B^2) < \infty,
\end{equation*}
the two-point statistics in \eqref{eq:two_point_em} are well-defined, and
\begin{equation}\label{eq:two_point_error}
\left|(k^{\epsilon}_{A,B})_{n} - (k_{A,B})_{n}\right|  \leq R\pi^{\epsilon}(B^{2})^{\frac{1}{2}}\epsilon +  \left|\frac{\td }{\td\; \epsilon} \pi^{\epsilon}(f_{n})\Big|_{\epsilon =0} \right|\epsilon + O(\epsilon^2), \quad f_{n}(x) = B(x)\mathbb{E}^{x}[A(X_{n})],
\end{equation}
for some constants $R\in (0,+\infty)$ independent of $A$ and $B$. Here, the $\epsilon$-derivative in \eqref{eq:two_point_error} is well-defined and satisfies Eq.~\eqref{eq:lin_resp1} with $f = f_n$ and $t= \delta$. Moreover, if we further assume that $A$ is centered with respect to $\pi$ ($\pi(A)=0$), the error bound in \eqref{eq:two_point_error} satisfies
\begin{equation}\label{eq:cen_error}
\left|(k^{\epsilon}_{A,B})_{n} - (k_{A,B})_{n}\right|  \leq R\left[\pi^{\epsilon}(B^{2})^{\frac{1}{2}} + \frac{1}{1-\lambda}\|B\|_{1;G/V,H/V} \left( \lambda^{n}\|A\|_{1;V,VH/G} + \rho^{n}\pi(G)\right)  \right]\epsilon + O(\epsilon^2),
\end{equation}
for some constants $R\in(0,+\infty)$, $\lambda\in(0,1)$ ( the same as those in Assumption~\ref{assu:4}), and $\rho\in (0,1)$ (the same as in Theorem~\ref{thm:ergodic}) independent of $A$ and $B$. The norms in \eqref{eq:cen_error} are defined by Eq.~\eqref{eq:GH_norm}.
\end{prop}

See Appendix~\ref{App:two_point} for the proof. In the proof we analyze the two factors that contribute to the error of the two-point statistics: the transition kernel and the invariant measure, which lead to the first and second terms on the right-hand side of the error bound \eqref{eq:two_point_error}, respectively. When the observable $A$ is centered with respect to $\pi$, suggested by Eq.~\eqref{eq:cen_error}, the leading order term of the error caused by the invariant measure, goes to zero as $n\rightarrow +\infty$.

Proposition~\ref{prop:two_point} states the error of the two-point statistics between the two sampling Markov chains $\{X_{n}\}$ and $\{X_{n}^{\epsilon}\}$. Similar to Lemma~\ref{lem:well_pose} and Proposition~\ref{prop:euler_inv_bound}, one can extend the error bounds in Proposition~\ref{prop:two_point} corresponding to the two discretized Markov chains: $\{X_{n}^{\delta}\}$ and $\{X_{n}^{\epsilon, \delta}\}$, where the two-point statistics are defined with respect to the discrete transition kernel. For example, the two-point statistics of $\{X_{n}^{\delta}\}$ are defined as 
\begin{equation*}
    \left( k_{A,B}^{\delta}\right)_{n}: = \mathbb{E}_{\tilde{\pi}^{\delta}} \left[A(X_{n}^{\delta})B(X_{0}^{\delta})\right] = \int\int A(x)B(x_0) (P^{\delta})^{n}(x_0, \td x) \tilde{\pi}^{\delta}(\td x), \quad \forall n\geq 0,
\end{equation*}
where $A,B \in L^{2}(\mathbb{R}^{d}, \tilde{\pi}^{\delta})$ and $P^{\delta}$ denotes the transition kernel of $\{X_{n}^{\delta}\}$. Here, $(P^{\delta})^{n}$ corresponds to the product of the transition kernel, that is,
\begin{equation*}
    (P^{\delta})^{n}(x, A) = \int (P^{\delta})^{n-1}(x, \td y) P^{\delta}(y, A), \quad \forall A \in \mathcal{B}(\mathbb{R}^{d}), \quad \forall n \ge 1.
\end{equation*}

\begin{remark}
Proposition~\ref{prop:two_point} is valid under a series of assumptions, which can be classified into the following three categories.
\begin{enumerate}[i.]
    \item \textbf{Assumptions on the unperturbed dynamics:} Assumptions~\ref{assu:Ito_coef}-\ref{assu:2} are proposed to ensure the geometric ergodicity of the unperturbed dynamics and its numerical discretization.
    \item \textbf{Assumptions on the perturbed dynamics:} Assumption~\ref{assu:coe_per} is postulated so that the perturbed dynamic is a ``reasonable'' approximation of the unperturbed dynamics (Lemma~\ref{lem:well_pose}). Assumptions~\ref{assu:3}-\ref{assu:6} are proposed so that the long-time linear response theory holds (Theorem~\ref{theo:lin_resp}).
    \item \textbf{Assumptions on the observables}: We assume the observables $A$ and $B$ satisfy certain regularity and integrability conditions, so that the two-point statistics in \eqref{eq:two_point_em} are well-defined and the long-time linear response theory is applicable to $f_{n}$ in \eqref{eq:two_point_error}. The regularity assumptions of $A$ and $B$ are not identical. Specifically, $A\in C^{1}_{V,VH/G}$ while $B\in C^1_{G/V, H/V}$.
        In particular, when $G\geq V^2$, we have
    \begin{equation*}
        G/V \geq V, \quad H/V \geq VH/G,
    \end{equation*}
    which imply that $C^{1}_{V,VH/G}\subset C^{1}_{G/V, H/V}$, that is, the regularity assumption of $B$ is weaker than that of $A$. In practice, it is common that the observable $B$ is not as regular as the observable $A$. For example, in the Fluctuation-Dissipation theory (FDT), the linear response operator defines two-point statistics of the form in \eqref{eq:two_point} between the observable $A$ and the conjugate variable $B$ \cite{zhang2020estimating}. The conjugate variable $B$, produced by a differential operator (a typical example is $B =\nabla \log(\pi)$ \cite{zhang2019parameter}), is often less regular than $A$.
\end{enumerate}

\end{remark}

Among all the assumptions, the assumptions associated with the long-time linear response theory (Assumptions~\ref{assu:3}-\ref{assu:6}) are somewhat abstract and difficult to be directly verified. To provide some insight, we give an example here to show how the long-time linear response theory is applied to a class of It\^o diffusions and their perturbations.

\begin{example}
Consider an It\^o diffusion \eqref{eq:unper} satisfying Assumptions~\ref{assu:Ito_coef}-\ref{assu:0}. If the drift coefficients $b$ in \eqref{eq:unper} satisfies the dissipative condition \cite{mattingly2002ergodicity}, that is, there exist constants $a,d \in (0, +\infty)$ such that
\begin{equation*}
    \langle b(x), x \rangle \leq a  - d \|x\|^2, \quad \forall x\in \mathbb{R}^{d},
\end{equation*}
then, the Assumptions~\ref{assu:1}-\ref{assu:2} hold with $V(x) = 1 + \|x\|^{2\ell}$ for all $\ell\geq 1$. Moreover, consider the family of Markov operators $\{\mathcal{P}_{t}^{\epsilon}\}$ induced by the perturbed dynamics in \eqref{eq:per} with coefficients satisfy Assumption~\ref{assu:coe_per}, then Assumptions~\ref{assu:3}-\ref{assu:6} hold for all $\epsilon$ small enough with
\begin{equation*}
    G(x) = 1 + \alpha e^{\eta \|x\|^2}, \quad H(x) = \beta^{-1} e^{\eta \|x\|^2}, \quad U(x) = G(x) + e^{2\eta \|x\|^2},
\end{equation*}
for some positive constants $\alpha$, $\beta$, and $\eta$ sufficiently small.
\end{example}

The choice of the Lyapunov function is a result of Lemma 4.2 in \cite{mattingly2002ergodicity}. Since the perturbation to the drift coefficients $b_{\epsilon}-b$ is of linear growth (Assumption~\ref{assu:coe_per}), for $\epsilon$ small enough, the drift coefficients $b_{\epsilon} = b + (b_{\epsilon}-b)$ of the perturbed dynamics still satisfies the dissipative condition. Under such observations, the choice of functions $G,H$ and $U$ can be found in the proof of Theorem 4.4 in \cite{hairer2010simple}. In \cite{hairer2010simple}, the authors considered the case where the drift coefficients $b$ is a linear combination of symmetric multi-linear maps, so that they can apply the long-time linear response theory to the family of Markov operators parameterized by the parameters in the SDEs. In our situation, the family of Markov operators is parameterized by $\epsilon$ in the perturbed dynamics. With the key ingredient, the dissipative condition, being preserved, the proof in \cite{hairer2010simple} is still valid. It is worthwhile to mention that, given the dissipative condition, Assumption~\ref{assu:0} (i) can be replaced by the assumption that the dynamics \eqref{eq:unper} is approximately controllable \cite{mattingly2002ergodicity,hairer2010simple}. As for the existence of the invariant measure $\pi^{\epsilon}$, one can also consider the stationary Fokker–Planck equation of \eqref{eq:per}. In our case, since the noise may be degenerate, the existence of $\pi^{\epsilon}$ can be established by results in \cite{huang2016steady} with the help of the Lyapunov function.

\section{Learning with the kernel-based spectral regression method  }\label{sec:application}

In Section~\ref{sec:per_theory}, we have deduced a linear dependence of the error of the invariant statistics to the error in the estimation of the drift and diffusion coefficients. Our primary interest here (and in Section~\ref{sec:application_ML}) is to understand under which conditions can the results in Section~\ref{sec:per_theory} be achieved when existing popular machine learning methods are used to estimate the coefficients $b$ and $\sigma\sigma^\top$. Here, we will focus on the kernel-based spectral regression method \cite{rosasco2010learning}, whereas, in the next section, we will focus on the single hidden-layer random neural networks with the ReLU activation functions \cite{gonon2020approximation}.

In Section~\ref{sec:learning_SDE}, we have formally proposed a framework for learning ergodic SDEs by solving a regression problem and introduced the concept of the generalization error, which will be further decomposed into two parts: estimation error (error caused by sampling) and approximation error (error due to the choice of hypothesis space). In particular, we have pointed out that the coefficients $(b_{\epsilon}, \sigma_{\epsilon}\sigma_\epsilon^\top)$ in the approximated system \eqref{eq:per} are the empirical estimates defined through \eqref{eq:sig_e_1} for a given sample. Using the notation in \eqref{general_regression}, our goal now is to quantify the generalization error defined in \eqref{generalizationerror} induced by $b_\epsilon$, an estimator obtained from training on a set of i.i.d. labelled data, sampled from $(X,Y)$ generated by \eqref{discretesupervisedmodel}. Through the assumed scaling in \eqref{generalizationerror}, this analysis elucidates how the parameter $\epsilon$ depends on the training sample size, the parameters in the hypothesis space, noise amplitude parameter, and step size for the time discretization. 

To simplify the discussion, we will perform the analysis component-wise. Abusing the notation, we refer $b:\mathbb{R}^d\to\mathbb{R}$ as a generic notation for each component of $b: \mathbb{R}^d\to\mathbb{R}^d$. Correspondingly, we now refer to $y_i:=b(x_i) + \eta_i \in\mathbb{R}$ as a generic component of a ($d$-dimensional vector) sample of $Y$ in \eqref{discretesupervisedmodel}. Here, $\eta_i$ denotes a real-valued component of the $d-$dimensional Gaussian random variable $\mathcal{N}(0,\delta^{-1}\sigma\sigma^\top)$. With this abuse of notation, our training data set is denoted by the labeled data $\{x_i,y_i\}_{i=1}^N$ with $(x_{i}, y_{i})\in \mathbb{R}^{d}\times \mathbb{R}$. 

For completeness, we will review some basic concepts of reproducing kernel Hilbert spaces (RKHS) in Section~\ref{sec:RKHS}. Subsequently, in Section~\ref{sec:Spectral}, we will discuss a data-driven kernel-based spectral regression approach, whose mathematical foundation lies in the theory RKHSs. While controlling the generalization error bound $\epsilon$ is a practical interest, as we pointed out in Section~\ref{sec:learning_SDE}, this is only a necessary and not a sufficient condition for achieving the main results (Propositions~\ref{prop:euler_inv_bound} and \ref{prop:two_point}) in this paper. Therefore, it is crucial to understand whether the conditions in the Assumption~\ref{assu:coe_per}, the linear growth bound and globally Lipschitz continuity, which are the necessary conditions to the results in Section~\ref{sec:per_theory}, can be satisfied (see Corollary~\ref{coro:kernel} and Proposition~\ref{prop:Lip_spec}). We close this section with some discussions of the advantages and shortcomings of this estimation approach in Section~\ref{sec:kernelsummary}.

\subsection{A brief review of RKHS} \label{sec:RKHS}

For simplicity, we will consider the class of RKHS of the real valued function spaces on $\mathbb{R}^d$, while the argument can be extended to general locally compact metric spaces \cite{steinwart2008support}. To begin with, recall that a function $K: \mathbb{R}^{d} \times \mathbb{R}^{d} \rightarrow \mathbb{R}$ is called a (Mercer) kernel \cite{steinwart2008support, sun2005mercer} if it is continuous, symmetric and positive semidefinite, i.e.,
\begin{equation*}
K_{x}: = K(\cdot, x) \in C(\mathbb{R}^{d}), \quad \forall x \in \mathbb{R}^{d},
\end{equation*}
and for any finite set of points $\{x_{i}\}_{i=1}^{N}\subset \mathbb{R}^{d}$ the matrix
\begin{equation}\label{eq:Gram}
    \bm{K}_{N}: = \frac{1}{N}\left(K(x_{i}, x_{j}))\right)_{i,j=1}^{N}\in \mathbb{R}^{N\times N},
\end{equation}
is symmetric positive semidefinite. In the literature, such a matrix \eqref{eq:Gram} is called the \textit{empirical kernel} \cite{rosasco2010learning} of $K$ with respect to the sample points $\{x_{i}\}_{i=1}^{N}$ and the map $ \Phi: x\rightarrow K_{x}$ is called the \textit{feature map} \cite{steinwart2008support,caponnetto2007optimal}.

Using the feature map, the RKHS $\mathcal{H}$ associated with the kernel $K$ is defined to be the closure of $\operatorname{span}\{K_{x}\;:\; x\in \mathbb{R}^{d}\}$ \cite{sun2005mercer} with the inner produce given by
\begin{equation*}
    \langle f, g \rangle_{\mathcal{H}}: = \sum_{i=1}^{M_{1}} \sum_{j=1}^{M_{2}} c_{i}d_{j}K(x_{i}, y_{j}), \quad f = \sum_{i=1}^{M_1} c_{i}K_{x_i}, \quad g = \sum_{j=1}^{M_2} d_{j} K_{y_j}.
\end{equation*}
The reproducing property takes the form
\begin{equation}\label{eq:reproduce}
    f(x) = \langle f, K_{x} \rangle_{\mathcal{H}}, \quad  \forall f\in \mathcal{H}, \quad \forall x\in \mathbb{R}^{d},
\end{equation}
which implies that $\mathcal{H}$ consists of continuous functions on $\mathbb{R}^{d}$. As a result of the reproducing property, RKHS has a remarkable property that the $\mathcal{H}$-norm convergence implies pointwise convergence since the evaluation functionals are bounded. Let $\|\cdot\|_{\mathcal{H}}$ denote the norm in the RKHS $\mathcal{H}$. To develop the orthonormal basis of $\mathcal{H}$ for our applications, we propose the following assumption on the kernel.
\begin{assu}\label{assu:kernel}
The kernel $K$ is Hilbert–Schmidt with respect to the nondegenerate probability measure $\tilde{\pi}^{\delta}$, that is,
\begin{equation}\label{eq:kernel_bound}
\int_{\mathbb{R}^{d}}\int_{\mathbb{R}^{d}} (K(x,y))^2\td \tilde{\pi}^{\delta}(x) \td \tilde{\pi}^{\delta}(y) < +\infty,
\end{equation}
and $K_{x}\in L^{2}(\mathbb{R}^{d}, \tilde{\pi}^{\delta})$ for all $x\in \mathbb{R}^{d}$. Recall that the available sample points $\{x_{i}\}_{i=1}^{N}$ are drawn from $\tilde{\pi}^{\delta}$.
\end{assu}
As a result of Assumption~\ref{assu:kernel}, the following integral operator
\begin{equation}\label{eq:L_K}
    L_{K}: L^{2}(\mathbb{R}^{d}, \tilde{\pi}^{\delta}) \rightarrow L^{2}(\mathbb{R}^{d}, \tilde{\pi}^{\delta}), \quad \left( L_{K} g \right)(x) = \int_{\mathbb{R}^{d}} K(x,s) g(s) \td \tilde{\pi}^{\delta}(s),
\end{equation}
is bounded, compact and positive on $L^{2}(\mathbb{R}^{d}, \tilde{\pi}^{\delta})$ with countably many positive eigenvalues $\{\lambda_i\}_{i=1}^{\infty}$ \cite{sun2005mercer}. Such type of integral operators are widely studied in various contexts, e.g., graph Laplacian \cite{trillos2020error} and diffusion maps \cite{berry2016variable}. In particular, we have $L_{K}g\in \mathcal{H}$ for any $g\in L^{2}(\mathbb{R}^{d}, \tilde{\pi}^{\delta})$ \cite{sun2005mercer}. Thus, for each positive eigenvalue $\lambda_{i}>0$, we can take the  eigenfunction $u_{i}\in L^{2}(\mathbb{R}^{d}, \tilde{\pi}^{\delta}) \cap \mathcal{H}$ such that
\begin{equation*}
    L_{K}u_{i} = \lambda_i u_{i}, \quad \langle u_{i}, u_{j} \rangle_{\tilde{\pi}^{\delta}} = \delta_{ij}, \quad \forall i,j \geq 1,
\end{equation*}
where $\langle \cdot, \cdot \rangle_{\tilde{\pi}^{\delta}}$ denotes the inner product in $L^{2}(\mathbb{R}^{d}, \tilde{\pi}^{\delta})$. Note that $L^{2}(\mathbb{R}^{d},\tilde{\pi}^{\delta})$ consists of equivalence classes of functions and consequently it is not an RKHS. The fact that for nonzero eigenvalues, one can pick eigenfunctions in $\mathcal{H}$ is critical for our later construction of the estimates.

To see the connection between the eigenfunctions $\{u_{i}\}$ and the kernel, we introduce the spectral decomposition of $L_{K}$ \cite{rosasco2010learning},
\begin{equation*}
    L_{K} = \sum_{i= 1}^{\infty} \lambda_{i} \langle \cdot, u_{j}\rangle_{\tilde{\pi}^{\delta}} u_{j}.
\end{equation*}
Formally, we can exchange the order of summation and integration in the decomposition above and reach the following representation of the kernel $K$,
\begin{equation}\label{eq:Mercer}
    K(x,y) = \sum_{i=1}^{\infty} \lambda_i u_{i}(x) u_{i}(y).
\end{equation}
The relation in \eqref{eq:Mercer} is known as the Mercer theorem \cite{steinwart2008support}, which is valid even on noncompact domains, e.g., $\mathbb{R}^{d}$ \cite{sun2005mercer}. As a corollary of the Mercer theorem, $\{\sqrt{\lambda}_{i} u_{i}\}_{i=1}^{\infty}$ form an orthonormal basis of $\mathcal{H}$. In particular, notice that
\begin{equation*}
    K_{x} = \sum_{i=1}^{\infty} \lambda_iu_{i}(x) u_{i}, 
\end{equation*}
and
\begin{equation*}
    \langle K_{x}, K_{y} \rangle_{\mathcal{H}} = K(x,y) \Rightarrow \sum_{i,j=1}^{\infty} \sqrt{\lambda_{i}\lambda_{j}}u_{i}(x) u_{j}(y) \langle \sqrt{\lambda_{i}}u_{i}, \sqrt{\lambda_{j}}u_{j} \rangle_{\mathcal{H}} = \sum_{i=1}^{\infty} \lambda_{i} u_{i}(x) u_{i}(y),
\end{equation*}
which leads to the orthogonality properties, 
\begin{equation*}
    \langle \sqrt{\lambda_{i}}u_{i}, \sqrt{\lambda_{j}}u_{j} \rangle_{\mathcal{H}} = \delta_{ij}, \quad \forall i,j\geq 1.
\end{equation*}
Since $\{u_{i}\}$ is the orthonormal basis of $(\operatorname{Ker} L_{K})^{\perp}$ as a subspace of $L^{2}(\mathbb{R}^{d}, \tilde{\pi}^{\delta})$, we have the following isometric isomorphism
\begin{equation}\label{eq:iso_iso}
    L_{K}^{\frac{1}{2}}: \bar{D}_{K} \rightarrow \mathcal{H},
\end{equation}
where $\bar{D}_{K}=\left(\operatorname{Ker} L_{K}\right)^{\perp}$ is the closure of $D_{K}=\operatorname{span}\{u_{i}\}$ in $L^{2}(\mathbb{R}^{d}, \tilde{\pi}^{\delta})$. Here, $L_{K}^{\frac{1}{2}}$ is the square-root of $L_{K}$ satisfying
\begin{equation*}
    L_{K}^{\frac{1}{2}} u_{i} = \sqrt{\lambda_{i}} u_{i}, \quad \forall i \geq 1.
\end{equation*}
The isomorphism $L^{\frac{1}{2}}_{K}$ also reveals the regularity difference between functions in $\bar{D}_{K}$ and functions in $\mathcal{H}$. We should point out that if all the eigenvalues of $L_{K}$ are positive, then $L_K^{1/2}$ is injective and $\bar{D}_{K}$ is dense in $L^2(\mathbb{R}^d,\tilde{\pi}^{\delta})$ (e.g., Theorem~4.26 in \cite{steinwart2008support}). This means that any function $f\in L^2(\mathbb{R}^d,\tilde{\pi}^{\delta})$ can be approximated with arbitrary precision by a function in the RKHS $\mathcal{H}$, where the convergence is valid in $L^2(\mathbb{R}^d,\tilde{\pi}^{\delta})$. As we shall see later, this fact allows one to quantify the approximation error in term of the finite number of basis functions used in the numerical approximation of $b$.

\subsection{Spectral regression with integral operators} \label{sec:Spectral}

In Section~\ref{sec:RKHS}, we have reviewed some basic concepts of RKHS. Given a kernel $K$ that satisfies Assumption~\ref{assu:kernel}, the corresponding RKHS consists of continuous functions (rather than equivalent classes as in $L^{2}(\mathbb{R}^{d}, \tilde{\pi}^{\delta})$) that can be written as linear combinations (possibly infinite) of either images of feature maps $\{K_{x_i}\}$ or orthonormal basis functions $\{\sqrt{\lambda_i} u_i\}$. The latter  representation leads to a statistical learning approach that we shall explore in this section. Throughout the section, $\mathcal{H}$ always denotes the RKHS associated with a kernel $K$ that satisfies the Assumption~\ref{assu:kernel}. 

We consider the first $M$ (counting multiplicity) eigenvalues of the integral operator $L_{K}$ in \eqref{eq:L_K} satisfying
\begin{equation}\label{eq:spec_LK}
    \lambda_{1} \geq \lambda_{2} \geq \cdots \geq \lambda_{M} > \lambda_{M+1} \geq \cdots.
\end{equation}
Here, we have assumed that there is a spectral gap between $\lambda_{M}$ and $\lambda_{M+1}$. Let $P_{M}$ denote the orthogonal projection from $L^{2}(\mathbb{R}^{d}, \tilde{\pi}^{\delta})$ onto the span of the first $M$ eigenfunctions $\{u_{i}\}_{i=1}^{M}$, that is,
\begin{equation}\label{eq:proj_L2}
    P_{M} b  = \sum_{i=1}^{M} \langle b, u_{i} \rangle_{\tilde{\pi}^{\delta}} u_{i}, \quad \forall b\in L^2(\mathbb{R}^d, \tilde{\pi}^\delta).
\end{equation}
Recall that $\{u_{j}\}$ form an orthonormal family in $L^{2}(\mathbb{R}^{d}, \tilde{\pi}^{\delta})$. For any $b\in L^{2}(\mathbb{R}^{d}, \tilde{\pi}^{\delta})$, the $L^{2}$-convergence of $P_{M}b$ to $Pb$ is clear, where
\begin{equation*}
    P b  = \sum_{i=1}^{\infty} \langle b, u_{i} \rangle_{\tilde{\pi}^{\delta}} u_{i},
\end{equation*}
is the projection onto $\bar{D}_{K}$. Moreover, if $b\in \mathcal{H}$, such a convergence is valid in $\mathcal{H}$ as well. 

\begin{lem}\label{lem:spectral}
Let $K$ be a kernel that satisfies Assumption~\ref{assu:kernel} and $\mathcal{H}$ be the corresponding RKHS. Then $\mathcal{H}\subset L^{2}(\mathbb{R}^{d}, \tilde{\pi}^{\delta})$, and $\forall b \in \mathcal{H}$, we have $P_{M}b\rightarrow b$ as $M\rightarrow \infty$ in $\mathcal{H}$ and 
\begin{equation*}
    |b(x) - P_{M}b(x)| \leq \|b-P_{M}b\|_{\mathcal{H}} K^{\frac{1}{2}}(x,x), \quad \forall x\in \mathbb{R}^{d},
\end{equation*}
where $P_{M}b$ is the projection of $b$ defined in \eqref{eq:proj_L2}.
\end{lem}

\begin{proof}
The inclusion $\mathcal{H}\subset L^{2}(\mathbb{R}^{d}, \tilde{\pi}^{\delta})$ is a result of Assumption~\ref{assu:kernel} (see e.g.\cite{sun2005mercer}). As a result, $\forall b\in \mathcal{H}$ the projection $P_{M}b$ in \eqref{eq:proj_L2} is well-defined. We first show that $Pb = b$. By the isometric isomorphism in \eqref{eq:iso_iso}, we have $L_{K}^{-\frac{1}{2}}b \in \bar{D}_{K}$, that is,
\begin{equation*}
    L_{K}^{-\frac{1}{2}}b = P\left( L_{K}^{-\frac{1}{2}}b \right) = \sum_{i=1}^{\infty} \langle L_{K}^{-\frac{1}{2}}b, u_{i} \rangle_{\tilde{\pi}^{\delta}} u_{i} = \sum_{i=1}^{\infty} \frac{\langle b, u_{i} \rangle_{\tilde{\pi}^{\delta}}}{\sqrt{\lambda_{i}}} u_{i},
\end{equation*}
which leads to
\begin{equation*}
    b =L_{K}^{\frac{1}{2}}L_{K}^{-\frac{1}{2}}b= \sum_{i=1}^{\infty} \frac{\langle b, u_{i} \rangle_{\tilde{\pi}^{\delta}}}{\sqrt{\lambda_{i}}} L_{K}^{\frac{1}{2}}u_{i} = \sum_{i=1}^{\infty} \frac{\langle b, u_{i} \rangle_{\tilde{\pi}^{\delta}}}{\sqrt{\lambda_{i}}} \sqrt{\lambda_{i}}u_{i} = Pb.
\end{equation*}
Since $\{\sqrt{\lambda_{i}}u_{i}\}_{i=1}^{\infty}$ is an orthonormal basis of $\mathcal{H}$, we have
\begin{equation*}
    \| b\|_{\mathcal{H}}^2 = \sum_{i = 1}^{\infty} \frac{\langle b, u_{i} \rangle^2_{\tilde{\pi}^{\delta}}}{\lambda_{i}} < \infty.
\end{equation*}
In particular, 
\begin{equation*}
    b - P_{M}b = \sum_{i = M+1}^{\infty} \langle b, u_{i} \rangle_{\tilde{\pi}^{\delta}} u_{i} = \sum_{i = M+1}^{\infty} \frac{\langle b, u_{i} \rangle_{\tilde{\pi}^{\delta}}}{\sqrt{\lambda_{i}}} \sqrt{\lambda_{i}}u_{i},
\end{equation*}
which shows that,
\begin{equation*}
    \|b - P_{M}b\|_{\mathcal{H}}^{2} = \sum_{i = M+1}^{\infty} \frac{\langle b, u_{i} \rangle^2_{\tilde{\pi}^{\delta}}}{\lambda_{i}} \rightarrow 0,
\end{equation*}
as $M\rightarrow \infty$.

By the reproducing property \eqref{eq:reproduce}, we have $\|K_{x}\|_{\mathcal{H}}^{2} = \langle K_{x}, K_{x} \rangle_{\mathcal{H}} = K(x,x)$ and
\begin{equation*}
    |b(x)-P_{M}b(x)|  = | \langle b - P_{M}b, K_{x} \rangle_{\mathcal{H}}  |\leq \|b-P_{M}b\|_{\mathcal{H}} \|K_{x}\|_{\mathcal{H}} = \|b-P_{M}b\|_{\mathcal{H}} K^{\frac{1}{2}}(x,x), \quad \forall x\in\mathbb{R}^{d}.
\end{equation*}
\end{proof}
Lemma~\ref{lem:spectral} provides a pointwise error bound for the projection $P_{M}b$ in \eqref{eq:proj_L2} given $b\in \mathcal{H}$, which is closely related to the error condition \eqref{eq:error_linear} in  Assumption~\ref{assu:coe_per}. In particular, the following corollary clarifies this relation and provides sufficient conditions for Assumption~\ref{assu:kernel}.

\begin{coro}\label{coro:kernel}
Let $K$ be a kernel of linear growth bound, that is,
\begin{equation}\label{eq:kernel_linear}
    K(x,x) \leq C(1 + \|x\|^{2}), \quad \forall x\in \mathbb{R}^{d},
\end{equation}
for some constant $C \in (0, +\infty)$. Then, $K$ satisfies Assumption~\ref{assu:kernel} with respect to the probability measure $\tilde{\pi}^{\delta}$. Moreover, let $\mathcal{H}$ be the corresponding RKHS, we have $\forall b \in \mathcal{H}$
\begin{equation}\label{eq:coro_fm}
    |b(x) - P_{M}b(x)|^2 \leq C(1+\|x\|^2)\|b -P_{M}b\|^2_{\mathcal{H}}, \quad \forall x\in \mathbb{R}^{d}, 
\end{equation}
which is analogous to Eq.~\eqref{eq:error_linear} in Assumption~\ref{assu:coe_per}.
\end{coro}

\begin{proof}
To see the kernel $K$ satisfies the Hilbert-Schmidt condition in \eqref{eq:kernel_bound}, notice that
\begin{equation*}
    (K(x,y))^{2} = \langle K_{x}, K_{y} \rangle_{\mathcal{H}}^2 \leq \|K_{x}\|^2_{\mathcal{H}} \|K_{y}\|^2_{\mathcal{H}} = K(x,x)K(y,y),
\end{equation*}
such that,
\begin{equation*}
    \int_{\mathbb{R}^{d}}\int_{\mathbb{R}^{d}} (K(x,y))^{2} \td \tilde{\pi}^{\delta}(x) \td \tilde{\pi}^{\delta}(y) \leq \left( \int_{\mathbb{R}^{d}} K(x,x) \td \tilde{\pi}^{\delta}(x)  \right)^2 \leq C^{2} \left( \int_{\mathbb{R}^{d}} (1+\|x\|^2) \td \tilde{\pi}^{\delta}(x)  \right)^2 < +\infty,
\end{equation*}
where we have used the linear growth bound of the kernel. Here, $\mathbb{E}_{\tilde{\pi}^{\delta}}[\|x\|^2]<\infty$ due to the Lyapunov property of the process $X_{n}^{\delta}$ (see Proposition~\ref{prop:erg_euler} for the details).

On the other hand, $\forall x \in \mathbb{R}^{d}$, we have
\begin{equation*}
  \|K_{x}\|_{\tilde{\pi}^{\delta}}^2 =   \int_{\mathbb{R}^{d}} (K(x,y))^2 \td \tilde{\pi}^{\delta}(y) \leq K(x,x) \int_{\mathbb{R}^{d}} K(y,y) \td \tilde{\pi}^{\delta}(y) \leq C^{2}(1+\|x\|^2) \int_{\mathbb{R}^{d}}(1+\|y\|^2) \td \tilde{\pi}^{\delta}(y) <+\infty,
\end{equation*}
that is, $K_{x} \in L^{2}(\mathbb{R}^{d}, \tilde{\pi}^{\delta})$. Thus, $K$ satisfies Assumption~\ref{assu:kernel}, and the bound in \eqref{eq:coro_fm} is a result of Lemma~\ref{lem:spectral}.
\end{proof}

\begin{example}\label{kernelexamples}
(Polynomial kernels) The polynomial kernel of degree-$1$ \cite{steinwart2008support},
    \begin{equation}\label{eq:poly_ker}
        P_{1}(x,y) = x^{\top}y + d, \quad d>0,
    \end{equation}
    naturally satisfies the condition in Corollary~\ref{coro:kernel}. For functions in the RKHS associated with the polynomial kernel $K = P_{1}$ \eqref{eq:poly_ker}, we can choose the Lipschitz constant to be proportional to the RKHS norm.  In particular, by the reproducing property, we have $\forall b\in \mathcal{H}$,
\begin{equation*}
|b(x) - b(y)|  = |\langle b, K_{x} - K_{y} \rangle_{\mathcal{H}} | \leq \|b\|_{\mathcal{H}} \| K_{x} - K_{y}\|_{\mathcal{H}},
\end{equation*}
where
\begin{equation*}
    \| K_{x} - K_{y}\|_{\mathcal{H}} = \langle K_{x} - K_{y} , K_{x} - K_{y} \rangle_{\mathcal{H}}^{\frac{1}{2}} = \left( P_{1}(x,x) + P_{1}(y,y) - 2P_{1}(x,y)  \right)^{\frac{1}{2}} = \|x-y\|.
\end{equation*}
Thus, in this case, we have
\begin{equation*}
    |b(x) - b(y)| \leq \|b\|_{\mathcal{H}} \|x-y\|,
\end{equation*}
that is, all $b\in \mathcal{H}$ are globally Lipschitz with Lipschitz constants $\|b\|_{\mathcal{H}}$.

\end{example}

\subsubsection{Nystr\"om interpolation}

In applications, the projection $P_{M}b$ in \eqref{eq:proj_L2} is not a practical estimate of $b$ since the eigenfunctions $\{u_{j}\}$ are unknown in general. Another issue is that assuming $b\in \mathcal{H}$ is too optimistic. Of course, if we are given the information such as the linear growth condition in Assumption~\ref{assu:Ito_coef}, then we should choose a kernel that also satisfies the condition (as in Corollary 4.1), such as the polynomial kernel in Example~\ref{kernelexamples}. Without a priori information, the best we can hope is that $b \in L^{2}(\mathbb{R}^{d}, \tilde{\pi}^{\delta})\cap C(\mathbb{R}^{d})$. To resolve these issues, we first need to come up with empirical estimates of the projection $P_{M}$. Then, we should study the properties, including the convergence and the Lipschitz continuity, of the resulting estimates under a mild assumption that $b \in L^{2}(\mathbb{R}^{d}, \tilde{\pi}^{\delta})\cap C(\mathbb{R}^{d})$.

We are going to construct the estimates of the projection $P_{M}$ based on the eigenvalues and eigenvectors of the empirical kernel $\bm{K}_{N}$ \eqref{eq:Gram} given by the i.i.d. sample points $\{x_{i}\}$  according to $\tilde{\pi}^{\delta}$. Such a spectral projection method (onto the data-driven basis constructed by eigen-spaces of the kernel integral operator) has been advocated and widely used in many applications. In the context of learning dynamical systems, see  \cite{berry2015nonparametric,berry2020bridging,alexander2020operator,gilani2021kernel} and the references therein. Assume $ r_N = \operatorname{rank}(\bm{K}_{N}) \geq M$, and denote $\{\hat{\lambda}_{j}\}_{j=1}^{r_{N}}$ as the set of all nonzero eigenvalues (in descending order, counting  multiplicity) of $\bm{K}_{N}$ with the corresponding normalized eigenvectors $\{\hat{u}_{j}\}\subset \mathbb{R}^{N}$, which form an orthonormal family under the inner product of $\langle \cdot,\cdot\rangle_{\tilde{\pi}^{\delta}_N}:= \frac{1}{N}\langle \cdot,\cdot\rangle$. Here $\tilde{\pi}^{\delta}_N$ denotes the delta measure corresponding to discrete samples $\{x_i\}_{i=1}^N$, that is,
\begin{equation*}
    \tilde{\pi}^{\delta}_N = \frac{1}{N} \sum_{i=1}^{N} \delta_{x_i}.
\end{equation*}
The main difficulty in relating $L_{K}$ and $\bm{K}_{N}$ is that they operate on different spaces. To resolve the issue, we follow  \cite{rosasco2010learning} and  introduce $\hat{v}_{j}\in \mathcal{H}$ as
\begin{equation}\label{eq:vjhat}
    \hat{v}_{j} = \frac{1}{N\sqrt{\hat{\lambda}_{j}}} \sum_{i=1}^{N} (\hat{u}_{j})_i K_{x_i},
\end{equation}
where $(\hat{u}_{j})_i$ denotes the $i$-th component of the eigenvector $\hat{u}_{j}$. Here, the set $\{\hat{v}_{j}\}$ forms an orthonormal family in $\mathcal{H}$, satisfying
\begin{equation}\label{eq:ujhat}
     \left(\hat{u}_{j}\right)_{i} = \frac{1}{\sqrt{\hat{\lambda}_{j}}} \hat{v}_{j}(x_{i}), \quad i = 1,2,\dots, N,
\end{equation}
that is, $\hat{v}_{j}/\sqrt{\hat{\lambda}_{j}}$ can be interpreted as the (Nystr\"om) interpolation of the vector $\hat{u}_{j}$ in $\mathcal{H}$. To verify the orthogonality, we have
\begin{equation*}
\langle \hat{v}_{k}, \hat{v}_{k'} \rangle_{\mathcal{H}}
 = \frac{1}{N^2\sqrt{\hat{\lambda}_{k}\hat{\lambda}_{k'}}} \sum_{i,j=1}^{N} (\hat{u}_{k})_{i}(\hat{u}_{k'})_{j} K(x_{i},x_{j}) = \frac{1}{N\sqrt{\hat{\lambda}_{k}\hat{\lambda}_{k'}}} \hat{u}_{k}^{\top} \bm{K}_{N} \hat{u}_{k'} = \delta_{kk'}.
\end{equation*}
To this end, we define the following empirical approximation of $P_{M}$,
\begin{equation}\label{eq:pmhat1}
    \hat{P}_{M}b = \sum_{i=1}^{M} \langle b, \hat{v}_{i}\rangle_{\mathcal{H}} \hat{v}_{i},
\end{equation}
where
\begin{equation*}
    \langle b, \hat{v}_{j}\rangle_{\mathcal{H}} = \frac{1}{N\sqrt{\hat{\lambda}_{j}}} \sum_{i=1}^{N} (\hat{u}_{j})_i \langle b, K_{x_i}\rangle_{\mathcal{H}} = \frac{1}{\sqrt{\hat{\lambda}_{j}}} \langle R_{N}b, \hat{u}_{j} \rangle_{\tilde{\pi}^{\delta}_N}, \quad R_{N}b := \left(b(x_{1}), b(x_{2}), \cdots, b(x_{N})\right)^{\top}.
\end{equation*}
Here, $R_{N}: \mathcal{H}\rightarrow \mathbb{R}^{N}$ is called the sampling operator \cite{rosasco2010learning, von2008consistency} associated with the discrete set $\{x_{i}\}_{i=1}^{N}$. Note that the inner product $\langle R_{N}b, \hat{u}_{j} \rangle_{\tilde{\pi}^{\delta}_{N}}$ is well-defined for a general function $b\in L^2(\mathbb{R}^{d}, \tilde{\pi}^{\delta})\cap C(\mathbb{R}^{d})$ (continuity is necessary so that the sampling operator $R_{N}$ is well-defined). Thus, we can extend the definition of $\hat{P}_{M}$ in \eqref{eq:pmhat1} to
\begin{equation}\label{eq:pmhat2}
   \hat{P}_{M} b =  \sum_{i=1}^{M} \langle R_{N}b, \hat{u}_{i} \rangle_{\tilde{\pi}^{\delta}_N} \frac{\hat{v}_{i}}{\sqrt{\hat{\lambda}_{i}}}.
\end{equation}
Here, the projection $\hat{P}_{M}$ maps functions in $L^2(\mathbb{R}^{d}, \tilde{\pi}^{\delta})\cap C(\mathbb{R}^{d})$ to a finite-dimensional space $\hat{\mathcal{H}}_{M}:= \operatorname{span}\{\hat{v}_{i}\}_{i=1}^{M}\subset \mathcal{H}$. We will use Eq.~\eqref{eq:pmhat2}, instead of Eq.~\eqref{eq:pmhat1}, as the definition of the empirical estimates $\hat{P}_{M}b$ with respect to the projection $P_{M}b$.

\subsubsection{Estimation error} \label{sec:est_error_kernel}

Compared with $\mathcal{H}$, the hypothesis space $\hat{\mathcal{H}}_{M}$ is of dimension $M$ (the order of the estimates) regardless of the sample size. Note that $\hat{\mathcal{H}}_{M}$ still depends on the samples due to the choice of the basis functions $\{\hat{v}_{j}\}$ in \eqref{eq:vjhat}. We should interpret $\hat{\mathcal{H}}_{M}$ as the empirical approximation of the underlying hypothesis space $\mathcal{H}_{M}: = \operatorname{span}\{\sqrt{\lambda_{i}}u_{i}\}_{i=1}^{M}$, which is independent of the sample but unknown. 

As for our estimation, recall that the training data is given by $\{x_i,y_i\}_{i=1}^N$ with $y_i=b(x_i) + \eta_i$, where $\{\eta_i\}$ are i.i.d. mean zero Gaussian noise of finite variance. Denoting the random variables $\eta_{i}\sim E:\Omega\to\mathbb{R}$ with Gaussian probability distribution, $\mathcal{N}(0,R),$ where $R \leq 
\delta^{-1}\max\limits_{1\leq i\leq d}\{(\sigma\sigma^\top)_{ii}\}$, and $(\eta_1, \eta_2,\dots, \eta_{N}) \sim E_{N}:\Omega\to\mathbb{R}^{N}$, our empirical estimate corresponds to,
\begin{equation}\label{eq:pmhat3}
   \hat{P}_{M} (b+E) =
   \hat{P}_M b +  \hat{P}_M E, \quad 
   \hat{P}_M E(x) = 
   \sum_{i=1}^{M} \langle E_{N}, \hat{u}_{i} \rangle_{\tilde{\pi}^{\delta}_N} \frac{\hat{v}_{i}(x)}{\sqrt{\hat{\lambda}_{i}}}.
\end{equation}
Thus, the estimation error $P_{M}b - \hat{P}_{M} (b+E)$ consists of the error introduced by the empirical projection, $P_{M}b - \hat{P}_{M} b$ and the error induced by the noise in the observation, $\hat{P}_{M}E$.
Notice that,
\begin{equation*}
\|\hat{P}_ME\|^2_{\tilde{\pi}^\delta} =\sum_{i=1}^{M} \langle E_{N}, \hat{u}_{i} \rangle_{\tilde{\pi}^{\delta}_N}^2 \leq \|E_N \|^{2}_{\tilde{\pi}^\delta_N} = \frac{1}{N}\sum_{i=1}^N \eta_i^2,
\end{equation*}
follows a $\chi^2$-distribution of degree $N$. By the following concentration inequality (e.g., Example 2.11 in \cite{wainwright2019high}),
\BEA
\mathbb{P}\left(\left|\frac{1}{NR}\sum_{i=1}^N \eta_i^2 - 1 \right| \geq t \right) \leq 2 e^{-\frac{Nt^2}{8}}\notag, \quad \forall t\in (0,1),
\EEA
we conclude that for any $\tau>0$ and $N$ large enough, 
\BEA
\|\hat{P}_ME\|^2_{\tilde{\pi}^\delta} \leq \Big(1+\sqrt{\frac{8\tau}{N}}\Big)R, \quad R \leq \delta^{-1}\operatorname{Tr}[\sigma\sigma^\top],\label{noisebound}
\EEA
with probability greater than $1-2e^{-\tau}$. The norm $\|\cdot\|_{\tilde{\pi}^{\delta}_N}$ is defined with respect to the inner product $\langle \cdot, \cdot, \rangle_{\tilde{\pi}^{\delta}_N}$.
 
For bounded kernels, the following proposition characterizes the convergence of $\|P_{M}b - \hat{P}_{M}(b+E) \|_{\tilde{\pi}^{\delta}}$ for any $b\in L^2(\mathbb{R}^{d}, \tilde{\pi}^{\delta})\cap C(\mathbb{R}^{d})$.
\begin{prop}\label{prop:error_rate_reg}
Let $b\in L^2(\mathbb{R}^{d}, \tilde{\pi}^{\delta})\cap C(\mathbb{R}^{d})$ and given the training data $\{x_i, y_i\}_{i=1}^{N}$ as we previously proposed. Suppose that $K$ is a bounded kernel satisfying
\begin{equation}\label{eq:bounded_kernel}
     k_{max}:=\sup_{x\in \mathbb{R}^{d}} K^{\frac{1}{2}}(x,x) < \infty,
\end{equation}
and the corresponding empirical estimate $\hat{P}_{M}(b+E)$ \eqref{eq:pmhat3} to the projection $P_{M}f$ \eqref{eq:proj_L2}. For any $\tau>0$, we have
\begin{equation*}
    \left\|P_{M}b - \hat{P}_{M}(b+E)\right\|^2_{\tilde{\pi}^{\delta}} \leq  \frac{32 k_{max}^{2}\tau_1\|b\|^{2}_{\tilde{\pi}_{N}^{\delta}}}{(\lambda_{M+1}-\lambda_{M})^{2}\hat{\lambda}_{r_{N}}N} + \Big(1+\sqrt{\frac{8\tau_2}{N}}\Big)\delta^{-1}\operatorname{Tr}[\sigma\sigma^\top],
\end{equation*}
with probability greater than $1-2e^{-\tau}$, where $\tau = \min\{\tau_1,\tau_2\}$, given the number $N$ of samples satisfies
\begin{equation*}
    N > \frac{128 k_{max}^2 \tau_1}{(\lambda_{M}-\lambda_{M+1})^2}.
\end{equation*}
Here, $\{\lambda_i\}_{i=1}^{\infty}$ and $\{\hat{\lambda}_{i}\}_{i=1}^{r_N}$ are the positive eigenvalues of the integral operator $L_{K}$ and the empirical kernel $\bm{K}_{N}$, respectively, with $r_{N} = \operatorname{rank}(\bm{K}_{N})$.
\end{prop}

\begin{proof}

We first observe that,
\begin{eqnarray}
    \left\|P_{M}b - \hat{P}_{M}(b+E)\right\|^2_{\tilde{\pi}^{\delta}} &\leq& \left\|P_{M}b - \hat{P}_{M}b\right\|^2_{\tilde{\pi}^{\delta}} + \left\|P_{M}E\right\|^2_{\tilde{\pi}^{\delta}}\notag \\ &=& \left\|P_{M}b - P_{M}\hat{P}_{M}b\right\|^2_{\tilde{\pi}^{\delta}} + \left\|(I-P_{M})\hat{P}_{M}b\right\|^2_{\tilde{\pi}^{\delta}} + \left\|P_{M}E\right\|^2_{\tilde{\pi}^{\delta}},
\end{eqnarray}
where we have used the Pythagorean theorem to deduce the equality above.

For the first two terms, we have, 
\begin{equation*}
\left\|P_{M}b - P_{M}\hat{P}_{M}b\right\|^2_{\tilde{\pi}^{\delta}} = \left\|P_{M}(I - \hat{P}_{M})b\right\|^2_{\tilde{\pi}^{\delta}}  \leq \left(  \sum_{i=M+1}^{r_{N}} \frac{\langle R_{N}b, \hat{u}_{i} \rangle_{\tilde{\pi}^{\delta}_{N}}^{2}}{\hat{\lambda}_i}\right) \left( \sum_{i=M+1}^{r_{N}}\| P_{M}\hat{v}_{i} \|_{\tilde{\pi}^{\delta}}^{2} \right)
\end{equation*}
and
\begin{equation*}
    \left\|(I-P_{M})\hat{P}_{M}b\right\|^2_{\tilde{\pi}^{\delta}} \leq \left(  \sum_{i=1}^{M} \frac{\langle R_{N}b, \hat{u}_{i} \rangle_{\tilde{\pi}^{\delta}_{N}}^{2}}{\hat{\lambda}_i}\right) \left( \sum_{i=1}^{M} \left\|(I - P_{M})\hat{v}_{i} \right\|_{\tilde{\pi}^{\delta}}^{2} \right), \quad i = 1,2,\dots, r_{N},
\end{equation*}
by Eq.~\eqref{eq:pmhat2}. Notice that
\begin{equation*}
  \left(  \sum_{i=1}^{r_{N}} \frac{\langle R_{N}b, \hat{u}_{i} \rangle_{\tilde{\pi}^{\delta}_{N}}^{2}}{\hat{\lambda}_i}\right) \leq \frac{1}{\hat{\lambda}_{r_{N}}} \sum_{i=1}^{r_{N}} \langle R_{N}b, \hat{u}_{i} \rangle_{\tilde{\pi}^{\delta}_{N}}^2 =  \frac{1}{\hat{\lambda}_{r_{N}}} \|b\|_{\tilde{\pi}^{\delta}_N}^{2}.
\end{equation*}
Finally, by Theorem 12 in \cite{rosasco2010learning}, we have
\begin{equation*}
    \left\|P_{M}b - \hat{P}_{M}b\right\|^2_{\tilde{\pi}^{\delta}} \leq \frac{\|b\|^2_{\tilde{\pi}^{\delta}_N}}{\hat{\lambda}_{r_{N}}} \left(\sum_{i=1}^{M}\|(I - P_{M})\hat{v}_{i} \|_{\tilde{\pi}^{\delta}}^{2} + \sum_{i=M+1}^{r_{N}}\| P_{M}\hat{v}_{i} \|_{\tilde{\pi}^{\delta}}^{2}   \right)\leq  \frac{32 k_{max}^{2}\tau_1\|b\|^{2}_{\tilde{\pi}^{\delta}_N}}{(\lambda_{M+1}-\lambda_{M})^{2}\hat{\lambda}_{r_{N}}N}
\end{equation*}
with probability greater than $1-2e^{-\tau_1}$. Together with Eq.~\eqref{noisebound}, the proof is completed.
\end{proof}

In the error bound above, we rely on Theorem 12 in \cite{rosasco2010learning} that requires the boundedness of the kernel to apply Hoeffding's inequality. For general kernels satisfying Assumption~\ref{assu:kernel} (not necessarily bounded), one can develop similar but much weaker probability bounds via Chebyshev's inequality.

\subsubsection{Generalization error} \label{sec:gen_error_kernel}

In Section~\ref{sec:est_error_kernel}, we have studied the estimation error of the empirical estimate \eqref{eq:pmhat3} under a model with additive i.i.d. noise. In this section, we will study the approximation error and comment on the generalization error in learning the drift coefficients.

For $b\in L^2(\mathbb{R}^{d}, \tilde{\pi}^{\delta})\cap C(\mathbb{R}^{d})$, we introduce the following decomposition
\begin{equation}
    b - \hat{P}_{M}(b+E) = \underbrace{(b - Pb) + (Pb - P_{M}b)}_\text{approximation error} +\underbrace{(P_{M}b - \hat{P}_{M}b) + \hat{P}_{M}E}_\text{estimation error},
\end{equation}
which is commonly defined in learning theory. The approximation error $\|b- P_{M}b\|_{\tilde{\pi}^{\delta}}^2$ satisfies
\begin{equation*}
    \|b- P_{M}b\|_{\tilde{\pi}^{\delta}}^2 = \|b - Pb\|_{\tilde{\pi}^{\delta}}^2 + \|Pb - P_{M}b\|_{\tilde{\pi}^{\delta}}^2,
\end{equation*}
since $(I-P)b \in \operatorname{ker}(L_{K})$ and $(P-P_{M})b \in \operatorname{ker}(L_{K})^{\perp}$. The term $\|b - Pb\|_{\tilde{\pi}^{\delta}}^2$ on the right-hand side, independent of $M$, corresponds to the component of the approximation error induced by the choice of hypothesis space $\mathcal{H}$. This bias is intrinsic in the sense that  it only depends on the choice of the kernel. In particular, $\|b - Pb\|_{\tilde{\pi}^{\delta}}=0$ if and only if $b\in \operatorname{ker}(L_{K})^{\perp}$. A sufficient condition would be that the integral operator $L_{K}$ in \eqref{eq:L_K} has only positive eigenvalues. In such a case, $\mathcal{H}$ is dense in $\operatorname{ker}(L_{K})^{\perp}$ with respect to the topology induced by the norm $\|\cdot\|_{\tilde{\pi}^{\delta}}$. In general, the property of RKHS being dense in a certain function space corresponds to the \emph{universality} of the RKHS \cite{sriperumbudur2011universality}.

The term $\|Pb - P_{M}b\|_{\tilde{\pi}^{\delta}}^2$ describes the approximation error induced by truncation, that is, using $\mathcal{H}_{M}$, instead of $\mathcal{H}$, as the underlying hypothesis space, which vanishes as the order $M\to\infty$. In particular, we have
\begin{equation*}
    \|Pb - P_{M}b\|_{\tilde{\pi}^{\delta}}^2 =  \sum_{i=M+1}^{\infty} \langle b, u_{i} \rangle^{2}_{\tilde{\pi}^{\delta}},
\end{equation*}
and the decay rate of $\|Pb - P_{M}b\|_{\tilde{\pi}^{\delta}}^2$, without further assumption on $b$, is hard to identify. In our application, the drift coefficients $b$ in \eqref{eq:unper} is of linear growth bound according to Assumption~\ref{assu:Ito_coef}. On the other hand, the RKHS associated with a kernel of linear growth bounds (see Corollary~\ref{coro:kernel} for the details) consists of functions of linear growth bounds. Thus, it is reasonable to propose the following assumption on $b$.

\begin{assu} \label{assu:f}
Let $b$ be a component of the drift coefficients in \eqref{eq:unper} satisfying Assumption~\ref{assu:Ito_coef}.  Assume that there exists a kernel $K:\mathbb{R}^{d}\times \mathbb{R}^{d} \rightarrow \mathbb{R} $ of linear growth bound \eqref{eq:kernel_linear} and the corresponding RKHS $\mathcal{H}$ such that:
\begin{enumerate}[i.]
    \item (Decay rate of the eigenvalue) the positive eigenvalues $\{\lambda_{i}\}$ of the integral operator $L_{K}$ \eqref{eq:L_K} follow the following decay rate \cite{caponnetto2007optimal},
    \begin{equation}\label{eq:eig_rate}
    \alpha i^{-r} \leq \lambda_{i} \leq \beta i^{-r}, \quad \alpha, \beta, r>0, \quad i = 1,2,\dots. 
    \end{equation}
\item (Existence of the target function) the projection sequence $\{P_{M}b\}\subset \mathcal{H}$ converges to a function $b_{\mathcal{H}}$ in $\mathcal{H}$.
\end{enumerate}
\end{assu}
The decay rate assumption in \eqref{eq:eig_rate} is related to the \emph{effective dimension} \cite{caponnetto2007optimal} of the RKHS $\mathcal{H}$ with respect to the space $L^{2}(\mathbb{R}^{d}, \tilde{\pi}^{\delta})$. Recall that by Corollary~\ref{coro:kernel}, all kernels of linear growth bound satisfy Assumption~\ref{assu:kernel}, which leads to the inclusion $\mathcal{H} \subset L^{2}(\mathbb{R}^{d}
,\tilde{\pi}^{\delta})$. Since the convergence in $\mathcal{H}$ implies the convergence in $L^{2}(\mathbb{R}^{d}, \tilde{\pi}^{\delta})$, we have 
\begin{equation*}
    \| Pb - b_{\mathcal{H}} \|_{\tilde{\pi}^{\delta}} \leq \|Pb - P_{M}b\|_{\tilde{\pi}^{\delta}} + \| P_{M}b - b_{\mathcal{H}}\|_{\tilde{\pi}^{\delta}} \rightarrow 0,
\end{equation*}
as $M \rightarrow \infty$, that is, $Pb = b_{\mathcal{H}}$ in $L^{2}(\mathbb{R}^{d}, \tilde{\pi}^{\delta})$. The function $b_{\mathcal{H}}$ is often called the \emph{target function} of $b$ with respect to the hypothesis space $\mathcal{H}$ \cite{cucker2002mathematical}. With the convergence $P_{M}b \rightarrow b_{\mathcal{H}}$ in $\mathcal{H}$, we have
\begin{equation*}
    \| b_{\mathcal{H}}\|^2_{\mathcal{H}} = \lim_{M \rightarrow +\infty} \|P_{M} b\|^2_{\mathcal{H}} = \sum_{i = 1}^{\infty} \frac{\langle b, u_{i} \rangle^2_{\tilde{\pi}^{\delta}}}{\lambda_{i}} < \infty.
\end{equation*}
As a result, we have
\begin{equation*}
    \|Pb - P_{M}b\|_{\tilde{\pi}^{\delta}}^2 =  \sum_{i=M+1}^{\infty} \langle b, u_{i} \rangle^{2} \leq \lambda_{M+1} \sum_{i=M+1}^{\infty} \frac{\langle b, u_{i} \rangle^2_{\tilde{\pi}^{\delta}}}{\lambda_{i}} \leq \lambda_{M+1} \|b_{\mathcal{H}}\|_{\mathcal{H}}^2 = O(M^{-r}),
\end{equation*}
where we have used the decay rate assumption of the eigenvalues \eqref{eq:eig_rate}.

To conclude, in our notation, the estimator $\hat{P}_M(b+E)$ is a component of $b_\epsilon$ in \eqref{eq:sig_e_1}. In this case, the generalization error is given by,
\BEA
\mathbb{E}_{\tilde{\pi}^{\delta}} \big[\|b-b_\epsilon\|^2 \big] = O(M^{-r})+ O\left(\frac{1}{(\lambda_{M+1}-\lambda_{M})}\hat{\lambda}_{r_{N}}^{-1}N^{-1}\right)+  \delta^{-1}\operatorname{Tr}[\sigma\sigma^\top],\label{generalkernelerror}
\EEA
as $M,N \to\infty$, where we have assumed that the error rate is uniform component-wise.
Recall that the first term in \eqref{generalkernelerror} is the approximation error and the last two terms are estimation errors, respectively, which were derived under various assumptions reported throughout the previous and the current subsections.  
Importantly, this error bound is valid only for bounded kernels by the assumption in Proposition~\ref{prop:error_rate_reg}. For (unbounded) kernels, e.g., kernels of linear growth, one needs to replace the second error term in \eqref{generalkernelerror} with another appropriate rate.

\subsubsection{Lipschitz continuity}

Now, we check the Lipschitz continuity of the estimator $\hat{P}_{M}(b+E)$, which is one of the fundamental assumptions (Assumption~\ref{assu:coe_per}) for the statistical error bounds in Propositions~\ref{prop:euler_inv_bound} and \ref{prop:two_point}. 
\begin{prop}\label{prop:Lip_spec}
Let $K$ be a kernel satisfying Assumption~\ref{assu:kernel} (not necessarily bounded). We further assume $K\in  C^{1}(\mathbb{R}^{d}\times \mathbb{R}^{d})$ such that the following function
\begin{equation}\label{eq:fun_L}
    L(x): = \sup_{z\in \mathbb{R}^{d}}\|\nabla_{z} K_{x}(z)\| < \infty, \quad \forall x\in \mathbb{R}^{d},
\end{equation}
is well-defined. Then, given the data set $\{x_i,y_i\}_{i=1}^N$ as in Proposition~\ref{prop:error_rate_reg}, the order-$M$ spectral regression estimates $\hat{P}_{M}(b+E)$ in \eqref{eq:pmhat3} satisfies
\begin{equation}\label{eq:Lip_spec}
    \sup_{x,x'\in \mathbb{R}^{d}, x\not = x'} \frac{ \left| \hat{P}_{M}(b+E)(x) - \hat{P}_{M}(b+E)(x')  \right|}{\|x-x'\|} \leq \left( \sum_{i=1}^{M} \hat{\lambda}_{i}^{-2} \right)^{\frac{1}{2}} \big(\|b\|_{\tilde{\pi}^{\delta}_{N}}+\|E_N\|_{\tilde{\pi}^\delta_N}\big) \|L\|_{\tilde{\pi}^{\delta}_{N}}.
\end{equation}
\end{prop}

\begin{proof}
We rewrite the order-$M$ estimates in \eqref{eq:pmhat2} as
\begin{equation*}
    \hat{P}_{M}(b+E)(x)= \frac{1}{N}\sum_{i=1}^{M} \langle R_{N}b +E_N, \hat{u}_{i} \rangle_{\tilde{\pi}^{\delta}_{N}} \frac{1}{\hat{\lambda}_{i}} \sum_{j=1}^{N} \left( \hat{u}_{i}\right)_{j}K(x,x_j) = \sum_{i=1}^{M} \frac{1}{\hat{\lambda}_{i}} \langle R_{N}b+E_N, \hat{u}_{i} \rangle_{\tilde{\pi}^{\delta}_{N}} \langle R_{N}K_{x}, \hat{u}_{i} \rangle_{\tilde{\pi}^{\delta}_{N}}.
\end{equation*}
Thus, for $x,x'\in \mathbb{R}^{d}$, we have
\begin{eqnarray}
    \left| \hat{P}_{M}(b+E)(x) - \hat{P}_{M}(b+E)(x')  \right| &\leq& \sum_{i=1}^{M} \left| \frac{1}{\hat{\lambda}_{i}} \langle R_{N}b+E_N, \hat{u}_{i} \rangle_{\tilde{\pi}_{N}^{\delta}} \langle R_{N}K_{x} - R_{N}K_{x'}, \hat{u}_{i} \rangle_{\tilde{\pi}_{N}^{\delta}} \right| \notag\\ &\leq&  \sum_{i=1}^{M} \left|\frac{1}{\hat{\lambda}_{i}} \langle R_{N}b+E_N, \hat{u}_{i} \rangle_{\tilde{\pi}_{N}^{\delta}} \right| \| R_{N}K_{x} - R_{N}K_{x'}\|_{\tilde{\pi}^{\delta}_{N}}.\notag
\end{eqnarray}
Notice that
\begin{equation*}
    \sum_{i=1}^{M} \left| \frac{1}{\hat{\lambda}_{i}}\langle R_{N}b+E_N, \hat{u}_{i} \rangle_{\tilde{\pi}^{\delta}_{N}} \right| \leq \left( \sum_{i=1}^{M} \hat{\lambda}_{i}^{-2} \right)^{\frac{1}{2}} \left(  \sum_{i=1}^{M} \langle R_{N}b+E_N, \hat{u}_{i} \rangle_{\tilde{\pi}^{\delta}_{N}}^2 \right)^{\frac{1}{2}} \leq \left( \sum_{i=1}^{M} \hat{\lambda}_{i}^{-2} \right)^{\frac{1}{2}} \big(\|b\|_{\tilde{\pi}^{\delta}_{N}}+\|E_N\|_{\tilde{\pi}^\delta_N}\big).
\end{equation*}
As a result, we have,
\begin{equation*}
     \left| \hat{P}_{M}(b+E)(x) - \hat{P}_{M}(b+E)(x')  \right|  \leq \left( \sum_{i=1}^{M} \hat{\lambda}_{i}^{-2} \right)^{\frac{1}{2}} \big(\|b\|_{\tilde{\pi}^{\delta}_{N}}+\|E_N\|_{\tilde{\pi}^\delta_N}\big) \|R_{N}K_{x} - R_{N}K_{x'}\|_{\tilde{\pi}^{\delta}_{N}}.
\end{equation*}
By the definition of the function $L(x)$, we have
\begin{equation*}
      \left| K_{x}(x_{i}) - K_{x'}(x_i)\right| =  \left|K_{x_i}(x) - K_{x_i}(x')\right| \leq L(x_i) \|x -x'\|, \quad i=1,2,\dots,N, \quad  \forall x, x' \in \mathbb{R}^{d}.
\end{equation*}
Thus,
\begin{eqnarray}
    \sup_{x,x'\in \mathbb{R}^{d}, x\not = x'} \frac{ \left| \hat{P}_{M}(b+E)(x) - \hat{P}_{M}(b+E)(x')  \right|}{\|x-x'\|} &\leq&  \left( \sum_{i=1}^{N} \hat{\lambda}_{i}^{-2} \right)^{\frac{1}{2}} \big(\|b\|_{\tilde{\pi}^{\delta}_{N}}+\|E_N\|_{\tilde{\pi}^\delta_N}\big) \left( \frac{1}{N} \sum_{i=1}^{N} L(x_i)^2\right)^{\frac{1}{2}}\notag \\ &=& \left( \sum_{i=1}^{M} \hat{\lambda}_{i}^{-2} \right)^{\frac{1}{2}} \big(\|b\|_{\tilde{\pi}^{\delta}_{N}}+\|E_N\|_{\tilde{\pi}^\delta_N}\big) \|L\|_{\tilde{\pi}^{\delta}_{N}}.\notag
\end{eqnarray}
\end{proof}

Here, the function $L$ can be defined for unbounded kernels. For example, for the polynomial kernel $P_{1}(x,y) = x^{\top}y+d$, the corresponding $L(x) = \|x\|$. If the kernel $K$ is a radial basis function (RBF) kernel, e.g., Gaussian  kernels, the function $L$ in \eqref{eq:fun_L} will reduce to a constant function.

Notice that when $N\rightarrow \infty$,  the upper bound in \eqref{eq:Lip_spec} stays bounded and only depends on the order of the estimates. In particular, when the kernel is bounded, we have the convergence of the eigenvalue as the sample increases, that is, $ \hat{\lambda}_{i} \rightarrow \lambda_{i}$ in high probability as $n \rightarrow \infty$ \cite{rosasco2010learning}. Under the decay rate assumption in \eqref{eq:eig_rate}, the sum in \eqref{eq:Lip_spec} satisfies (in high probability),
\begin{equation*}
    \left( \sum_{i=1}^{M} \hat{\lambda}_{i}^{-2} \right)^{\frac{1}{2}} \leq \left(\sum_{i=1}^{M} \alpha^{-1} i^{2r}\right)^{\frac{1}{2}} = O(M^{r+\frac{1}{2}}),
\end{equation*}
as $N\rightarrow \infty$. Thus, the Lipschitz constants of the estimates $\hat{P}_{M}(b+E)$ in \eqref{eq:pmhat3} is at most of polynomial growth rate with respect to the order of the estimates (in high probability) under the infinite sample assumption. 

\subsection{Remarks on the spectral regression approach}
\label{sec:kernelsummary}

The spectral regression approach has several advantages. First, one can impose the characteristics of the functions to be estimated in the kernel (such as those in Corollary~\ref{coro:kernel}). Even when the unknown function to be estimated is unbounded, thanks to the integral operator $L_{K}$ being compact (as an operator from $L^{2}(\mathbb{R}^{d}, \tilde{\pi}^{\delta})$ to itself) with range in the RKHS $\mathcal{H}$, we are allowed to construct a set of eigenfunctions in $(\operatorname{Ker}(L_{K}))^{\perp}$ to characterize the RKHS $\mathcal{H}$ associated with the kernel. One important  issue in practice is to identify a kernel such that $L_K$ is strictly positive such that $\operatorname{Ker}(L_{K}))^{\perp} = L^{2}(\mathbb{R}^{d}, \tilde{\pi}^{\delta})$. This remains difficult since the sampling distribution $\tilde{\pi}^\delta$ is usually unknown.

The Nystr\"om interpolation is a convenient tool for associating the eigenvectors of the empirical kernel $\bm{K}_{N}$ to the eigenfunctions in $\mathcal{H}$, which leads to the desirable projection that defines our estimates. Notice that each approximated eigenfunction in \eqref{eq:vjhat} is still a linear combination of $\{K_{x_i}\}_{i=1}^{n}$. 
One advantage of spectral decomposition is that it allows one to separate the effect of finite sample size and the dimension of the hypothesis space, even when the hypothesis space is empirically constructed by interpolating the eigenvectors that depend on the data size. This is in contrast to the general kernel ridge-regression approach \cite{caponnetto2007optimal}
with hypothesis spaces that cannot be classified in terms of the dimension. Particularly, when the kernel is radial-type function, there is a lack of ordering in the corresponding set of features $\{K_x(\cdot), \forall x\in \mathbb{R}^d\}$, which is empirically estimated by $\{K_{x_i}(\cdot)\}_{i=1}^N$. By controlling the dimension of the hypothesis space (i.e., fixing the number of basis functions used in the representation), we can easily deduce the Lipschitz continuity as shown in Proposition~\ref{prop:Lip_spec}. 

One practical limitation with the projection-based method is the high computational cost in solving the eigenvalue problem associated with the empirical kernel $\mathbf{K}_N$ for large $N$. While it is desirable to have a small number of basis functions, $M$, to remedy this issue, it remains an open question which kernels can induce an RKHS space that can effectively represent the target function with a small number of basis functions. our experience indicates that a careful choice of kernels that also account for the information from the labeled data $\{y_i\}$ in addition to just the covariate data, $\{x_i\}$, is an important direction to pursue in the future study.

\section{Learning with ReLU random neural networks} \label{sec:application_ML}

In Section~\ref{sec:application}, we have discussed the kernel-based spectral regression method in learning the drift coefficients and visited various issues, including the
consistency, the
generalization error, and the Lipschitz continuity. In particular, the hypothesis space is the span of a finite number of eigenfunctions determined by the Nystr\"om interpolation. The orthogonality of the basis functions provides explicit expressions for the coefficients in minimizing the empirical risk. In this section, we will consider the \emph{random neural network} (RNN) model with the ReLU activation function. As it turns out, the hypothesis space is a convex subset of the span of a class of single-hidden-layer feed-forward networks with randomly generated coefficients. Unlike the spectral method, we determine the estimate by solving a least-squares problem. Similar to the previous section, we will focus on the issues regarding the generalization error and the Lipschitz continuity,  under the same setting.  The results on the approximation error and estimation error are mainly inspired by the work \cite{gonon2020approximation} on random neural networks and the machine-learning theory \cite{wang2011optimal, wang2012erm, cucker2002mathematical}, respectively.

\subsection{Hypothesis space and the approximation error }\label{sec:RNN_1}

Following the notations in Section~\ref{sec:application}, we recall that $b:\mathbb{R}^d\to\mathbb{R}$ denotes a generic component of the drift coefficient, and the available i.i.d. training data $\{x_{i}, y_{i}\}_{i=1}^{N}\sim (X,Y)$ satisfies $y_{i} = b(x_i) + \eta_{i}$ with the Gaussian noise $\eta_{i}\sim E$.
To introduce the hypothesis space, we define the random function $H^{\bm{A},\zeta}_{W}: \mathbb{R}^{d} \rightarrow \mathbb{R}$ by
\begin{equation}\label{eq:RNN_fun}
    H^{\bm{A},\zeta}_{W}(x) = \sum_{i=1}^{M} W_{i}\phi\left( \langle A_{i}, x \rangle + \zeta_{i} \right), \quad \phi(z): = \max\{0,z\},
\end{equation}
where $\zeta = (\zeta_{1}, \zeta_{2}, \dots, \zeta_{d})^{\top} \in \mathbb{R}^{d}$ and $\bm{A} \in \mathbb{R}^{M\times d}$ (with row vectors $A_i$) are generated randomly. Given the realization of $\bm{A}$ and $\zeta$, the coefficient vector $W\in\mathbb{R}^{M}$ is trained via empirical risk minimization. The function $\phi$ in \eqref{eq:RNN_fun} is known as the ReLU activation function. We shall point out that there are other choices of activation functions. We restrict to the ReLU to directly use the results in \cite{gonon2020approximation} regarding the approximation error.

In particular, in \cite{gonon2020approximation}, the approximation error is formulated with respect to a (essentially) compactly supported probability measure. To fulfill such an assumption, we introduce the following truncation to the invariant measure $\tilde{\pi}^{\delta}$,
\begin{equation}\label{eq:pi_D}
    \tilde{\pi}^{\delta}_{D} : = \mathbbm{1}_{B_{D}} \frac{\tilde{\pi}^{\delta}}{\tilde{\pi}^{\delta}(\mathbbm{1}_{B_{D}})}, \quad B_{D}: = \left\{ x\in \mathbb{R}^{d}\; \big| \;  \|x\|\leq D\right\}, \quad D>1,
\end{equation}
where $\mathbbm{1}_{B_{D}}(\cdot)$ denotes the characteristic function with respect to the ball $B_{D}$. By the linear growth bound in Assumption~\ref{assu:Ito_coef}, we have 
$|b(x)| \leq K_{2}\sqrt{1+D^2}$ for all $x\in B_{D}$. Fixing $K_2$, we introduce the following convex hypothesis space,
\begin{equation}\label{eq:RNN_hyp}
    \mathcal{H}^{\bm{A}, \zeta}_{D}: = \left\{  f = \mathbbm{1}_{B_{D}}H_{W}^{\bm{A}, \zeta}  \; \Big| \; \|f\|_{\infty}\leq K_{2}\sqrt{1+D^2} \right\}.
\end{equation}
Given a realization of $\bm{A}$ and $\zeta$, we may introduce the target function, with a slight abuse of the notation,
\begin{equation}\label{eq:RNN_bH}
    b_{\mathcal{H}}: = \arg\min_{h\in  \mathcal{H}^{\bm{A}, \zeta}_{D}} \| b - h\|_{\tilde{\pi}^{\delta}}^2 = \arg\min_{h\in  \mathcal{H}^{\bm{A}, \zeta}_{D}} \| b - h\|_{\tilde{\pi}_{D}^{\delta}}^2,
\end{equation}
which is well-defined since $\mathcal{H}^{\bm{A}, \zeta}_{D}$, according to the definition of the random function in \eqref{eq:RNN_fun}, is a convex subset of the $M$-dimensional function space, $\operatorname{span}\left\{\phi\left( \langle A_{i}, \cdot \rangle + \zeta_{i} \right), \; i =1,2,\dots, M \right\}$. The second identity in \eqref{eq:RNN_bH} holds because  functions in the hypothesis space are supported in $B_{D}$ and the measures $\tilde{\pi}^{\delta}$ and $\tilde{\pi}^{\delta}_{D}$ are proportional to each other in $B_{D}$.

The following result, as a direct consequence of Corollary~2 in \cite{gonon2020approximation}, specifies how to generate the random coefficients $\bm{A}$ and $\zeta$ in the RNN \eqref{eq:RNN_fun} as well as the approximation error of $b_{\mathcal{H}}$ in \eqref{eq:RNN_bH}.

\begin{prop}\label{prop:RNN_approx} Assume that $b \in C^{k}(\mathbb{R}^{d})$ for some integer $k \geq \frac{d}{2} + 1 + s$ with $s>0$. Let $T = M^{\frac{1}{2k-2s+1}}$. Suppose that the row vectors $\{A_{i}\}_{i=1}^{M}$ of the matrix $\bm{A}$ are i.i.d. samples sampled from the uniform distribution on the ball $B_{T} \subset \mathbb{R}^{d}$, and the entries of the vector $\zeta$, $\{\zeta_{i}\}_{i=1}^{M}$, are i.i.d. samples sampled from the uniform distribution on $[-DT, DT]$. Assume the two uniform distributions together with the stationary distribution $\mu^{\delta}$ are mutually independent. Then for any $\tau \in (0,1)$, with probability $1-\tau$, the target function $b_{\mathcal{H}}$ in \eqref{eq:RNN_bH} satisfies,
\begin{equation}\label{eq:RNN_approx_1}
    \|b - b_{\mathcal{H}}\|^2_{\tilde{\pi}^{\delta}}  \leq \frac{C}{\tau} M^{-\frac{2}{\alpha}} + K_{2}^2\int_{\{\|x\|>D\}} (1 +\|x\|^2) \tilde{\pi}^{\delta}(\td x), \quad \alpha = 2+ \frac{d+1}{k-\frac{d}{2}-s},
\end{equation}
for some constant $C\in (0,+\infty)$, where the integral in \eqref{eq:RNN_approx_1}  can be further bounded by
\begin{equation}\label{eq:RNN_approx_2}
    \int_{\{\|x\|>D\}} (1 +\|x\|^2) \tilde{\pi}^{\delta}(\td x) \leq C_{1}^{-\ell} (1+D^2)^{1-\ell} \tilde{\pi}^{\delta}(V).
\end{equation}
Recall that $V$ denotes the Lyapunov function satisfying Assumption~\ref{assu:2}, and the constants $\ell$ and $C_1$ are the same as in Assumption~\ref{assu:2}.
\end{prop}

\begin{proof}
To begin with, since the target function $b_{\mathcal{H}}$ is  supported in $B_{D}$, the $L^{2}$-error $\|b - b_{\mathcal{H}}\|_{\tilde{\pi}^{\delta}}^2$ is bounded as follows,
\begin{equation*}
    \|b - b_{\mathcal{H}}\|_{\tilde{\pi}^{\delta}}^2 = \tilde{\pi}^{\delta}_{D}(\mathbbm{1}_{B_{D}})^2\|b - b_{\mathcal{H}}\|_{\tilde{\pi}_{D}^{\delta}}^2 + \int_{\{\|x\|>D\}} \|b\|^2 \tilde{\pi}^{\delta}(\td x)  \leq \|b - b_{\mathcal{H}}\|_{\tilde{\pi}_{D}^{\delta}}^2+ K_{2}^2\int_{\{\|x\|>D\}} (1 +\|x\|^2) \tilde{\pi}^{\delta}(\td x),
\end{equation*}
where we have used the linear growth bound on $b$ in Assumption~\ref{assu:Ito_coef}. To apply Corollary~2 in \cite{gonon2020approximation} to $\|b - b_{\mathcal{H}}\|_{\tilde{\pi}_{D}^{\delta}}^2$, we introduce the mollification $b^{*}\in W^{k,2}(\mathbb{R}^{d})\cap L^{1}(\mathbb{R}^{d})$ of $b$ such that
\begin{equation*}
    b^{*} = b, \quad \forall x\in B_{D}.
\end{equation*}
Here, $W^{k,2}(\mathbb{R}^{d})$ indicates a Sobolev space. We have the following relation
\begin{equation*}
   b_{\mathcal{H}}=\arg\min_{h\in  \mathcal{H}^{\bm{A}, \zeta}_{D}} \| b - h\|_{\tilde{\pi}^{\delta}}^2 = \arg\min_{h\in  \mathcal{H}^{\bm{A}, \zeta}_{D}} \| b - h\|_{\tilde{\pi}_{D}^{\delta}}^2 = \arg\min_{h\in  \mathcal{H}^{\bm{A}, \zeta}_{D}} \| b^{*} - h\|_{\tilde{\pi}_{D}^{\delta}}^2,
\end{equation*}
that is, $b_{\mathcal{H}}$ \eqref{eq:RNN_bH} is also the target function of $b^{*}$ with respect to $L^{2}(\mathbb{R}^{d}, \tilde{\pi}^{\delta}_{D})$. In this case, the corresponding approximation error can be bounded by Corollary~2 in \cite{gonon2020approximation},
\begin{equation*}
    \mathbb{E}\left[\|b - b_{\mathcal{H}}\|_{\tilde{\pi}^{\delta}_{D}}^2\right]  = \mathbb{E}\left[\|b^{*} - b_{\mathcal{H}}\|_{\tilde{\pi}^{\delta}_{D}}^2\right]\leq CM^{-\frac{2}{\alpha}}, \quad \alpha = 2+ \frac{d+1}{k-\frac{d}{2}-s},
\end{equation*}
for some constant $C\in (0, +\infty)$. Here the expectation is taken with respect to the random coefficients $\bm{A}$ and $\zeta$. Together with the Markov's inequality, we reach the upper bound in \eqref{eq:RNN_approx_1}.

As for the integral in \eqref{eq:RNN_approx_1}, simply notice that by Assumption~\ref{assu:2}, one has,
\begin{equation*}
 1 +\|x\|^2 =   \frac{(1 +\|x\|^2)^{\ell}}{(1 +\|x\|^2)^{\ell-1}} \leq C_{1}^{-\ell} (1+D^2)^{1-\ell} V(x), \quad \forall x \in \{ \|x\|> D\},
\end{equation*}
which leads to the upper bound in \eqref{eq:RNN_approx_2}.
\end{proof}

We want to point out that in \cite{gonon2020approximation},  an explicit expression of the constant $C$ in \eqref{eq:RNN_approx_1} has been provided, which is proportional to the square of the norm of $b^{*}$ in $W^{k,2}(\mathbb{R}^{d})$. 

\begin{remark} \label{rm:exp}
We should also point out that the polynomial decay bound in \eqref{eq:RNN_approx_2} can be improved under additional assumptions. Particularly, if $X$ is a center random variable with sub-exponential distribution $SE(\nu^2,\alpha)$ with $\nu, \alpha>0$, then by concentration inequality for sub-exponential distribution, one obtains
\BEA
\mathbb{P}(\|X\|\geq D) \leq e^{-\frac{D}{2\alpha}}, \quad\forall D>\nu^2\alpha^{-1}. \nonumber
\EEA
This means, 
\BEA
\int_{\{\|x\|>D\}} (1 +\|x\|^2) \tilde{\pi}^{\delta}(\td x) \leq \mathbb{P}(\|X\|\geq D)^{1/2} \Big(\int_{\mathbb{R}^d} (1 +\|x\|^2)^2 \tilde{\pi}^{\delta}(\td x)   \Big)^{1/2} \leq e^{-\frac{D}{4\alpha}}  \tilde{\pi}^{\delta}(V^2)^{1/2}, \label{exponentialdecayD}
\EEA
decays exponentially in $D$ when  $\tilde{\pi}^{\delta}(V^2)<\infty$.
\end{remark}

\subsection{The Generalization error}\label{sec:RNN_gen}

Due to the truncation to the distribution $\tilde{\pi}^{\delta}$ in \eqref{eq:pi_D}, we shall define the risk function $\mathcal{E}[h]$ in \eqref{general_regression} with respect to the truncated random variables $(X_{D}, Y_{D})$, where $X_{D}$ follows the distribution $\tilde{\pi}^{\delta}_{D}$ and $Y_{D}$ is still determined by the model \eqref{discretesupervisedmodel}. Let $\mu_{D}^{\delta}$ denote their joint distribution, and the risk function in \eqref{general_regression} becomes 
\begin{equation*}
    \mathcal{E}_{D}[h]: = \mathbb{E}_{\mu^{\delta}_{D}} \left[ \| h(X_{D}) - Y_{D}\|^2 \right].
\end{equation*}
For simplicity, we assume the constant $D$ is large enough so that the samples $\{x_i\}_{i=1}^{N} \subset B_{D}$. Then, the corresponding empirical risk function, denoted as $\mathcal{E}_{D, N}$, is the same as the risk $\mathcal{E}_{N}$ in \eqref{empiricalcost}. In particular, following \eqref{eq:sig_e_1}, the empirical estimate $b_{\epsilon,D}$ is given by, 
\begin{equation}\label{eq:RNN_be}
    b_{\epsilon, D}: = \arg\min_{h\in  \mathcal{H}^{\bm{A}, \zeta}_{D}} \mathcal{E}_{D,N}[h].
\end{equation}
The least-squares problem in \eqref{eq:RNN_be} is conditionally  linear in the sense that, given a realization of $\bm{A}$ and $\zeta$, functions in $\mathcal{H}^{\bm{A}, \zeta}_{D}$ are linear combinations of basis functions $\left\{ \phi \left( \langle A_{i}, \cdot \rangle + \zeta_{i} \right) \right\}_{i=1}^{M}$. Thus, if we let 
\begin{equation}\label{eq:RNN_be_est}
    b_{\epsilon, D}(x) = \mathbbm{1}_{B_{D}}(x)\sum_{i=1}^{M} \hat{W}_{i} \phi \left( \langle A_{i}, x \rangle + \zeta_{i} \right),
\end{equation}
then $\hat{W} = (\hat{W}_{1}, \hat{W}_{2}, \dots, \hat{W}_{M})^{\top}$ solves the linear system
\begin{equation}\label{eq:RNN_linsys}
    \Phi^{\top} \Phi \hat{W} = \Phi^{\top} y, \quad y = (y_{1}, y_{2}, \dots, y_{N})^{\top}, 
\end{equation}
where $\Phi\in \mathbb{R}^{N\times M}$ such that 
\begin{equation*}
    \Phi = \left( \Phi_{ij}\right), \quad \Phi_{ij}= \phi \left( \langle A_{j}, x_{i} \rangle + \zeta_{j} \right), \quad 1\leq i \leq N, \quad 1\leq j \leq M.
\end{equation*}
In practice, when $\Phi^{\top} \Phi$ in \eqref{eq:RNN_linsys} is singular, we write $\hat{W} = ( \Phi^{\top} \Phi )^{\dagger} \Phi^{\top} y$, with $( \Phi^{\top} \Phi )^{\dagger}$ being the Moore–Penrose inverse.

Various results on the estimation error of empirical risk minimization are available under a bounded sampling assumption, e.g, \cite{cucker2002mathematical}. However, due to the Gaussian noise $\eta_{i}$ in $y_{i}$, our problem belongs to the unbounded sampling case. As a remedy, we will consider the result in \cite{wang2012erm} to study the estimation error $\mathcal{E}_{D}[b_{\epsilon,D}] - \mathcal{E}_{D}[b_{\mathcal{H}}]$.

\begin{prop}\label{prop:RNN_est}
Let $\{x_{i}, y_{i}\}_{i=1}^{N}$ be i.i.d. samples with $x_{i}\sim X_{D}$ of distribution $\tilde{\pi}^{\delta}_{D}$ and $y_{i} = b(x_{i}) + \eta_{i}$ with Gaussian noise $\eta_{i}\sim E =   \mathcal{N}(0, R)$ independent of  $\tilde{\pi}^{\delta}_{D}$. Then, for any $\tau \in (0,1)$, with probability $1-\tau$, the estimation error satisfies
\begin{equation*}
\mathcal{E}_{D}[b_{\epsilon,D}] - \mathcal{E}_{D}[b_{\mathcal{H}}] \leq \frac{B_{1}}{\sqrt{N}} \left(\ln \frac{8}{\tau^2} +M \ln\ln \frac{2}{\tau} \right)+ \frac{B_1 M}{\sqrt{N}}\left( 1 + \ln\left(\frac{4K_2B_2\sqrt{1+D^2}}{B_1}\cdot\frac{\sqrt{N}}{M } \right) \right),
\end{equation*}
for some constants $B_{1}, B_{2}\in (0,+\infty)$, which will be  specified in the proof.
\end{prop}

\begin{proof}
The result in \cite{wang2012erm} relies on the notion of \emph{covering number of} the hypothesis space $\mathcal{H}^{\bm{A}, \zeta}_{D}$ and a moment condition on the output $y_i$. 

Recall that for $\theta>0$, the covering number of $\mathcal{H}^{\bm{A}, \zeta}_{D}$ with radius $\theta$, denoted by $N(\mathcal{H}^{\bm{A}, \zeta}_{D}, \theta)$, is defined to be the minimal integer $n$ such that there exist $n$ balls with radius $\theta$ covering $\mathcal{H}^{\bm{A}, \zeta}_{D}$. Here, the topology are induced by the uniform norm $\|\cdot\|_{\infty}$. Notice that $\mathcal{H}^{\bm{A}, \zeta}_{D}$ is a bounded subset of an $M$-dimensional normed space. By Proposition~5 in \cite{cucker2002mathematical}, we have
\begin{equation}\label{eq:RNN_cover}
    \ln N(\mathcal{H}^{\bm{A}, \zeta}_{D}, \theta) \leq M \ln \left( \frac{4 K_{2}\sqrt{1+D^2} }{\theta} \right), \quad \forall \theta >0.
\end{equation}
The output $y_{i}$ yields a Gaussian distribution $\rho(y|x_{i})= \mathcal{N}(b(x_i), R)$. The moment of the output $\mathbb{E}[|b(X_{D}) + E|^{n}]$ is given by the following integral
\begin{equation*}
    \mathbb{E}\left[|b(X_{D}) + E|^{n}\right] = \int_{B_{D}}\int_{\mathbb{R}} |y|^{n} \rho(y|x) \td y\tilde{\pi}_{D}^{\delta}(\td x), \quad n \in \mathbb{N},
\end{equation*}
where for all $x\in B_{D}$,
\begin{equation*}
    \int_{\mathbb{R}} |y|^{n} \rho(y|x) \td y \leq 2^{n-1}\left(\int_{\mathbb{R}} |y-b(x)|^{n} \rho(y|x) \td y +\int_{\mathbb{R}} |b(x)|^{n} \rho(y|x) \td y \right)\leq 
    2^{n-1}\left(
    (2R)^{\frac{n}{2}} \frac{\Gamma\left(\frac{n+1}{2}\right)}{\sqrt{\pi}} + K_{2}^{n}(1+D^{2})^{\frac{n}{2}}\right).
\end{equation*}
Here, we have used the identity regarding the central absolute moments of Gaussian distribution and the linear growth bound of $b$. Since the Gamma function is monotone increasing, that is,
\begin{equation*}
    \Gamma\left(\frac{n+1}{2}\right) \leq \Gamma\left( m + \frac{1}{2} \right) = \frac{(2m-1)!!}{2^{m}}\sqrt{\pi} \leq \frac{n !!}{2^{\frac{n}{2}}} \sqrt{\pi}, \quad m = \ceil*{\frac{n}{2}}
\end{equation*}
($\ceil*{\frac{n}{2}}$ denotes the smallest integer that is greater than or equal than $\frac{n}{2}$), which suggests 
\begin{equation*}
    (2R)^{\frac{n}{2}}\frac{\Gamma\left( \frac{n+1}{2} \right)}{\sqrt{\pi}} \leq R^{\frac{n}{2}}n !! \leq R^{\frac{n}{2}}\sqrt{n!}.
\end{equation*}
Choosing,
\begin{equation}\label{eq:RNN_MD}
    M_{D}: = 2\max\left\{ K_{2} \sqrt{1+D^2}, \sqrt{R}\right\},
\end{equation}
we obtain,
\begin{equation}\label{eq:RNN_mom}
     \mathbb{E}\left[|b(X_{D}) + E|^{n}\right] \leq \sqrt{n !}M_{D}^{n}, \quad \forall n \in \mathbb{N},
\end{equation}
which fulfills the moment hypothesis in \cite{wang2012erm}.

By Propositions~2.2 and 3.3 in \cite{wang2012erm}, we see that, with probability at least $1-\tau$, there holds for all $\theta>0$
\begin{equation*}
    \begin{split}
    \mathcal{E}_{D}[b_{\epsilon,D}] - \mathcal{E}_{D}[b_{\mathcal{H}}] & \leq \left(\mathcal{E}_{D}[b_{\epsilon,D}] - \mathcal{E}_{D,N}[b_{\epsilon,D}] \right) + \left( \mathcal{E}_{D,N}[b_{\mathcal{H}}] - \mathcal{E}_{D}[b_{\mathcal{H}}] \right) \\
    & \leq \frac{B_{1}}{\sqrt{N}} \left( \ln \frac{8}{\tau^2}  + \ln N(\mathcal{H}^{\bm{A}, \zeta}_{D}, \theta) \right) +  B_{2}\theta \ln \frac{2}{\tau} \\
    &\leq \frac{B_{1}}{\sqrt{N}} \ln \frac{8}{\tau^2}+ \left(\frac{ B_{1} M}{\sqrt{N}} \ln \left( \frac{4 K_{2}\sqrt{1+D^2} }{\theta} \right) + B_{2}\theta \ln \frac{2}{\tau}\right),
    \end{split}
\end{equation*}
where the constants $B_{1}$ and $B_{2}$ are given by
\begin{equation}\label{B1B2constant}
    B_{1} = 40 K_{2}^{2}(1+D^2) + 160 M_{D}^2 \leq 180 M_{D}^2, \quad B_{2} = 4\left[K_{2}(1+D^2)^{\frac{1}{2}} + (2+3\sqrt{2})M_{D}\right] \leq (10+12\sqrt{2}) M_{D}.
\end{equation}
In particular, when
\begin{equation*}
    \theta  =  \theta^{*}:= \frac{B_{1}M}{B_{2} \ln \frac{2}{\tau} \sqrt{N} }>0
\end{equation*}
the upper bound reaches its minimum, and we have
\begin{equation*}
    \mathcal{E}_{D}[b_{\epsilon,D}] - \mathcal{E}_{D}[b_{\mathcal{H}}] \leq \frac{B_{1}}{\sqrt{N}} \ln \frac{8}{\tau^2}+ \frac{B_1 M}{\sqrt{N}}\left[ 1 + \ln\left(\frac{4K_2B_2\sqrt{1+D^2}}{B_1} \right) + \ln\ln \frac{2}{\tau} + \ln \frac{\sqrt{N}}{M } \right].
\end{equation*}
\end{proof}
It is worthwhile mentioning that the moment condition in \eqref{eq:RNN_mom} is a necessary condition for using the Bennet inequality to deduce the bounds above. Importantly, this moment bound gives a characterization of the estimation error in terms
of the noise variance $R$ through constant $M_D$ in \eqref{eq:RNN_MD} that appears in both $B_1$ and $B_2$ as shown in \eqref{B1B2constant}. Recall that in our application $R \leq \delta^{-1}Tr[\sigma\sigma^\top]$ as in \eqref{noisebound}. Using the covering number of the hypothesis space to analyze the estimation error is a classical approach in statistical learning theory, e.g., \cite{zhou2002covering}. Combining Propositions~\ref{prop:RNN_approx}-\ref{prop:RNN_est}, with probability $1-\tau$, we conclude the following upper bound for the generalization error,
\begin{equation*}
\begin{split}
\mathbb{E}_{\tilde{\pi}^{\delta}}\left[ \|b - b_{\epsilon,D}\|^2 \right] & =  \mathbb{E}_{\tilde{\pi}^{\delta}}\left[ \|b - b_{\mathcal{H}}\|^2 \right] + \left( \mathbb{E}_{\tilde{\pi}^{\delta}}\left[ \|b - b_{\epsilon,D}\|^2 \right] - \mathbb{E}_{\tilde{\pi}^{\delta}}\left[ \|b - b_{\mathcal{H}}\|^2 \right] \right) \\
& \leq  \mathbb{E}_{\tilde{\pi}^{\delta}}\left[ \|b - b_{\mathcal{H}}\|^2 \right] + \left( \mathbb{E}_{\tilde{\pi}^{\delta}_{D}}\left[ \|b - b_{\epsilon,D}\|^2 \right] - \mathbb{E}_{\tilde{\pi}^{\delta}_{D}}\left[ \|b - b_{\mathcal{H}}\|^2 \right] \right) \\
& = \mathbb{E}_{\tilde{\pi}^{\delta}}\left[ \|b - b_{\mathcal{H}}\|^2 \right] + \mathcal{E}_{D}\left[b_{\epsilon,D}\right] - \mathcal{E}_{D}\left[b_{\mathcal{H}}\right] \\
& = O\left(\tau^{-1}M^{-\frac{2}{\alpha}}\right) + O\left((1+D^2)^{1-\ell}\right) + O\left( M_{D}^2 N^{-\frac{1}{2}} \left(\ln \tau^{-2} + M \ln\ln \tau^{-1}   \right) \right) + O\left( M_{D}^2\frac{M}{\sqrt{N}}  \ln \frac{\sqrt{N}}{M} \right),
\end{split}
\end{equation*}
where the constant $M_{D}$ depends on the noise variance as defined in \eqref{eq:RNN_MD}. If we assume the invariant measure $\tilde{\pi}^{\delta}$ yields an exponential decay as in Remark~\ref{rm:exp}, then the second term above can be replaced by the error bound in \eqref{exponentialdecayD}. For fixed $\tau>0$, the leading error term is the last component. Choosing $N=O(M^2)$, the last error term is $O(M_D^2)$ which is effectively $O(R)$, where $R\leq \delta^{-1}Tr[\sigma\sigma^\top]$. This means the contribution from noise is comparable to that in the error from the kernel method in \eqref{generalkernelerror}.

\subsection{Lipschitz continuity} \label{sec:RNN_Lip}

To close this section, we check the Lipschitz continuity of the estimator $b_{\epsilon,D}$ in \eqref{eq:RNN_be}. From \eqref{eq:RNN_be_est}, we have
\begin{equation*}
    \begin{split}
    \left|b_{\epsilon,D}(x) - b_{\epsilon,D}(y)  \right| & \leq \sum_{i=1}^{M}|\hat{W}_{i}| \left|\phi\left( \langle A_{i}, x \rangle + \zeta_{i} \right) - \phi\left( \langle A_{i}, y \rangle + \zeta_{i} \right)  \right|\leq \sum_{i=1}^{M}|\hat{W}_{i}| \|A_{i}\| \|x-y\| \\
    &   \leq \left( \sum_{i=1}^{M}\hat{W}_{i}^2 \right)^{\frac{1}{2}} \left(\sum_{i=1}^{M} \|A_{i}\|^2 \right)^{\frac{1}{2}} \|x-y\| = \|\hat{W}\| \|\bm{A}\|_{F}\|x-y\|,
    \end{split}
\end{equation*}
which suggests that $b_{\epsilon,D}$ is globally Lipschitz. Recall that the matrix $\bm{A}$ is randomly generated following Proposition~\ref{prop:RNN_approx}. We have the following concentration bounds for $\|\bm{A}\|_{F}$.

\begin{prop}\label{prop:RNN_A}
Let $\bm{A} \in \mathbb{R}^{M\times d}$ with row vectors $A_{i}$ randomly generated as in Proposition~\ref{prop:RNN_approx}. Then, for $\tau>0$ we have
\begin{equation*}
    \mathbb{P}\left( \|\bm{A}\|^2_{F} \in [M \mu_{A} - \tau, M \mu_{A}+\tau] \right)\geq 1- 2\exp\left(-\frac{\tau^2}{2M\sigma_{A}^2 + \frac{2}{3}\mu_{A} \tau}  \right), \quad \mu_{A} = \frac{d}{d+2}T^2, \quad \sigma^2_{A} = \frac{4d}{(d+4)(d+2)^2}T^4,
\end{equation*}
where the constant $T$ is the same as in Proposition~\ref{prop:RNN_approx}.
\end{prop}

\begin{proof}
Recall that the row vectors $\{A_{i}\}_{i=1}^{M}$ are i.i.d. samples drawn from the uniform distribution on the ball $B_{T} \subset \mathbb{R}^{d}$. Let $\xi$ denote the random variable corresponding to $\|A_{i}\|$ with density $\rho_{T}$. Using the spherical coordinates, one  sees that $\rho_{T}(r) \propto r^{d-1}$. Moreover, $\rho_{T}$  is supported in $[0,T]$, that is,
\begin{equation*}
    \rho_{T}(r) = \frac{d}{T^{d}}r^{d-1}, \quad r \in [0,T].
\end{equation*}
Since $\|\bm{A}\|_{F}^2 = \sum_{i=1}^{M} \|A_{i}\|^{2}$, to apply concentration inequalities, we need to identify the statistics of the random variable $\xi^2$. By direct computation, we find
\begin{equation*}
   \mu_{A}:=\mathbb{E}[\xi^2] = \frac{d}{d+2}T^2, \quad \sigma^{2}_{A}: = \mathbb{E}[\xi^4] - \mathbb{E}[\xi^2]^2 = \frac{4d}{(d+4)(d+2)^2} T^{4}.
\end{equation*}
Further notice that  $\xi^2 \in [0, T^2]$ and $|\xi^2 - \mu_{A}|\leq \mu_{A} $, and by the Bernstein inequality (e.g., Proposition~2 in \cite{cucker2002mathematical}), we reach the bound in the proposition's statement.
\end{proof}

We want to point out that since $T \leq  M^{\frac{1}{d+3}}$ according to Proposition~\ref{prop:RNN_approx}, Proposition~\ref{prop:RNN_A} suggests that
\begin{equation*}
    \|\bm{A}\|_{F} = O\left(M^{\frac{d+5}{2d+6}}\right),
\end{equation*}
in high probability.

As for $\|\hat{W}\|$, since $\hat{W}$ solves the linear system in \eqref{eq:RNN_linsys} the norm $\|\hat{W}\|$ depends on the smallest positive eigenvalue of the matrix $\Phi^{\top}\Phi$. Thus, in practice, to control $\|\hat{W}\|$, a regularization is necessary. This provably improves the Lipschitz continuity of the estimates, but it comes in the expense of estimation error. For general discussions on the Lipschitz function approximation using neural network, we refer the readers to \cite{anil2019sorting} and the references therein.

\subsection{Remarks on the ReLU random neural networks}

The ReLU RNN approach reviewed here has several advantages. First, compared with the spectral regression method in Section~\ref{sec:application}, the RNN approach is numerically cheaper to implement since it does not require solving large eigenvalue problems. Given the realization of $\bm{A}$ and $\zeta$, the target function $b_{\mathcal{H}}$ can be properly defined as the minimizer of the least-squares problem in \eqref{eq:RNN_bH} with no extra assumption on the unknown function $b$. The existence of a target function simplifies the discussion of the generalization error.

An obvious drawback of the RNN approach is that the estimates $b_{\epsilon,D}$ \eqref{eq:RNN_be_est} has compact support, while the unknown drift coefficients $b(x)$ is often unbounded as $\|x\|\rightarrow +\infty$. Under such constructions, the consistency assumption \eqref{eq:error_linear} in Assumption~\ref{assu:coe_per} can never be satisfied beyond the compact support. While the error beyond the compact support decays, either polynomial (see \eqref{eq:RNN_approx_2}) or exponential (see \eqref{exponentialdecayD}) under additional assumptions, as a function of the radius of the ball, in practice, we may not be able to estimate on a domain with large $D$. This issue is due to the difficulty in obtaining training samples on the tail of the distribution $X$. Additionally, larger $D$ induces a larger estimation error through the constant $M_D$ in \eqref{eq:RNN_MD}.
 
\comment{\color{red}This sentence can be changed as well:  However, recall that such a compact support is artificially introduced by truncating of the invariant measure in \eqref{eq:pi_D}
such that we can specify the generalization error using the existing results. Since the random function $H^{\bm{A},\xi}_{W}$ in \eqref{eq:RNN_fun} is defined on entire $\mathbb{R}^{d}$, similar to the construction in Eqs. \eqref{eq:RNN_be}-\eqref{eq:RNN_linsys}, one can develop an empirical estimate on the entire domain, $\mathbb{R}^{d}$, under an additional assumption on the sampling distribution, e.g., sub-exponential class such that in high-probability the data is concentrated on a ball of radius $D$. In such a case, however, the generalization error is unknown.
}

\section{Summary}
\label{sec:summary}

In this paper, we studied the error bounds of the invariant statistics in learning ergodic It\^o diffusion. Using the perturbation theory of ergodic Markov chains \cite{rudolf2018perturbation,shardlow2000perturbation} and the linear response theory \cite{hairer2010simple}, we established a linear dependence of the errors of one-point and two-point invariant statistics on the spectral error of the diffusion matrix estimator. Under a proper consistency condition on the estimator of the drift coefficient, one can identify the error bound in terms of the size of the training sample, ``size'' of hypothesis space, noise amplitude, and the discretization error induced by the SDE's solver, using standard $L^2$ generalization error
analysis corresponding to the specific machine learning algorithm. An important takeaway point from this study is that the $L^2$ characterization of the learning generalization error is not sufficient for achieving the linear dependence error bound presented in this paper. Besides the consistent estimator in the hypothesis space that retains certain characteristics of the drift coefficients, a sufficient condition to achieve the error bound here is through learning algorithms that produce a uniformly Lipschitz.

From our examination of two simple learning algorithms, the kernel-based spectral regression method and the shallow random neural networks with the ReLU activation function, we conclude that to satisfy these sufficient conditions, one needs to specify the hypothesis space carefully to avoid bias. In practice, the main challenge will be in the kernel specification when a kernel-based algorithm is used. For random neural network, the consistency depends on the universality of the random bases \cite{rahimi2008uniform}. Another practical issue is to overcome biased estimation with the random neural networks that arises when sampling data on the tail of the distribution are not available. These issues suggest that it is important to have a physical understanding of the problem at hand to have appropriate hypothesis space for convergence guarantees rather than just using machine learning as a black-box. Furthermore, consistent estimates can only be achieved on the domain where the training data is available. 

We view this study as a first step to understand machine learning of dynamical systems with invariant statistical properties. Many issues remain open. For example, in our study, we consider a consistent supervised learning problem in the following sense. Specifically, we model the response variable $y$ in \eqref{discretesupervisedmodel} to be compatible with the integration Euler-Maruyama scheme. In practice, when the underlying scheme is not known, the inconsistent learning model will produce a biased estimator. A much harder yet significant problem is to carry this analysis on deterministic dynamical systems. In this context, the existence of the invariant measure of the estimated dynamics is an essential question in dynamical system theory \cite{young2002srb}. Furthermore, the validity of linear response theory is also a critical problem \cite{baladi2014} that is difficult to justify in general.

\section*{Acknowledgment}

The research of JH was partially supported under the NSF grant DMS-1854299. XL is supported by NSF grant DMS-1819011.

\appendix

\section{Proof of Lemma~\ref{lem:well_pose}}
\label{app:gron}

In this appendix, we discuss the proof of Lemma~\ref{lem:well_pose} in Section~\ref{sec:per_theory}. For reader's convenience, we first review a nonlinear generalization of Gronwall's inequality.

\begin{prop}\label{prop:gron}
Let $v(t)$ be a nonnegative function that satisfies the integral inequality
\begin{equation*}
    v(t) \leq c + \int_{t_{0}}^{t} \left(a_{1}(s) v(s)  + a_{2}(s)v^{\alpha}(s) \right) \td s, \quad c\geq 0, \quad \alpha \geq 0,
\end{equation*}
where $a_1(t)$ and $a_2(t)$ are continuous nonnegative functions on $[t_0, +\infty)$. For $0\leq \alpha <1$, we have
\begin{equation*}
    v(t) \leq \left\{c^{1-\alpha} \exp \left[(1-\alpha) \int_{t_0}^{t} a_1(s) \td s  \right] + (1-\alpha) \int_{t_0}^{t} a_2(s) \exp \left[(1-\alpha) \int_{s}^{t} a_1(r)\td r  \right]\td s \right\}^{\frac{1}{1-\alpha}}.
\end{equation*}
\end{prop}
When $\alpha = 1$, such a result reduces to the standard Gronwall's inequality. The proof is an application of the Bernoulli equation (e.g., Theorem 21 in \cite{dragomir2002some}).

 {\bf Proof of Lemma~\ref{lem:well_pose}:} We start with the case: $(u_{n}, u_{n}^{\epsilon}) = (X_{n}, X_{n}^{\epsilon})$. By the definition of $\mathcal{G}_{\ell}$ in Proposition~\ref{prop:erg_euler}, we have
\begin{equation}\label{eq:G_ell}
|f(x)-f(y)| \leq C_{\ell}\left(1+ \|x\|^{2\ell-1} + \|y\|^{2\ell-1}\right)\|x-y\|, \quad \forall x,y\in \mathbb{R}^{d}, \; \forall f\in \mathcal{G}_{\ell},
\end{equation}
where $C_{\ell}\in (0, +\infty)$ is a fixed constant independent of $f$. Taking expectation on \eqref{eq:G_ell}, and employing the Cauchy-Schwarz inequality, we obtain
\begin{equation*}
\left| \mathbb{E}^{x}[f(X_{1})] - \mathbb{E}^{x}[f(X_{1}^{\epsilon})]  \right| \leq  C_{\ell} \left(\mathbb{E}^{x}\left[\left(1+ \|X_{1}\|^{2\ell-1} + \|X_{1}^{\epsilon}\|^{2\ell-1}\right)^2\right] \right)^{\frac{1}{2}} \left(  \mathbb{E}^{x}\left[\|X_{1} - X_{1}^{\epsilon}\|^{2}\right] \right)^{\frac{1}{2}},
\end{equation*}
Further notice that
\begin{equation*}
    \mathbb{E}^{x}\left[\left(1+ \|X_{1}\|^{2\ell-1} + \|X_{1}^{\epsilon}\|^{2\ell-1}\right)^2\right] = \mathbb{E}^{x}\left[1+ \|X_{1}\|^{4\ell-2} + \|X_{1}^{\epsilon}\|^{4\ell-2}\right] + 2  \mathbb{E}^{x}\left[ \|X_{1}\|^{2\ell-1} + \|X_{1}^{\epsilon}\|^{2\ell-1}+\|X_{1}\|^{2\ell-1}\|X_{1}^{\epsilon}\|^{2\ell-1}\right],
\end{equation*}
where the last term can be bounded as follows,
\begin{equation*}
    \mathbb{E}^{x}\left[ \|X_{1}\|^{2\ell-1}\|X_{1}^{\epsilon}\|^{2\ell-1}\right] \leq \left(\mathbb{E}^{x}\left[ \|X_{1}\|^{4\ell-2}\right]  \right)^{\frac{1}{2}}\left(\mathbb{E}^{x}\left[ \|X_{1}^{\epsilon}\|^{4\ell-2}\right]  \right)^{\frac{1}{2}} \leq \frac{1}{2}\left(\mathbb{E}^{x}\left[ \|X_{1}\|^{4\ell-2}\right] + \mathbb{E}^{x}\left[ \|X_{1}^{\epsilon}\|^{4\ell-2}\right] \right).
\end{equation*}
To bound the remaining order-$(2\ell-1)$ moments of $X_{1}$ and $X_{1}^{\epsilon}$, we apply the Jensen's inequality, 
\begin{equation*}
    \mathbb{E}^{x}\left [\|A\|^{2\ell-1} \right] \leq \mathbb{E}^{x}\left [\|A\|^{4\ell-2} \right]^{\frac{1}{2}}\leq  \max\left\{\mathbb{E}^{x}\left [\|A\|^{4\ell-2} \right],1\right\}, 
\end{equation*}
(for $A = X_{1}$ and $X_{1}^{\epsilon}$, respectively), which leads to
\begin{equation}\label{eq:three_errors_1}
\left| \mathbb{E}^{x}[f(X_{1})] - \mathbb{E}^{x}[f(X_{1}^{\epsilon})]  \right| \leq  R_{1} \left(\mathbb{E}^{x}\left[1+ \|X_{1}\|^{4\ell-2} + \|X_{1}^{\epsilon}\|^{4\ell-2} \right]\right)^{\frac{1}{2}} \left( \mathbb{E}^{x}\left[\|X_{1} - X_{1}^{\epsilon}\|^{2}\right] \right)^{\frac{1}{2}}
\end{equation}
for some constant $R_{1}\in (0,+\infty)$ independent of $\epsilon$. On the right-hand side of Eq.~\eqref{eq:three_errors_1}, the moments are bounded by applying Lemma~\ref{lem:mom_Ito} to $X(t)$ and $X^{\epsilon}(t)$, respectively. As for the error $\mathbb{E}^{x}\left[\|X_{1} - X_{1}^{\epsilon}\|^{2}\right]$ in \eqref{eq:three_errors_1}, we will derive an integral inequality of the form in Proposition~\ref{prop:gron} using It\^o formula \cite{oksendal2013stochastic}.

Notice that since $X(t)$ in \eqref{eq:unper} and $X^{\epsilon}(t)$ in \eqref{eq:per} are driven by the same Brownian motion $W$, we can consider the following coupled SDEs
\begin{equation}\label{eq:coupled}
    \frac{\td }{\td t}\begin{pmatrix} X \\ X^{\epsilon}  \end{pmatrix} = \begin{pmatrix} b(X) \\ b_{\epsilon}(X^{\epsilon})  \end{pmatrix} + \begin{pmatrix} \sigma \\ \sigma_{\epsilon}  \end{pmatrix} \dot{W}, \quad \begin{pmatrix} X(0) \\ X^{\epsilon}(0)  \end{pmatrix} = \begin{pmatrix} x \\ x  \end{pmatrix}.
\end{equation}
Recall that in Section~\ref{sec:learning_SDE}, we estimated $\sigma\sigma^{\top}$ by $\sigma_{\epsilon}\sigma_{\epsilon}^{\top}$ without specifying the diffusion coefficient $\sigma_{\epsilon}$. Here, we shall take $\sigma_{\epsilon}$ so that $\|\sigma - \sigma_{\epsilon}\|_{F} = O(\epsilon)$ for the sake of the proof. To construct such $\sigma_{\epsilon}$, we introduce the \emph{thin} SVD factorization of $\sigma$ (recall that $\sigma \in \mathbb{R}^{d\times m}$ is full rank with $m\leq d$),
\begin{equation}\label{eq:SVD}
    \sigma = U \Lambda V^{\top}, \quad U \in \mathbb{R}^{d\times m}, \quad \Lambda \in \mathbb{R}^{m \times m}, \quad V \in \mathbb{R}^{m \times m},
\end{equation}
where $\Lambda = \operatorname{diag}(\sigma_1, \sigma_{2}, \dots, \sigma_{m})$ with $\sigma_{i}>0$, $\forall i$. With the SVD factorization, we get  $\sigma\sigma^{\top} = U \Lambda^2  U^{\top}$. Given the estimate $\sigma_{\epsilon}\sigma_{\epsilon}^{\top}$ of $\sigma\sigma^{\top}$ satisfying \eqref{eq:epsilon_def}, we define 
\begin{equation*}
    \Sigma_{\epsilon}: = U^{\top} \sigma_{\epsilon}\sigma_{\epsilon}^{\top} U \in \mathbb{R}^{m\times m},
\end{equation*}
and we have
\begin{equation}\label{eq:Sigma_e}
    \| \Lambda^2  - \Sigma_{\epsilon}\|_{2} = \| U^{\top} (\sigma\sigma^{\top} -\sigma_{\epsilon}\sigma_{\epsilon}^{\top}) U\|_2 = \| \sigma\sigma^{\top} - \sigma_{\epsilon}\sigma_{\epsilon}^{\top}\|_{2} = \epsilon.
\end{equation}
Since $\Lambda^2 = \operatorname{diag}(\sigma_1^2, \sigma_{2}^2, \dots, \sigma_{m}^2)$ is positive definite, for $\epsilon$ small enough, $\Sigma_{\epsilon}$ is also positive definite according to \eqref{eq:Sigma_e}. Thus, there exists a unique lower triangular Cholesky factorization of $\Sigma_{\epsilon}$, e.g., Theorem 4.2.7 in \cite{golub2013matrix}, with the lower triangular matrix denoted by $L_{\epsilon}\in \mathbb{R}^{m\times m}$. Namely, 
\begin{equation*}
     L_{\epsilon}L_{\epsilon}^{\top}= \Sigma_{\epsilon}.
\end{equation*}
Moreover, by the forward stability of the Cholesky factorization subject to small perturbations, e.g., Theorem~2.1 in \cite{drmavc1994perturbation}, we have
\begin{equation}\label{eq:L_Le}
     \|\Lambda - L_{\epsilon}\|_{F} \leq \| \Lambda^2  - \Sigma_{\epsilon}\|_{F} \leq \sqrt{m} \| \Lambda^2  - \Sigma_{\epsilon}\|_{2} = \sqrt{m}\epsilon.
\end{equation}
In other words, $L_{\epsilon}$ approximates the matrix $\Lambda$ in the SVD factorization of $\sigma$ \eqref{eq:SVD}. Thus, by replacing $\Lambda$ in \eqref{eq:SVD} by $L_{\epsilon}$, we define
\begin{equation}\label{eq:def_sigma_e}
    \sigma_{\epsilon}: = U L_{\epsilon} V^{\top},
\end{equation}
which satisfies
\begin{equation*}
    \left(U L_{\epsilon} V^{\top}\right) \left(U L_{\epsilon} V^{\top}\right)^{
    \top} = U \Sigma_{\epsilon} U^{\top} = \sigma_{\epsilon}\sigma_{\epsilon}^{\top}.
\end{equation*}
In particular, by \eqref{eq:L_Le}, we have
\begin{equation}\label{eq:sigma_err_bound}
    \|\sigma - \sigma_{\epsilon}\|_{F} \leq \sqrt{m} \|\sigma - \sigma_{\epsilon}\|_{2} = \sqrt{m} \|U(\Lambda - L_{\epsilon})V^{\top}\|_{2} \leq \sqrt{m} \|\Lambda - L_{\epsilon}\|_{2} \leq m \epsilon.
\end{equation}
We assign $\sigma_{\epsilon}$ in \eqref{eq:def_sigma_e} to the coupled system \eqref{eq:coupled}, and apply the It\^o formula to the process
\begin{equation*}
    U(t) : = \|X(t) - X^{\epsilon}(t)\|^{2}, \quad U(0)  = 0, \quad 0 \leq t \leq \delta.
\end{equation*}
Direct calculations yield,
\begin{equation*}
\begin{split}
    \dot{U} & = 2 \langle X - X^{\epsilon},  b(X) - b_{\epsilon}(X^{\epsilon})\rangle +  2 \langle X - X^{\epsilon},  \left(\sigma - \sigma_{\epsilon} \right)\dot{W}\rangle  + \langle \left(\sigma - \sigma_{\epsilon}\right) \dot{W},  \left(\sigma - \sigma_{\epsilon}\right) \dot{W}\rangle \\
    & = 2 \langle X - X^{\epsilon},  b(X) - b_{\epsilon}(X^{\epsilon})\rangle + \|\sigma - \sigma_{\epsilon}\|_{F}^{2} + 2 \langle X - X^{\epsilon},  \left(\sigma - \sigma_{\epsilon} \right)\dot{W}\rangle,
\end{split}
\end{equation*}
where $\langle \cdot, \cdot \rangle$ denotes the inner product in $\mathbb{R}^{d}$. This can be rewritten as an  It\^o integral representation of $U(t)$,
\begin{equation*}
    U(t) = \int_{0}^{t} 2 \langle X(s) - X^{\epsilon}(s),  b(X(s)) - b_{\epsilon}(X^{\epsilon}(s))\rangle + \|\sigma - \sigma_{\epsilon}\|_{F}^{2} \td s + \int_{0}^{t} (\sigma - \sigma_{\epsilon})^{\top}(X- X^{\epsilon}) \cdot \td W.
\end{equation*}    
Here,
\begin{equation*}
\int_{0}^{t} (\sigma - \sigma_{\epsilon})^{\top}(X - X^{\epsilon}) \cdot \td W: = \sum_{i=1}^{m} \int_{0}^{t} (\sigma - \sigma_{\epsilon})_{i}^{\top} (X - X^{\epsilon}) \td W_{i},
\end{equation*}
where $(\sigma - \sigma_{\epsilon})_{i} \in \mathbb{R}^{d}$ denotes the $i$-th column vector of the matrix $(\sigma - \sigma_{\epsilon})$. Using the It\^o-isometry \cite{oksendal2013stochastic}, we have
\begin{equation*}
    \mathbb{E}^{x}\left[ \int_{0}^{t} (\sigma - \sigma_{\epsilon})^{\top}(X- X^{\epsilon}) \cdot \td W  \right] \leq \mathbb{E}^{x}\left[ \int_{0}^{t} \|(\sigma - \sigma_{\epsilon})^{\top}(X - X^{\epsilon})\|^2 \td s  \right]^{\frac{1}{2}} \leq \| \sigma - \sigma_{\epsilon}\|_{F} \int_{0}^{t} \mathbb{E}^{x}[U(s)]^{\frac{1}{2}} \td s.
\end{equation*}
Further observe that $b(X) - b_{\epsilon}(X^{\epsilon}) = (b(X) - b(X^{\epsilon})) + (b(
X^{\epsilon}) - b_{\epsilon}(X^{\epsilon}))$. By the Lipschitz continuity of $b$ in Assumption~\ref{assu:Ito_coef} and the linear growth bound of $b - b_{\epsilon}$ in Assumption~\ref{assu:coe_per}, and we have
\begin{equation*}
    \begin{split}
        \left|\langle X - X^{\epsilon},  b(X) - b_{\epsilon}(X^{\epsilon})\rangle \right| & \leq \left| \langle X - X^{\epsilon},  b(X) - b(X^{\epsilon})\rangle \right|+ \left|\langle X - X^{\epsilon},  b(X^{\epsilon}) - b_{\epsilon}(X^{\epsilon})\rangle \right| \\
        & \leq K_{1} \| X - X^{\epsilon}\|^2 + \epsilon K_{3} (1 + \|X^{\epsilon}\|^2)^{\frac{1}{2}} \| X - X^{\epsilon}\|.
    \end{split}
\end{equation*}
Combining these inequalities, we arrive at,
\begin{equation*}
\mathbb{E}^{x}[U(t)] \leq t\|\sigma - \sigma_{\epsilon}\|_{F}^{2}  + 2\int_{0}^{t} K_{1}\mathbb{E}^{x}[U(s)] + \left(\epsilon  K_{3}\mathbb{E}^{x}\left[(1+\|X^{\epsilon}(s)\|^2)^{\frac{1}{2}}\right] + \|\sigma-\sigma_{\epsilon}\|_{F} \right) \mathbb{E}^{x}[U(s)]^{\frac{1}{2}} \td s, \quad \forall t \in [0,\delta].
\end{equation*}
Thus, the function $v(t;x): = \mathbb{E}^{x}[U(t)]$ satisfies the following integral inequality 
\begin{equation*}
    v(t;x) \leq  c + \int_{0}^{t} \left( a_1 v(s) + a_2v^{\frac{1}{2}}(s) \right) \td s,  \quad \forall t\in (0, \delta),
\end{equation*}
with
\begin{equation*}
    c =\delta m^2\epsilon^2, \quad a_1 = 2 K_{1}, \quad a_2 = R_2(1+ \|x\|^2)^{\frac{1}{2}} \epsilon,
\end{equation*}
where the constant $R_{2} \in (0, +\infty)$ is independent of $\epsilon$. Here, we have used the error bound on $\|\sigma - \sigma_{\epsilon}\|_{F}$ in Eq.~\eqref{eq:sigma_err_bound} to get $c$ and the moment bound of $\mathbb{E}^{x}\left[ \|X^{\epsilon}\|^{2}\right]$ (by applying Lemma~\ref{lem:mom_Ito} to $X^{\epsilon}$) to get $a_2$. By the Gronwall-type inequality in Proposition~\ref{prop:gron}, we conclude the following bound for $v$,
\begin{equation*}
    \mathbb{E}^{x}\left[\|X(t) - X^{\epsilon}(t)\|^{2}\right] = v(t;x) \leq R_{3}e^{2K_{1}t} (1+\|x\|^{2}) \epsilon^2, \quad \forall t\in [0, \delta],
\end{equation*}
where the constant $R_{3} \in (0, +\infty)$ is independent of $\epsilon$. Thus, Eq.~\eqref{eq:three_errors_1} becomes
\begin{equation}
    \left| \mathbb{E}^{x}[f(X_{1})] - \mathbb{E}^{x}[f(X_{1}^{\epsilon})]  \right| \leq  R_{1}R_{3}^{\frac{1}{2}}e^{K_{1}\delta} (1+\|x\|^{2})^{\frac{1}{2}}  \left(\mathbb{E}^{x}\left[1+ \|X_{1}\|^{4\ell-2} + \|X_{1}^{\epsilon}\|^{4\ell-2} \right]\right)^{\frac{1}{2}}\epsilon \leq K V(x) \epsilon,
\end{equation}
for some constant $K\in (0, +\infty)$ independent of $\epsilon$. In the last inequality, we have used the moment bounds on $X_{1}$ and $X_{1}^{\epsilon}$ based on Lemma~\ref{lem:mom_Ito} and Assumption~\ref{assu:2} on the Lyapunov function $V$. Finally, by the definition of $\gamma_{X}$, we have,
\begin{equation*}
    \gamma_{X} = \sup_{x\in \mathbb{R}^{d}} \sup_{f\in \mathcal{G}_{\ell}}  \frac{\left| \mathbb{E}^{x}[f(X_{1})] - \mathbb{E}^{x}[f(X^{\epsilon}_1)]\right|}{V(x)}\leq K\epsilon,
\end{equation*}
which is the desirable result.

When $(u_{n}, u_{n}^{\epsilon}) = (X_{n}^
\delta, X_{n}^{\epsilon,\delta})$, as an analogy of \eqref{eq:three_errors_1}, we have
\begin{equation*}
    \left| \mathbb{E}^{x}[f(X^{\delta}_{1})] - \mathbb{E}^{x}[f(X_{1}^{\epsilon, \delta})]  \right| \leq  R_{1} \left(\mathbb{E}^{x}\left[1+ \|X^{\delta}_{1}\|^{4\ell-2} + \|X_{1}^{\epsilon,\delta}\|^{4\ell-2} \right]\right)^{\frac{1}{2}} \left( \mathbb{E}^{x}\left[\|X^{\delta}_{1} - X_{1}^{\epsilon, \delta}\|^{2}\right] \right)^{\frac{1}{2}}.
\end{equation*}
The bound for $\mathbb{E}^{x}\left[\|X^{\delta}_{1} - X_{1}^{\epsilon, \delta}\|^{2}\right]$ can be derived via direct computations. We apply EM scheme \eqref{eq:Euler_approx} to the coupled system \eqref{eq:coupled}, and we have
\begin{equation*}
\begin{split}
\mathbb{E}^{x}\left[\|X^{\delta}_{1} - X_{1}^{\epsilon, \delta}\|^{2}\right] &  = \mathbb{E}\left[\|\delta (b(x) - b_{\epsilon}(x)) + \sqrt{\delta} (\sigma - \sigma_{\epsilon})\xi_1 \|^{2}\right]  \\
& = \delta^2 \|b(x) - b_{\epsilon}(x)\|^2 + \delta \mathbb{E}\left[\| (\sigma - \sigma_{\epsilon})\xi_1 \|^{2}\right]   \leq R_{4}(1+\|x\|^2) \epsilon^2,
\end{split}
\end{equation*}
for some constant $R_{4}\in (0, +\infty)$ independent of $\epsilon$. Repeating the argument for $\gamma_{X}$, we reach the same desirable result for $\gamma_{X^\delta}$.

\section{Proof of Proposition~\ref{prop:two_point}}
\label{App:two_point}

In this Appendix, we discuss the proof of Proposition~\ref{prop:two_point}, which provides both the well-posedness of the two-point statistics of the perturbed dynamics and the corresponding error bound.

By the finite second moments assumption, we know the two-point statistics in \eqref{eq:two_point_error} are well-defined. In particular, we have
\begin{equation*}
\begin{split}
& (k^{\epsilon}_{A,B})_{n} - (k_{A,B})_{n}  = \int\int A(x)B(x_0)P^{\epsilon}_{n\delta}(x_0,\td x)\pi^{\epsilon}(\td x_0) - \int\int A(x)B(x_0)P_{n\delta}(x_0,\td x) \pi(\td x_0) \\
& = \int\int A(x)B(x_0)\left(P^{\epsilon}_{n\delta}(x_0,\td x)-P_{n\delta}(x_0,\td x) \right) \pi^{\epsilon}(\td x_0) +  \int\int A(x)B(x_0)P_{n\delta}(x_0,\td x)(\pi^{\epsilon} -\pi)(\td x_0)= : I_1 + I_2.
\end{split}
\end{equation*}
For the term $I_1$, we have
\begin{equation}\label{eq:bound_I1}
\begin{split}
|I_1|& \leq \int |B(x_0)|\left|\int A(x)\left(P^{\epsilon}_{n\delta}(x_0,\td x)-P_{n\delta}(x_0,\td x) \right)\right| \pi^{\epsilon}(\td x_0) = \int|B(x_0)|\left|\mathbb{E}^{x_0}[A(X_{n}^{\epsilon})] - \mathbb{E}^{x_0}[A(X_{n})]\right| \pi^{\epsilon}(\td x_0) \\
& \leq  R_{1} \pi^{\epsilon}(|B|V) \epsilon \leq R_{1}  (\pi^{\epsilon}(B^{2}))^{\frac{1}{2}} (\pi^{\epsilon}(V^2))^{\frac{1}{2}} \epsilon,
    \end{split}
\end{equation}
for some constants $R_{1}\in (0,+\infty)$ and $D\in (1,+\infty)$ (independent of $A$ and $B$). Here we have applied Proposition~\ref{prop:per_bound} to $\left|\mathbb{E}^{x_0}[A(X_{n}^{\epsilon})] - \mathbb{E}^{x_0}[A(X_{n})]\right|$ based on Lemma~\ref{lem:well_pose}.

Meanwhile for the term $I_{2}$, we have
\begin{equation*}
I_{2} = \int B(x_0)\int A(x)P_{n\delta}(x_0,\td x)(\pi^{\epsilon} -\pi)(\td x_0) = \int B(x_0) E^{x_{0}}[A(X_{n})] (\pi^{\epsilon} -\pi)(\td x_0) = \pi^{\epsilon}(f_n) - \pi(f_{n}),
\end{equation*}
where $f_{n}(x) = B(x)\mathbb{E}^{x}[A(X_{n})]$. Here, $f_{n}$, in general, is not a function in $\mathcal{G}_{\ell}$, and we cannot apply the the existing one-point statistics error bound derived from the perturbation theory, e.g., Proposition~\ref{prop:per_bound}. As a remedy, we consider the long-time linear response theory reviewed in Section~\ref{sec:lin_resp}.

Before applying Theorem~\ref{theo:lin_resp} to $f_n$, we need to show that $f_{n}\in C_{G,H}^{1}$. Let
\begin{equation*}
    A_{n}(x): = \mathbb{E}^{x}[A(X_{n})] - \pi(A).
\end{equation*}
Notice $A\in C^{1}_{V,VH/G}\subset C^{1}_{G,H}$ (since $G\geq V$), and, by the spectral gap assumption (Assumption~\ref{assu:4}), we have
\begin{equation}\label{eq:spec_gap}
    \|A_{n}\|_{1;G,H}  = \left\|\mathcal{P}_{n\delta}^{0}A - \pi(A)\right\|_{1;G,H} \leq \lambda^{n}\|A - \pi(A)\|_{1;G,H}, \quad \lambda \in (0,1).
\end{equation}
With this bound, we turn to $f_n$,
\begin{equation*}
    \left\|f_{n}\right\|_{1;G,H}  = \|BA_{n} + B\pi(A)\|_{1;G,H} \leq \|BA_{n}\|_{1;G,H} +\pi(A)\|B\|_{1;G,H},
\end{equation*}
where $\|B\|_{1;G,H} < \infty$ since $B\in C^{1}_{G/V,H/V}\subset C^{1}_{G,H}$. Therefore, to show $f_{n}\in C^{1}_{G,H}$ it is enough to control the norm $\|BA_{n}\|_{1;G,H}$. By the definition \eqref{eq:GH_norm}, we have
\begin{equation*}
    \|BA_{n}\|_{1;G,H} = \sup_{x\in \mathbb{R}^{d}}\left\{\frac{|BA_{n}|}{G} + \frac{\|\nabla B A_{n} + B\nabla A_{n}\|}{H} \right\}.
\end{equation*}
Notice that, 
\begin{equation*}
  \frac{|BA_{n}|}{G} + \frac{\|\nabla B A_{n} + B\nabla A_{n}\|}{H} \leq |B| \left( \frac{|A_{n}|}{G} + \frac{\|\nabla A_{n} \|}{H} \right) + \frac{\|\nabla B\|}{H/G} \frac{|A_n|}{G}\leq \left( \frac{|A_{n}|}{G} + \frac{\|\nabla A_{n} \|}{H} \right)\left( |B| + \frac{\|\nabla B\|}{H/G}  \right),
\end{equation*}
where, by Eq.~\eqref{eq:spec_gap},
\begin{equation*}
    \sup_{x\in \mathbb{R}^{d}} \left\{ \frac{|A_{n}|}{G} + \frac{\|\nabla A_{n} \|}{H} \right\}  = \|A_{n}\|_{1;G,H} \leq \lambda^{n} \|A - \pi(A)\|_{1;G,H} = \lambda^{n}\sup_{x\in \mathbb{R}^{d}} \left\{ \frac{|A - \pi(A)|}{G} + \frac{\|\nabla A \|}{H} \right\}.
\end{equation*}
Thus, we obtain the following bound,
\begin{equation*}
 \|BA_{n}\|_{1;G,H}  \leq \lambda^{n}\sup_{x\in \mathbb{R}^{d}} \left\{ \left(\frac{|A - \pi(A)|}{G} + \frac{\|\nabla A \|}{H} \right) \left( |B| + \frac{\|\nabla B\|}{H/G}  \right) \right\}
 = \lambda^{n}\sup_{x\in \mathbb{R}^{d}} \left\{ \left(\frac{|A - \pi(A)|}{V} + \frac{\|\nabla A \|}{VH/G} \right) \left( \frac{|B|}{G/V} + \frac{\|\nabla B\|}{H/V}  \right) \right\}.
\end{equation*}
This shows that,
\begin{equation}\label{eq:GH_fn}
    \left\|f_{n}\right\|_{1;G,H} \leq \lambda^{n} \|A-\pi(A)\|_{1;V,VH/G} \|B\|_{1;G/V,H/V} + \pi(A)\|B\|_{1;G,H}<\infty.
\end{equation}
With $f_{n}\in C^{1}_{G,H}$, by invoking Theorem~\ref{theo:lin_resp}, we have
\begin{equation*}
    |I_{2}| = \left|\pi^{\epsilon}(f_n) - \pi(f_{n}) \right| \leq \left|\frac{\td }{\td \;\epsilon}  \pi^{\epsilon}(f_{n})\Big|_{\epsilon =0} \right|\epsilon + O(\epsilon^2),
\end{equation*}
where the $\epsilon$-derivative is well-defined and satisfies Eq.~\eqref{eq:lin_resp1} with $f = f_n$ and $t= \delta$.

Given $\pi(A)=0$, Eq.~\eqref{eq:GH_fn} reduces to
\begin{equation*}
   \left\|f_{n}\right\|_{1;G,H} \leq \lambda^{n}\|A\|_{1;V,VH/G} \|B\|_{1;G/V,H/V},
\end{equation*}
and the $\epsilon$-derivative in \eqref{eq:two_point_error}, by Eq.~\eqref{eq:lin_resp1}, satisfies
\begin{equation*}
    \frac{\td }{\td \; \epsilon} \pi^{\epsilon}(f_{n})\Big|_{\epsilon=0} = \mathbb{E}_{\pi} \left[ \partial\mathcal{P}_{\delta}^{0}(I - \mathcal{P}_{\delta}^{0})^{-1}\left(f_{n}-\pi(f_n)\right)\right],
\end{equation*}
where 
\begin{equation*}
    \left\|f_{n}-\pi(f_{n})\right\|_{1;G,H} \leq \left\|f_{n}\right\|_{1;G,H} + |\pi(f_{n})| \leq  \lambda^{n}\|A\|_{1;V,VH/G} \|B\|_{1;G/V,H/V} + |\pi(f_{n})|, \quad \forall n\geq 0.
\end{equation*}
To bound $\pi(f_{n})$, since $A\in \mathcal{G}_{\ell}\subset \mathcal{G}$, we apply Theorem~\ref{thm:ergodic} to $A$, 
\begin{equation*}
    \left|f_{n}(x)\right| = |B(x)|\left|\mathbb{E}^{x}[A(X_{n})]\right| = |B(x)|\left|\mathbb{E}^{x}[A(X_{n})]-\pi(A)\right|\leq R_{2}\rho^{n}|B(x)|V(x)\leq R_{2}\rho^{n}\|B\|_{1;G/V,H/V} G(x), \quad \forall x\in \mathbb{R}^{d},
\end{equation*}
for some constants $R_{2}\in (0,\infty)$ and $\rho = \rho(\delta) \in (0,1)$ independent of $A$. Here, $|B(x)|V(x)\leq \|B\|_{1;G/V,H/V} G(x)$ by the definition of the norm $\|\cdot\|_{G/V,H/V}$. Thus, $|\pi(f_{n})| \leq R_{2}\|B\|_{1;G/V,H/V}\rho^{n}\pi(G)$ and
\begin{equation}\label{eq:pi_fn}
    \left\|f_{n}-\pi(f_{n})\right\|_{1;G,H} \leq \|B\|_{1;G/V,H/V}\left( \lambda^{n}\|A\|_{1;V,VH/G}  + R_{2}\rho^{n}\pi(G)\right),
\end{equation}
where $\pi(G)\leq \pi(U)<\infty$ by the Assumptions~\ref{assu:5}-\ref{assu:6}.

Let $K_{\pi}:= \left\{f \in C^{1}_{G,H} \; |\; \pi(f) = 0\right\}$. Recall that
$(I - \mathcal{P}_{\delta}^{0})^{-1}$ defines a bounded linear map from $K_{\pi}$ to itself (see the discussion after the Assumption~\ref{assu:4}). In particular, for any $f\in K_{\pi}$, we have,
\begin{equation*}
    (I - \mathcal{P}_{\delta}^{0})^{-1} f(x) = \sum_{n=0}^{\infty} \left(\mathcal{P}_{\delta}^{0}\right)^{n} f(x) = \sum_{n=0}^{\infty}\mathbb{E}^{x}[f(X_{n})]
\end{equation*}
(the summation converges due to the spectral gap assumption), which implies that,
\begin{equation*}
    \left\|(I - \mathcal{P}_{\delta}^{0})^{-1} f\right\|_{1;G,H} = \left\|\sum_{n=0}^{\infty} \left(\mathcal{P}_{\delta}^{0}\right)^{n} f(x)\right\|_{1;G,H}  \leq \sum_{n=0}^{\infty}\lambda^{n}  \|f\|_{1;G,H} = \frac{1}{1-\lambda} \|f\|_{1;G,H}, \quad \forall f\in K_{\pi}.
\end{equation*}
Together with the Assumption~\ref{assu:5} and Eq.~\eqref{eq:pi_fn}, we have
\begin{equation*}
   \left\|\partial\mathcal{P}_{\delta}^{0}(I - \mathcal{P}_{\delta}^{0})^{-1}\left(f_{n}-\pi(f_n)\right)\right\|_{U} \leq \frac{R_{3}}{1-\lambda}\|f_{n}-\pi(f_{n})\|_{1;G,H} \leq \frac{R_{3}}{1-\lambda} \|B\|_{1;G/V,H/V}\left( \lambda^{n}\|A\|_{1;V,VH/G}  + R_{2}\rho^{n}\pi(G)\right),
\end{equation*}
for a constant $R_3\in (0,+\infty)$ independent of $A$ and $B$,
which leads to the desirable error bound in Eq.~\eqref{eq:cen_error}.

\bibliographystyle{plain}  
\bibliography{arxiv,InfSDEs}

\end{document}